\documentclass[conference]{IEEEtran}

\usepackage{cite}
\usepackage{amsmath,amssymb,amsfonts}
\usepackage{amsthm}
\usepackage{algorithm}
\usepackage{algorithmic}
\usepackage{graphicx}
\usepackage{textcomp}
\usepackage{xcolor}
\usepackage{dirtytalk}
\def\BibTeX{{\rm B\kern-.05em{\sc i\kern-.025em b}\kern-.08em
    T\kern-.1667em\lower.7ex\hbox{E}\kern-.125emX}}

\usepackage{amsmath,amsfonts,bm}

\def\eqref#1{equation~\ref{#1}}

\def\1{\bm{1}}

\def\vm{{\bm{m}}}

\def\mS{{\bm{S}}}

\DeclareMathAlphabet{\mathsfit}{\encodingdefault}{\sfdefault}{m}{sl}
\SetMathAlphabet{\mathsfit}{bold}{\encodingdefault}{\sfdefault}{bx}{n}

\newcommand{\E}{\mathbb{E}}

\newcommand{\R}{\mathbb{R}}

\begin{document}

\title{Can Stochastic Gradient Langevin Dynamics Provide
           Differential Privacy for Deep Learning?}

\author{\IEEEauthorblockN{Guy Heller}
\IEEEauthorblockA{\textit{University of Bar-Ilan} \\
Ramat Gan, Israel \\
guy.heller@biu.ac.il}
\and
\IEEEauthorblockN{Ethan Fetaya}
\IEEEauthorblockA{\textit{University of Bar-Ilan} \\
Ramat Gan, Israel \\
ethan.fetaya@biu.ac.il}
}

\newcommand{\citep}{\cite}
\newcommand{\citet}{\cite}
\newcommand\ef[1]{\textcolor{blue}{[EF:  #1]}}
\newcommand\gh[1]{\textcolor{cyan}{[GH:  #1]}}
\newtheorem{theorem}{Theorem}[section]
\newtheorem{proposition}[theorem]{Proposition}
\newtheorem{lemma}[theorem]{Lemma}
\newtheorem{corollary}[theorem]{Corollary}
\newtheorem{claim}[theorem]{Claim}
\newtheorem{definition}[theorem]{Definition}
\newtheorem{assumption}[theorem]{Assumption}
\IEEEpeerreviewmaketitle 
\maketitle

\begin{abstract}
Bayesian learning via Stochastic Gradient Langevin Dynamics (SGLD) has been suggested for differentially private learning. While previous research provides differential privacy bounds for SGLD at the initial steps of the algorithm or when close to convergence, the question of what differential privacy guarantees can be made in between remains unanswered. This interim region is of great importance, especially for Bayesian neural networks, as it is hard to guarantee convergence to the posterior. This paper shows that using SGLD might result in unbounded privacy loss for this interim region, even when sampling from the posterior is as differentially private as desired.
\end{abstract}


\section{Introduction}
\label{Introduction}
Machine learning models, specifically deep neural networks, achieve state-of-the-art results in various fields such as computer vision, natural language processing, and signal processing (e.g., \citet{DETR, devlin-etal-2019-bert, signal-processing}). Training these models requires data, which in some domains, e.g., healthcare and finance, can include {sensitive} information that should not be made public. Unfortunately, information from the training data can, in some cases, be extracted from the trained model \citet{ModelInversion, Carlini2021ExtractingTD}. One common approach to handle this issue is Differential Privacy (DP). DP framework ensures that the distribution of the training output would remain approximately the same {when} we switch one of the training examples, thus ensuring we cannot extract information specific to a unique individual. 

As privacy is usually obtained by adding random noise, it is natural to investigate whether Bayesian inference, which uses a distribution over models, can {yield} private predictions. Previous works have shown that sampling from the posterior is differentially private under certain mild conditions \citet{PrivacyForFree,FouldsGWC16,PosterioSamplingDimitrakakis}. The main disadvantage of this method is that sampling from the posterior can be challenging. The posterior generally does not have a closed-form solution, so iterative methods such as Markov Chain Monte Carlo (MCMC), whose sample distribution converges to the posterior, are commonly used. While theoretical bounds on the convergence of MCMC methods for non-convex problems exist \citep{SamplingIsFaster}, they usually require an infeasible number of steps to guarantee convergence in practice.

Stochastic Gradient Langevin Dynamics (SGLD)  \citep{Welling11} is a popular MCMC algorithm, as it avoids the accept-reject step. There are good reasons to believe that this specific sampling algorithm can provide private predictions. First, SGLD returns an approximate sample from the posterior, which can be private. Second, the SGLD process of stochastic gradient descent with Gaussian noise mirrors the common Gaussian mechanism in DP. 

Previous work \citet{PrivacyForFree} gives two separate privacy analyses related to SGLD: The first is based on the Gaussian mechanism and the Advanced Composition theorem \citep{AlgorithmicFundations}. Therefore, it only applies to a limited number of steps and is not connected to Bayesian sampling.

The second is for approximate sampling from the Bayesian posterior, which is only relevant when SGLD nearly converges. Neither of these results is suitable for deep learning and many other problems: one would limit the model's accuracy, and the other is unattainable in a reasonable time. Consequently, the privacy properties of SGLD in the interim region (between these two private sections) remain unknown even though they are of great interest.

\textbf{Our Contributions:} 
\begin{itemize}
    \item We provide a rigorous analysis of a counter-example based on a Bayesian linear regression problem, showing that approximate sampling using SGLD might result in unbounded loss of privacy in the interim region, even if sampling from the posterior is as private as desired.
    
    \item We further empirically show that SGLD can result in nonprivate models.
\end{itemize}
These results imply that special care should be given when using SGLD for private predictions, especially for problems for which it is infeasible to guarantee convergence.

\section{Related Work}
Several previous works investigate the connection between Bayesian inference and differential privacy \citep{PrivacyForFree,FouldsGWC16,ZhangRD16,PosterioSamplingDimitrakakis,GeumlekSC17,Ganesh2020FasterDP, ConnectMCMCtoDP}. None of these papers guarantees SGLD differential privacy in the interim region. However, the closest work to ours is \citet{PrivacyForFree}, which specifically investigates stochastic MCMC algorithms such as SGLD. As mentioned, its analysis only covers the initial phase and when approximate convergence is achieved.

In \cite{ChourasiaYS21}, the authors study the privacy guarantees of the noisy projected gradient descent algorithm. They consider a smooth and strongly convex loss function on a closed convex set with a finite gradient sensitivity and show an upper bound over the privacy loss, which converges exponentially fast in these settings. They also prove a lower bound on the R\'enyi-DP, which converges exponentially fast for smooth loss function on an unconstrained convex set with a finite total gradient sensitivity.

Several concurrent works study the DP guarantees of noisy stochastic gradient descent \cite{DBLP:journals/corr/abs-2203-05363} or projected noisy stochastic gradient descent \cite{https://doi.org/10.48550/arxiv.2205.13710, DBLP:journals/corr/abs-2201-11980} and show an upper bound over the privacy, which plateaus after a certain number of iterations. In \cite{DBLP:journals/corr/abs-2203-05363}, the authors show an upper bound over the DP for a strongly convex, smooth loss function with a gradient that has bounded $\ell_2$-sensitivity. In \cite{https://doi.org/10.48550/arxiv.2205.13710}, the authors study the DP guarantees under assumptions of convex, Lipschitz, and smooth loss function on a convex set with a bounded diameter. They also show the existence of a family of loss functions for which the bound is tight up to a constant factor. In \cite{DBLP:journals/corr/abs-2201-11980}, the authors study the DP guarantees under assumptions of convex, Lipschitz, and smooth loss function on a closed convex set.


When training machine learning models in a differentially private way via Stochastic Gradient Descent, a common practice is to apply the Gaussian Mechanism by clipping the gradients of the loss with respect to the weights and adding a matching noise (see \cite{Abadi}, for example). SGLD learning step resembles the resulting learning step but does not include gradients clipping. Reference \cite{ConnectMCMCtoDP} suggests incorporating gradient clipping in the SGLD step. However, clipping the gradients changes the algorithm properties, and it is not obvious if it converges to the posterior. As such, we do not consider it SGLD. Reference \cite{PrivacyForFree} circumvents this issue by assuming the log-likelihood of the model is Lipschitz continuous.

Another related work on the privacy of SGLD is \cite{SGLD_membership}, although they investigate a weaker type of privacy called membership privacy.

As many of the Bayesian methods' privacy bounds require sampling from the posterior, if SGLD is to be used, it requires non-asymptotic convergence bounds. Reference \citet{ConvexDalalyan} provides non-asymptotic bounds on the approximation error for a smooth and log-concave target distribution by Langevin Monte Carlo. Reference \citet{Cheng2018ConvergenceOL} studies the non-asymptotic bounds on the error of approximating a target density $p^*$ where $\log p^*$ is smooth and strongly convex.

For the non-convex setting, \citet{pmlr-v65-raginsky17a} shows non-asymptotic bounds on the 2-Wasserstein distance between SGLD and the invariant distribution solving It\^{o} stochastic differential equation. However, the 2-Wasserstein metric is ill-suited for differential privacy - it is easy to create two distributions with 2-Wasserstein distance as small as desired but with disjoint support. 

Total Variation (for details about Total Variation, see \citet{10.5555/1522486}) is a more suitable distance for working with differential privacy. Reference \citet{SamplingIsFaster} examines a target distribution $p^*$, which is strongly log-concave outside of a region of radius R, and where $-\ln{p^*}$ is $L$-Lipschitz. They provided a bound on the number of steps needed for the Total Variation distance between the distribution at the final step and $p^*$ to be smaller than $\epsilon$. This bound is proportional to $O(e^{32LR^2}\frac{d}{\epsilon^2})$, where $d$ is the model dimension. This result suggests that it is impractical to run SGLD until convergence is guaranteed in the non-convex setting.   

A conclusion from this work is that basing the differential privacy of SGLD on the proximity to the posterior is impractical for non-convex settings.

\section{Background}

\subsection{Differential Privacy}
Differential Privacy \citet{10.1007/11681878_14, eurocrypt-2006-2319,10.1145/1866739.1866758,AlgorithmicFundations} is a definition and a framework that enables performing data analysis on a dataset while reducing one's risk posed by disclosing its personal data to the dataset. In a nutshell, an algorithm is differentially private if it does not change its output distribution by much due to a single record change in its dataset. Approximate Differential Privacy, Definition \ref{Def-ADP}, is an extension of pure Differential Privacy, where pure differential privacy is Approximate Differential Privacy with $\delta=0$.

\begin{definition}\label{Def-ADP} 
Approximate Differential Privacy: A randomized algorithm $f:\mathcal{D}\rightarrow{Range}(f)$ is $(\epsilon, \delta)$-differentially private if $\forall S \subseteq Range(f)$ and $\forall D,\hat{D} \in \mathcal{D}: d(D,\hat{D}) \leq 1$ eq. \ref{DifferentialPrivacy} holds, where $d$ is the distance between $D$ and $\hat{D}$. $D, \hat{D}$ are called neighboring datasets, and while the metric can change per application, Hamming distance is typically used.
\begin{align}\label{DifferentialPrivacy}
p(f(D) \in S) \leq \exp(\epsilon)p(f(\hat{D}) \in S) + \delta
\end{align}
\end{definition}

R\'enyi Divergence \citep{renyi1961measures}, which generalizes the Kullback-Leibler divergence, is defined as follows:
\begin{definition}
R\'enyi Divergence: For two probability distributions $Z$ and $Q$, the R\'eyni divergence of order $\nu > 1$ is
\begin{align*}
\mathrm{D}_\nu(Z || Q) \overset{\Delta}{=} \frac1{\nu - 1}\log\mathbb{E}_{x \sim Q}\left[\left(\frac{Z(x)}{Q(x)}\right)^\nu\right].
\end{align*}
\end{definition}
Reference \citet{RenyiDP} suggested a relaxation of differential privacy based on the R\'enyi divergence, termed R\'enyi Differential Privacy: 
\begin{definition}\label{Def-Renyi-DP}
$(\nu, \epsilon)$-RDP: A randomized algorithm $f: \mathcal{D} \to Range(f)$ is said to have $\epsilon$-R\'enyi differential privacy of order $\nu$, or $(\nu, \epsilon)$-RDP in short, if for any neighbouring datasets $D, \hat{D} \in \mathcal{D}$ eq. \ref{Eq-Renyi-DP} holds, where $\mathrm{D}_{\nu}$ is R\'enyi divergence of order $\nu$.
\begin{align}\label{Eq-Renyi-DP}
\mathrm{D}_\nu\left(f(D)|| f(\hat{D})\right) \leq \epsilon
\end{align}
\end{definition}
In this paper, we utilize the fact that RDP has a closed-form solution when both $f(D)$ and $f(\hat{D})$ are Normal distributions (see \citet{RenyiDivergence} and the proof of Lemma \ref{PosteriorRenyiFull} in the appendix for details). 

By Proposition \ref{Lemma-RDP-to-ADP}, RDP guarantees can be translated into approximate differential privacy guarantees.

\begin{proposition}\label{Lemma-RDP-to-ADP}
From RDP to $(\epsilon, \delta)$-DP \citep{RenyiDP}: If f is $(\nu, \epsilon)$-RDP, it also satisfies $(\epsilon + \frac{\log\frac1\delta}{\nu - 1}, \delta)$-differential privacy for any $0 < \delta < 1$.
\end{proposition}

\subsection{Stochastic Gradient Langevin Dynamics}\label{SGLD-bkgrnd}
Stochastic Gradient Langevin Dynamics (SGLD) is an MCMC method commonly used for Bayesian Inference  \citep{Welling11}. Given a Bayesian model parameterized by $\theta$, a dataset $D = \{x_i, y_i\}_{i=1}^n$, a prior distribution $p(\theta)$, the likelihood function $p(y_i|\theta, x_i)$, and a batch size $b$, SGLD can be used for approximate sampling from the posterior $p(\theta|D)$. The update step of SGLD is shown in eq. \ref{Eq-SGLD-Update-Rule}, where $\theta_j$ is the parameter vector at step $j$, and $\eta_j$ is the step size at step $j$.
SGLD can be seen as a Stochastic Gradient Descent with Gaussian noise, where the variance of the noise is calibrated to the step size. 
\begin{align}\label{Eq-SGLD-Update-Rule}
\begin{split}
\theta_{j+1} &= \theta_j
\\&+ \frac{\eta_j}{2}\left(\nabla_{\theta_j}\ln{p(\theta_j)}+\frac{n}{b}\sum_{i=1}^{b}\nabla_{\theta_j}\ln{p(y_{i_j}|\theta_j, x_{i_j})}\right)
\\& + \sqrt{\eta_j}\xi_j 
\\i_j &\sim uniform\{1,...,n\} 
\\\xi_j &\sim \mathcal{N}(0,1)
\end{split}
\end{align}
A common practice in deep learning is to use \textit{cyclic} Stochastic Gradient Descent. This modification to SGD first randomly shuffles the dataset samples and then cyclically uses the samples in this order. For optimization, there is empirical evidence that it works as well or better than SGD with reshuffling, and it was conjectured that it converges at a faster rate \citep{YunSJ21}.
Cyclic-SGLD\footnote{Cyclic SGLD, which cycles through examples, should be distinguished from cSGLD \citep{DBLP:journals/corr/abs-1902-03932}, which uses a cyclic step size schedule.} is the analog of cyclic-SGD for SGLD, where the difference is the use of the SGLD step instead of the SGD step. For simplicity, we will consider cyclic-SGLD in this work. While this assumption simplifies the proof, we expect the general behavior to be equivalent.

\section{Theoretical Results}\label{Method}
Our goal is to prove that even when sampling from the posterior is as private as desired, approximate sampling using SGLD can be as nonprivate as desired in the interim region. This requires analysing the distribution of SGLD in the interim region, which is hard in the general case.
To circumvent this difficulty, we investigate the Bayesian linear regression problem, where the distributions are a mixture of Gaussians and thus have closed-form expressions. Our result is summarized in Theorem \ref{MainTheorem}.

\begin{theorem}
\label{MainTheorem}
$\forall\ 0 < \delta < 0.5$ and $\forall\epsilon, \epsilon' > 0$, there exists a number $T$, a domain, and a Bayesian inference problem for which a single sample from the posterior distribution is $(\epsilon, \delta)$ differentially private. However, performing approximate sampling by running SGLD for $T$ steps is not $(\epsilon', \delta)$ differentially private.
\end{theorem}
The Bayesian inference problem, mentioned in Theorem \ref{MainTheorem}, refers to sampling from the posterior for a dataset and a model defined by likelihood and prior distributions. The specific model and dataset we will analyse in our work are defined in eq. \ref{LinearModel} and eq. \ref{eq-Domain}, respectively.

An example of the behavior described by the theorem is depicted in Fig. \ref{fig::sgld-figure}. In this case, given the model defined in eq. \ref{LinearModel} and the domain defined in eq. \ref{eq-Domain}, a single sample from the posterior is $(\epsilon = 0.5, \delta=0.001)$ DP; however, approximate sampling from the posterior by running SGLD for $48$ epochs ($T=48\cdot\text{ dataset size}$) is not $(\epsilon'=38, \delta=0.001)$ DP (further details for the figure are provided below).

As Theorem \ref{MainTheorem} allows $\epsilon'$ to be as big as desired and $\epsilon$ to be as small as desired, a corollary of Theorem \ref{MainTheorem} is that we could always find a problem for which the posterior is $(\epsilon, \delta)$ differentially private, but there will be a step in which SGLD will result in an unbounded loss of privacy. Therefore, SGLD alone can not provide any privacy guarantees in the interim region, even if the posterior is private.

Theorem \ref{MainTheorem} is presented and proved for a fixed and equal $\delta$ for both the posterior and the SGLD privacy analysis. This is done for simplicity; however, the proof could be augmented to prove a lower bound on SGLD privacy for all $\epsilon' > 0$ and $0 < \delta' < 0.25$ (i.e., approximately sampling via SGLD is not $(\epsilon', \delta')$-DP for all $\epsilon' > 0$ and $0 < \delta' < 0.25$).

To prove our theorem, we consider a Bayesian regression problem for a 1D linear model with Gaussian noise, as defined in eq. \ref{LinearModel}. 
\begin{align}\label{LinearModel}
\begin{split}
{}& y = \theta{x} + \xi\\
{}& \xi \sim \mathcal{N}(0,\beta^{-1})\\
{}& \theta \sim \mathcal{N}(0,\alpha^{-1})\\
{}& p(y|x,\theta) \sim \mathcal{N}(\theta x, \beta^{-1})
\end{split}
\end{align}
We assume our input domain is 
\begin{align}\label{eq-Domain}
\mathcal{D}(n, \gamma_1, x_h, x_l, c) =&\ 
\big\{(x_i, y_i)| |\frac{y_i}{x_i} - c| \leq n^{\gamma_1};\nonumber\\
& \ x_i, y_i, c, \gamma_1 \in \mathbb{R}_{>0};\nonumber
\\&\ n\in\mathbb{Z}_{>0}; x_l \leq x_i \leq x_h \big\}_{i=1}^n
\end{align}
, where $x_h^2\beta >3$ and $\gamma_1 < \frac12$. The constants $n, c, x_l, x_h$, and $\gamma_1$ are parameters of the problem ($c, x_l, x_h$, and $\gamma_1$ are used, together with the dataset size - $n$, to bound the dataset samples to a chosen region). For every $\epsilon$, $\epsilon'$, and $\delta$, we will show the existence of parameters $n, c, x_l, x_h, \gamma_1$ values that have the privacy properties required to prove Theorem \ref{MainTheorem}. The restrictions on the dataset simplify the proof but are a bit unnatural as it assumes we approximately know $c$, the parameter we are trying to estimate. Later we show in subsection \ref{ProposeReleaseTest} that they can be replaced with a Propose-Test-Release phase. 

For simplicity, we will address the problem of sampling (or approximately sampling via SGLD) from the posterior for the model described in eq. \ref{LinearModel} and a dataset from domain $\mathcal{D}(n, \gamma_1, x_h, x_l, c)$ as a \textit{Bayesian linear regression problem} on domain $\mathcal{D}(n, \gamma_1, x_h, x_l, c)$. This problem has a closed-form solution for both the posterior distribution and the distribution at each SGLD step, thus enabling us to get tight bounds on the differential privacy in each case.

In essence, our proof shows that for a big enough $n$, sampling from the posterior is $(\epsilon, \delta)$ differentially private, with $\epsilon \sim \mathcal{O}(\frac{c^2}{n^3})$. However, for the same problem instance, there exists an SGLD step in which releasing a sample will not be $(\epsilon', \delta)$ differentially private for $\epsilon' = \Omega(\frac{c^2}{n^2})$. Therefore, for problem instances where $c \sim \mathcal{O}(n^{\frac{3}{2}}\sqrt\epsilon)$ and $n$ is big enough, sampling from the posterior will be $(\epsilon, \delta)$ differentially private, while there will be an SGLD step in which releasing a sample will not be $(\epsilon', \delta)$ differentially private for $\epsilon' = \Omega(n\epsilon)$. We note that the bounds dependend on $\delta$, but since we are using a fixed and equal $\delta$ for both the posterior and SGLD privacy analysis, we omit it from the bounds for simplicity.

\begin{figure}[t!]
\centerline{\includegraphics[width=0.9\linewidth]{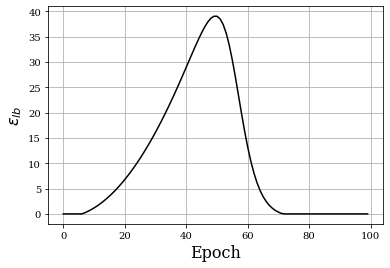}}
\caption{A lower bound over the DP of SGLD for the model defined in eq. \ref{LinearModel} and a dataset from domain $\mathcal{D}(n, \gamma_1, x_h, x_l, c)$ (defined in eq. \ref{eq-Domain}), given that $\delta = 0.001$. The domain parameters values are $n = 17864389, \alpha = 2, \beta = 1, \gamma_1 = 0.1, x_h = 1.8, x_l = 0.9$, which ensure $(0.5, 0.001)$-DP when sampling from the posterior.}
\label{fig::sgld-figure}
\end{figure}

Fig. \ref{fig::sgld-figure} depicts a lower bound over the DP of SGLD for the \textit{Bayesian Linear Regression Problem} on domain $\mathcal{D}(n, \gamma_1, x_h, x_l, c)$. The values of $n, \gamma_1, x_h, x_l, c$ ensure $(0.5, 0.001)$-DP when sampling from the posterior. However, we can see that sampling via SGLD in the interim region causes a significant privacy breach (sampling via SGLD at epoch $48$ is not $(38, 0.001)$-DP). For the derivation of the lower bound in Fig. \ref{fig::sgld-figure}, see subsection \ref{fig-deriv} in the appendix.

\subsection{Posterior Sampling Privacy}\label{PosteriorSamplingPrivacySection}
To prove Theorem \ref{MainTheorem}, we need to show that sampling from the posterior is private, while there is an SGLD sample that is not private at some intermediate step. In this section, we prove the first part - that a single sample from the posterior for the \textit{Bayesian linear regression problem} on domain $\mathcal{D}(n, \gamma_1, x_h, x_l, c)$ is differentially private.

We begin by using a well-known result for the closed-form solution of the posterior distribution for a Bayesian linear regression problem (see \citet{Bishop} for further details). By incorporating the parameters of our problem in this result, we get Lemma \ref{PosteriorDistribution}.

\begin{lemma}
\label{PosteriorDistribution}
The posterior distribution for the model defined in eq. \ref{LinearModel} on dataset $D=\{(y_i, x_i)\}_{i=1}^n$ is
\begin{align}\begin{split}
{}& p(\theta|D) = \mathcal{N}(\theta; \mu, \sigma^2);\\
{}& \mu = \frac{\sum_{i=1}^n x_i y_i\beta}{\alpha + \sum_{i=1}^n x_i^2\beta};
\sigma^2 = \frac{1}{\alpha + \sum_{i=1}^n x_i^2\beta}.
\end{split}\end{align}

\end{lemma}

As the posterior distribution is a Normal distribution, the R\'enyi divergence between every two posterior distributions has a closed-form solution. For two neighbouring datasets, $D, \hat{D} \in \mathcal{D}(n, \gamma_1, x_h, x_l, c)$, and matching posterior distributions $p(\theta|D) = \mathcal{N}(\theta;\mu, \sigma^2), p(\theta| \hat{D}) = \mathcal{N}(\theta;\hat\mu, \hat\sigma^2)$, the R\'enyi divergence of order $\nu$ is 
\begin{align*}\begin{split}
\mathrm{D}_\nu\left(p\left(\theta|D\right) || p(\theta|\hat{D})\right) & = \ln\frac{\sigma}{\hat\sigma} + \frac12(\nu - 1)\ln\frac{\hat\sigma^2}{(\sigma^2)^*_\nu} \\& + \frac12\frac{\nu(\mu - \hat\mu)^2}{(\sigma^2)^*_\nu}, \\
(\sigma^2)^*_\nu & = \nu\hat\sigma^2 + (1 - \nu)\sigma^2.
\end{split}\end{align*}

By bounding $\mathrm{D}_\nu(p(\theta|D) || p(\theta|\hat{D}))$ for every two neighbouring datasets, one can prove RDP. The first and second terms of 
$\mathrm{D}_\nu(p(\theta|D) || p(\theta|\hat{D}))$ can be bounded by $\mathcal{O}(\frac1n)$ using Taylor Theorem and the fact that the natural logarithm is monotonically increasing. By using direct computation, the third term can be bounded by $\mathcal{O}(\frac{1}{n^{1-2\gamma_1}}) + \mathcal{O}(\frac{c + n^{\gamma_1}}{n^{2 - \gamma_1}}) + \mathcal{O}(\frac{(c + n^{\gamma_1})^2}{n^3})$. This gives way to Lemma \ref{PosteriorRenyi}. For the full proof, see subsection \ref{subseq-PSP} in the appendix.

\begin{lemma}
\label{PosteriorRenyi}
For the Bayesian linear regression problem on domain $\mathcal{D}(n, \gamma_1, x_h, x_l, c)$, such that $n > \max\{1 +  10\frac{x_h^2}{x_l^2}\frac{\nu}{\beta}, 1 + \nu\frac{x_h^2}{x_l^2}\}$, one sample from the posterior is $(\nu, \epsilon_1)$-R\'enyi differentially private, and $\epsilon_1$ is
\begin{align}\label{eq-posterior-renyi}
\begin{split}
\epsilon_1 & =
\mathcal{O}\left(\frac1n\right) + \mathcal{O}\left(\frac{1}{n^{1-2\gamma_1}}\right) + \mathcal{O}\left(\frac{c + n^{\gamma_1}}{n^{2 - \gamma_1}}\right) 
\\&+
\mathcal{O}\left(\frac{(c + n^{\gamma_1})^2}{n^3}\right).
\end{split}
\end{align}
\end{lemma}

We can show that for $c >> n^{1 + \gamma_1}$, each of the terms in the right hand side of eq. \ref{eq-posterior-renyi} is bounded by $\mathcal{O}(\frac{c^2}{n^3})$. The first term is trivially bounded by $\mathcal{O}(\frac{1}{n})$. For the second term, noticing that $n^{2{\gamma_1} - 1} = \frac{n^{2(1 + \gamma_1)}}{n^3} < \frac{c^2}{n^3}$, we get that it is bounded by $\mathcal{O}(\frac{c^2}{n^3})$. As $c >> n^{\gamma_1}$, the third term is bounded by $\mathcal{O}(\frac{cn^{\gamma_1}}{n^2})$, and since $\frac{cn^{\gamma_1}}{n^{2}} = \frac{cn^{1 + \gamma_1}}{n^{3}} < \frac{c^2}{n^{3}}$, the term is bounded by $\mathcal{O}(\frac{c^2}{n^3})$. Lastly, since $c >> n^{\gamma_1}$ the last term is bounded by $\mathcal{O}(\frac{c^2}{n^3})$.

Translating the R\'enyi differential privacy guarantees of Lemma \ref{PosteriorRenyi} into approximate differential privacy terms can be done according to Lemma \ref{Lemma-RDP-to-ADP}, which gives Lemma \ref{PosteriorPrivacy}.

\begin{lemma}\label{PosteriorPrivacy}
With the conditions of Lemma \ref{PosteriorRenyi}, one sample from the posterior is $(\epsilon_1 + \frac{\ln(\frac1\delta)}{\nu - 1}, \delta)$ differentially private.
\end{lemma}

By choosing $\nu$ such that $\frac{\ln(\frac1\delta)}{\nu - 1} < \frac\epsilon2$ and then choosing $n$ big enough such that $\epsilon_1 < \frac{\epsilon}{2}$, we get that the posterior is $(\epsilon, \delta)$ differentially private.

\subsection{Stochastic Gradient Langevin Dynamics Privacy}\label{subseq-SGLD}
To complete the proof of Theorem \ref{MainTheorem}, we need to show that given a \textit{Bayesian linear regression problem} on domain $\mathcal{D}(n, \gamma_1, x_h, x_l, c)$, even if one sample from the posterior is ($\epsilon, \delta$) differentially private, it does not guarantee SGLD is private in the interim region. In order to do so, this section will first consider the loss of privacy when using SGLD for the \textit{Bayesian linear regression problem} on domain $\mathcal{D}(n, \gamma_1, x_h, x_l, c)$ and then, together with the results of section \ref{PosteriorSamplingPrivacySection}, will prove Theorem \ref{MainTheorem}.

In order to show that SGLD is not differentially private after initial steps and before convergence, it is enough to find two neighbouring datasets for which the loss in privacy is as big as desired after a certain number of steps. We define neighbouring datasets $D_1, D_2 \in \mathcal{D}(n, \gamma_1, x_h, x_l, c)$ in eq. \ref{N-DBs} and consider the \textit{Bayesian linear regression problem} on $D_1$ and $D_2$ with a learning rate: $\eta = \frac{2}{(\alpha + nx_h^2\beta)^2}$.

\begin{align} \label{N-DBs}
{}& D_1 = \{(x_i,y_i): x_i=x_h, y_i = c\cdot x_h\}_{i=1}^n\\
{}& D_2 = \{(x_i,y_i): x_i=x_h, y_i = c\cdot x_h\}_{i=1}^{n-1} \cup \{(\frac{x_h}2, c\cdot\frac{x_h}2)\}\nonumber
\end{align}

A closed-form solution for the distribution at each step enables us to get a tight lower bound over the differential privacy loss when approximately sampling via SGLD at each step. For dataset $D_1$, the solution is a Normal distribution. For dataset ${D}_2$, different shuffling of samples produces different Gaussian distributions, therefore giving a mixture of Gaussians.

We look at cyclic-SGLD with a batch size of $1$ and mark by $\theta_j, \hat\theta_j$ the samples on the $j$'th SGLD step when using datasets $D_1$ and ${D}_2$ accordingly. Since $D_1$ samples are all equal, the update step of the cyclic-SGLD is the same for every step (with different noise generated for each step). This update-step contains only multiplication by a scalar, addition of a scalar, and addition of Gaussian noise, therefore, together with a conjugate prior results in Normal distribution for $\theta_j$: $\mathcal{N}(\theta_j; \mu_j, \sigma_j^2)$, where $\mu_j, \sigma_j \in \R$.

For $D_2$, there is only one sample different from the rest. We mark by $r$ the index in which this sample is used in the cyclic-SGLD and call this order $r$-order. Note that there are only $n$ ($n$ is the dataset size, defined in eq. \ref{eq-Domain}) different values for $r$ and, as such, effectively only $n$ different samples orders. Since every order of samples is chosen with the same probability, $r$ is distributed uniformly in $\{1,..,n\}$. We mark by $\hat\theta_j^r$ the sample on the $j$'th SGLD step when using $r$-order. Since, for a given order, $\hat\theta_j^r$ is formed by a series of multiplications by a scalar, addition of scalar, and addition of Gaussian noise, and since the prior is also Gaussian, then $\hat\theta_j^r$ is distributed Normally, $\mathcal{N}(\hat\theta_j^r; \hat\mu_j^r, (\hat\sigma_j^r)^2)$, where $\hat\mu_j^r, \hat\sigma_j^r \in \R$. As $r$ is distributed uniformly, $\hat\theta_j$ distribution mass is distributed evenly between all $\hat\theta_j^r$, resulting in a mixture of Gaussians.

Intuitively what will happen is that each Gaussian component, $\hat{\theta}_j$ as well as $\theta_j$, will move towards a similar Gaussian posterior. However, at each epoch, $\hat{\theta}_j$ will drag a bit behind because a single gradient in one of the batches will be smaller. While this gap can be quite small, for large $n$, the Gaussians are very peaked with very small standard deviations; thus, they are separate enough that we can easily distinguish between the two distributions. 

According to the approximate differential privacy definition (Definition \ref{Def-ADP}), it is enough to find one set, $S$, such that $p(\theta_j \in S) > e^\epsilon p(\hat{\theta}_j \in S) + \delta$, to prove that releasing $\theta_j$ is not $(\epsilon, \delta)$ private. We choose $S = \{s| s > \mu_j\}$ at some step $j$ that we will define later on. 

To show that $p(\theta_j \in S) > e^\epsilon p(\hat{\theta}_j \in S) + \delta$, we first note that as the Gaussian $\theta_j$ is symmetric, it is clear that $p(\theta_j>\mu_j)=1/2$. Now we turn our focus to upper bounding $p(\hat{\theta}_j > \mu_j)$. This can be done using Chernoff bound, as stated in Lemma \ref{lemma:diff_of_means}.
\begin{lemma}\label{lemma:diff_of_means}
\label{SGLD-0}
$p(\hat{\theta}_j > \mu_j) \leq \frac1n\sum_{r=1}^n\exp(-\frac{(\mu_j - \hat{\mu}_j^r)^2}{2(\hat\sigma_j^r)^2})$.
\end{lemma}

To bound $p(\hat{\theta}_j > \mu_j)$ using Lemma \ref{SGLD-0}, we first need to lower bound $\frac{(\mu_j - \hat{\mu}_j^r)^2}{(\hat{\sigma}_j^r)^2}$ for a certain step. This is done in Lemma \ref{SGLD-1}.

\begin{lemma}
\label{SGLD-1}
$\exists k \in \mathbb{Z}_{>0}$ such that $\frac{(\mu_{(k+1)n} - \hat\mu_{(k+1)n}^r)^2}{(\hat{\sigma}^r_{(k + 1)n})^2} = \Omega(\frac{c^2}{n^2})$, for big enough $n$.
\end{lemma}

To prove Lemma \ref{SGLD-1}, we first find closed-form solutions for $\hat\theta_{(k+1)n}^r$, $\theta_{(k+1)n}$ distributions (Lemma \ref{Lem-App-SGLD-1}). Using the closed-form solutions, we find a lower bound over $(\mu_{(k+1)n} - \hat\mu_{(k+1)n}^r)^2$ as a function of $k$, which applies for all $k$ (Lemma \ref{Lem-App-SGLD-2}). To upper bound $(\hat{\sigma}^r_{(k+1)n})^2$, we find an approximation to the epoch in which the data and prior effect on the variance is approximately equal, marked $\dot{k}$. We choose $(\lceil\dot{k}\rceil + 1)n$ as the step in which we will consider the privacy loss and show that $(\hat{\sigma}^r_{(\lceil\dot{k}\rceil+1)n})^2$ is upper bounded at this step (Lemma \ref{Lemma-SGLD-2.3}). Using the lower bound on the difference in means and the upper bound on the variance, Lemma \ref{SGLD-1} is proved.

By using the lower bound from Lemma \ref{SGLD-1} in Lemma \ref{SGLD-0}, we get Lemma \ref{SGLD-2}.

\begin{lemma}
\label{SGLD-2}
For the \textit{Bayesian linear regression problem} over dataset $D_1$ and $n$ big enough, $\exists T \in \mathbb{Z}_{>0}$ such that approximate sampling by running SGLD for $T$ steps will not be $(\epsilon, \delta)$ private for $\epsilon = \Omega(\frac{c^2}{n^2}), \delta < 0.5$.
\end{lemma}

From Lemma \ref{PosteriorPrivacy}, we see that sampling from the posterior is $(\epsilon, \delta)$ differentially private for $\epsilon = \mathcal{O}(\frac{c^2}{n^3})$. From Lemma \ref{SGLD-2}, we see that for SGLD, there exists a step in which releasing a sample will not be $(\epsilon', \delta)$ differentially private for $\epsilon' = \Omega(\frac{c^2}{n^2})$. Therefore, for problem instances where $c = \mathcal{O}(n^\frac32\sqrt\epsilon)$, sampling from the posterior will be $(\epsilon, \delta)$ differentially private. However, there will be an SGLD step in which releasing a sample will not be $(\epsilon', \delta)$ differentially private for $\epsilon' = \Omega(n\epsilon)$. Since we can choose $n$ to be big as desired, we can make the lower bound over $\epsilon'$ as big as we desire it to be. This completes the proof of Theorem \ref{MainTheorem}.

\subsection{Propose Test Sample}\label{ProposeReleaseTest}
Our analysis of the posterior and SGLD is done on a restricted domain - $\mathcal{D}(n, \gamma_1, x_h, x_l, c)$. These restrictions over the dataset simplify the proof but are a bit unnatural as they assume we approximately know $c$, the parameter we are trying to estimate. This section shows that these restrictions could be replaced with a Propose-Test-Release phase \citep{10.1145/1536414.1536466} and common practices in data science.

When training a statistical model, it is common to first preprocess the data by restricting it to a bounded region and removing outliers. After the data is cleaned, the training process is performed. This is especially important in DP, as outliers can significantly increase the algorithm's sensitivity to a single data point and thus hamper privacy.

Informally, Algorithm \ref{A-PTR} starts by clipping the input to the accepted range. It then estimates a weighted average of the ratio $\frac{y_i}{x_i}$ (line 16) and throws away outliers that deviate too much from it. The actual implementation of this notion is a bit more complicated because of the requirement to do so privately. Once the dataset is cleaned, Algorithm \ref{A-PTR} privately verifies that the number of samples is big enough, so the sensitivity of $p(\theta|W)$ (where $W$ is the cleaned dataset) to a single change in the dataset will be small, therefore making sampling from $p(\theta |W)$ $(\epsilon, \delta)$ differentially private. This method is regarded as Propose-Test-Release, where we first propose a bound over the sensitivity, then test if the dataset holds this bound, and finally release the result if so.

In eq. \ref{alg-n_min} in the appendix, we define $n_{min}$ as the minimum size of $W$ for which the algorithm will sample from $p(\theta | W)$ with high probability. We will show later on that this limit ensures that sampling from $p(\theta | W)$ is $(\epsilon, \delta)$ differentially private.

\begin{algorithm}[tb]
   \caption{Propose Test Sample}
   \label{A-PTR}
\begin{algorithmic}[1]
   \STATE {\bfseries Input:} $D = \{x_i, y_i\}_{i=1}^{n_1}$
   \STATE {\bfseries Parameters:} $\epsilon, \delta < 0.5, x_l > 0, x_h > x_l, \alpha > 0, \beta \geq \frac3{x_h^2}, \rho_1 \in (1, \frac32), \rho_2 \in (0,\frac12), \gamma_1 \in (\rho_2, \frac12), n_1\in\mathbb{Z}_{>0}$
    \FOR {$i = 1,2,\ldots,N$}
	    \STATE $x_i \leftarrow \max\{x_i, x_l\}$
	    \STATE $x_i \leftarrow \min\{x_i, x_h\}$
	    \STATE $y_i \leftarrow \max\{y_i, 0\}$
	\ENDFOR
        \STATE $l_1 \leftarrow \text{sample from Laplace}(0, \frac1\epsilon)$
	\STATE $\breve{n}_1 \leftarrow n_1 - \frac1\epsilon\log\frac1{2\delta} + l_1$
	\STATE $V = \{x_i, y_i| \frac{y_i}{x_i} \leq \breve{n}_1^{\rho_1}\}$
        \STATE $l_2 \leftarrow \text{sample from Laplace}(0, \frac1\epsilon)$
	\STATE $n_2 \leftarrow |V| - \frac1\epsilon\log\frac1{2\delta} + l_2$
	\IF {$n_2 \leq 1$}
        \STATE return null
	\ENDIF
    \STATE $m \leftarrow \frac{\sum_{(x_i, y_i)\in V}x_iy_i}{\sum_{(x_i, y_i)\in V}x_i^2}$
    \STATE $l_3 \leftarrow \text{sample from Laplace}(0, \frac1\epsilon\breve{n}_1^{\rho_1}\frac{2(n_2 - 1)x_h^2x_l^2 + x_h^4}{n_2(n_2 - 1)x_l^4})$
    \STATE $\breve{m} \leftarrow m + l_3$
	\STATE $W \leftarrow \{(x_i, y_i) : |\frac{y_i}{x_i} - \breve{m}| \leq n_2^{\rho_2}\}$
        \STATE $l_4 \leftarrow \text{sample from Laplace}(0, \frac1\epsilon)$
	\STATE $n_W \leftarrow |W| - \frac1\epsilon\log(\frac1{2\delta}) + l_4$
	\IF {$n_W < n_{min}$}
	    \STATE return null
	\ENDIF
	\STATE return sample from $p(\theta | W)$
\end{algorithmic}
\end{algorithm}

We define $p(\theta|W)$ as the posterior of the 1D linear regression model defined in eq. \ref{LinearModel} over dataset $W$.
From Lemma \ref{PosteriorDistribution}, it follows that $p(\theta|W)$ has the form of
\begin{align*}\begin{split}
{}& p(\theta|W) = \mathcal{N}(\theta; \mu, \sigma^2);\ 
\\{}& \mu = \frac{\sum_{(x_i,y_i) \in W} x_i y_i\beta}{\alpha + \sum_{(x_i,y_i) \in W} x_i^2\beta};\ \sigma^2 = \frac{1}{\alpha + \sum_{(x_i,y_i) \in W} x_i^2\beta}.
\end{split}\end{align*} 

\begin{claim}\label{claim-A-PTR-Posterior-private}
Algorithm 1 is $(5\epsilon, 2\delta)$ differentially private.
\end{claim}
By Claim \ref{ro-5}, lines 8-18 are $(3\epsilon, \delta)$ differentially private. By Corollary \ref{cor-pts-4}, lines 19-25 are $(2\epsilon, \delta)$ differentially private given $\breve{m}$ and $n_2$. Therefore by the sequential composition theorem, the composition is $(5\epsilon, 2\delta)$ differentially private. The claim is proved by noticing that if lines 8-25 are private with respect to the updated dataset (after line 7), then they are also private for the original dataset.

\begin{claim}\label{claim-A-PTR-not-private}
When replacing line 25 with approximate sampling via SGLD with step size $\eta = \frac1{(\alpha + n_1x_h^2\beta)^2}$, there exists $T(n_1):\mathbb{Z}_{>0} \to \mathbb{Z}_{>0}$ such that the updated algorithm is not $(\epsilon, \delta)$ differentially private $\forall \epsilon \in \mathbb{R}_{>0}, \delta < \frac16$ if ran for $T(n_1)$ steps.
\end{claim}

Proof sketch (See appendix for full proof). 
We analyze a run of Algorithm \ref{A-PTR} on the neighbouring datasets, $D_3$ and $D_4$, defined in eq. \ref{DB34}. First, note that when choosing $1 + \rho_2 > \rho_1$, the sensitivity of $\breve{m}$ grows slower than the bound over the distance $|\frac{y_i}{x_i} - \breve{m}|$ in $n_1$ for both of the datasets.
Therefore, with high probability, for $n_1$ big enough, $W$ will contain all the samples that meet the condition $\frac{y_i}{x_i} = m$. Consequently, with high probability, the algorithm will reach line 25, which, from our previous analysis over SGLD (see subsection \ref{subseq-SGLD}) will cause an unbounded loss of privacy.
\begin{equation}\label{DB34}
\begin{split}
\rho_1 & > \rho_3 > 1
\\ D_3 & = \big\{(x_i,y_i): x_i=x_h, y_i = n_1^{\rho_3}\cdot x_h\big\}_{i=1}^{n_1}
\\ D_4 & = \big\{(x_i,y_i): x_i=x_h, y_i = n_1^{\rho_3}\cdot x_h\big\}_{i=1}^{n_1-1} 
\\ & \cup \big\{(\frac{x_h}2, n_1^{\rho_3}\cdot\frac{x_h}2)\big\}
\end{split}
\end{equation}

\section{Empirical Evidence}\label{Empirical}
We augment our theoretical analysis with an empirical study on privacy loss when training a deep neural network via SGLD.
This study strengthens our claim that one should use SGLD with great care for private learning.

To empirically estimate SGLD's privacy, we attack it using a version of the \textit{adversary instantiation} method described in \citet{Nasr}, with some modifications to the method's details. In broad strokes, we train with SGLD a set of models on each of two neighbouring datasets, $D$ and $D'$. Then we try to predict for each model on which dataset it was trained. If the algorithm is DP, it will be hard to distinguish which dataset was used to train the model, and the accuracy will be low. Concretely, by analyzing the prediction's false positive and false negative rates, we can deduce a lower bound over the training DP parameters - $(\epsilon,\delta)$.

To create the neighboring dataset $D'$, we replace one of the samples from $D$ with a novel data point - $(x^*, y^*)$. To show SGLD is not private, we need a sample, $(x^*,y^*)$, such that the models that were not trained on it will misclassify it, but models trained on it will classify it correctly after a small number of epochs.

To create $x^*$, we first train $M$ models, $\{m\}_{i=1}^M$, on dataset $D$. Then, we search for a sample in $D$, marked $(x^0,y^0)$, such that $\{m\}_{i=1}^M$ agree on it's label: $\forall 1 \leq i,j \leq M: m_i(x^0) = m_j(x^0)$. We then use DeepFool \citep{Moosavi-Dezfooli_2016_CVPR} to alter the sample $x^0$ into $x^*$ such that all the models will misclassify it with regard to their original prediction: $\forall 1 \leq i \leq M: m_i(x^0) \neq m_i(x^*)$. We set $y^* = y^0$ and $D' = D\setminus\{(x^0,y^0)\}\cup\{(x^*, y^*)\}$.

Given $D$ and $D'$, we generate a dataset, $A_1$, of models trained on $D$ and $D'$ with equal probability. We represent a model with parameters $\theta$ by four features, $p(y=y^*|x^*,\theta)$, $p(y=y^0|x^*,\theta)$, $p(y=y^*|x^0,\theta)$, $p(y=y^0|x^0,\theta)$, and train a simple linear classifier. Finally, we create a second independent test set, $A_2$, of models trained on $D$ and $D'$ with equal probability and estimate our classifier's false negative (FN) and false positive (FP) rates using the examples from $A_2$.

\subsection{Deducing a lower bound over $\epsilon$}
To translate the attack results into DP parameters, we follow the analysis approach suggested by \citet{10.1145/3243734.3243818, DBLP:conf/nips/JagielskiUO20} and extended by \citet{Nasr}. Without loss of generality, we define false positive as predicting dataset $D'$, when dataset $D$ was used for training a model, and false negative as vice versa. The probability for FP and FN are marked as $P_{FP}$ and $P_{FN}$, respectively. According to \citet{pmlr-v37-kairouz15}, if an algorithm is $(\epsilon, \delta)$-DP, then the following inequalities hold:
\begin{equation}\label{Kairouz}
\begin{split}
{}& P_{FP} + e^\epsilon P_{FN} \geq 1 - \delta
\\{}& P_{FN} + e^\epsilon P_{FP} \geq 1 - \delta.
\end{split}
\end{equation}
These inequalities can easily be translated into a lower bound over $\epsilon$, 
\begin{equation}\label{Kairouz-eps}
\epsilon_{lb} \geq \max\left(\log\frac{1 - \delta - P_{FP}}{P_{FN}}, \log\frac{1 - \delta - P_{FN}}{P_{FP}}\right).
\end{equation}

Since we can only estimate $P_{FP}$ and $P_{FN}$ empirically, we use confidence intervals to upper-bound them. The confidence intervals are calculated using the Clopper-Pearson method \citep{Clopper1934THEUO} on the attack's false positive and false negative rates. The resulting upper bounds, $P_{FP}^{high}$ and $P_{FN}^{high}$, are then used to provide an empirical lower bound on $\epsilon$ with high probability:
\begin{equation}\label{Kairouz-lb}
\epsilon_{lb}^{emp} \geq \max\left(\log\frac{1 - \delta - P_{FP}^{high}}{P_{FN}^{high}}, \log\frac{1 - \delta - P_{FN}^{high}}{P_{FP}^{high}}\right).
\end{equation}
It is important to note that this method can only prove that a model is not private. A low value for $\epsilon_{lb}^{emp}$ does not show the model is private, only that our attack failed to prove a lack of privacy.

\subsection{Results}\label{subseq::Resulsts-LeNet}
We performed our attack on the SGLD-based training process of a LeNet5 \citep{LeNet5Article}, trained on the MNIST dataset \citep{LeNet5Article}. We tested a learning rate of $0.001$\footnote{Effective learning rate after multiplication by SGLD's normalization factor, i.e. $\eta\frac{n}{2b}$. See SGLD step in eq. \ref{Eq-SGLD-Update-Rule} for details.} with a batch size of 4. We trained a different classifier for each epoch to find a lower bound on the DP at each epoch.

When using Clopper-Pearson \citep{Clopper1934THEUO} confidence intervals, the resulting upper bounds ($P_{FP}^{high}$ and $P_{FN}^{high}$) are limited by the number of experiments conducted, which limits the maximum $\epsilon^{emp}_{lb}$. We used $500$ models to train the classifier (i.e., $|A_1|=500$) and evaluated the attack on $500$ models (i.e., $|A_2|=500$), which limits $\epsilon^{emp}_{lb}$ to a maximum of $4.89$.

Fig. \ref{fig::SGLD-empirical-epsilon-LeNet} depicts lower bounds over $\epsilon$ given $\delta = 10^{-5}$, with a confidence value of $90\%$, i.e., $P(\epsilon \geq \epsilon_{lb}^{emp}) \geq 0.9$, as well as the accuracy of the network, as a function of the number of epochs.

It should be emphasized that we show a lower bound over $\epsilon$. As such, even a small value is sufficient to 
show that the classifier can reliably infer which of the datasets was used to train the model. For example, a lower bound of $(\delta = 10^{-5}, \epsilon = 3)$ allows a classifier to identify on which dataset a model was trained with an accuracy of $95\%$ ($P_{FP}^{high} = P_{FN}^{high} \simeq 0.05$).
\begin{figure}[t]
\begin{center}
\includegraphics[width=0.9\linewidth]{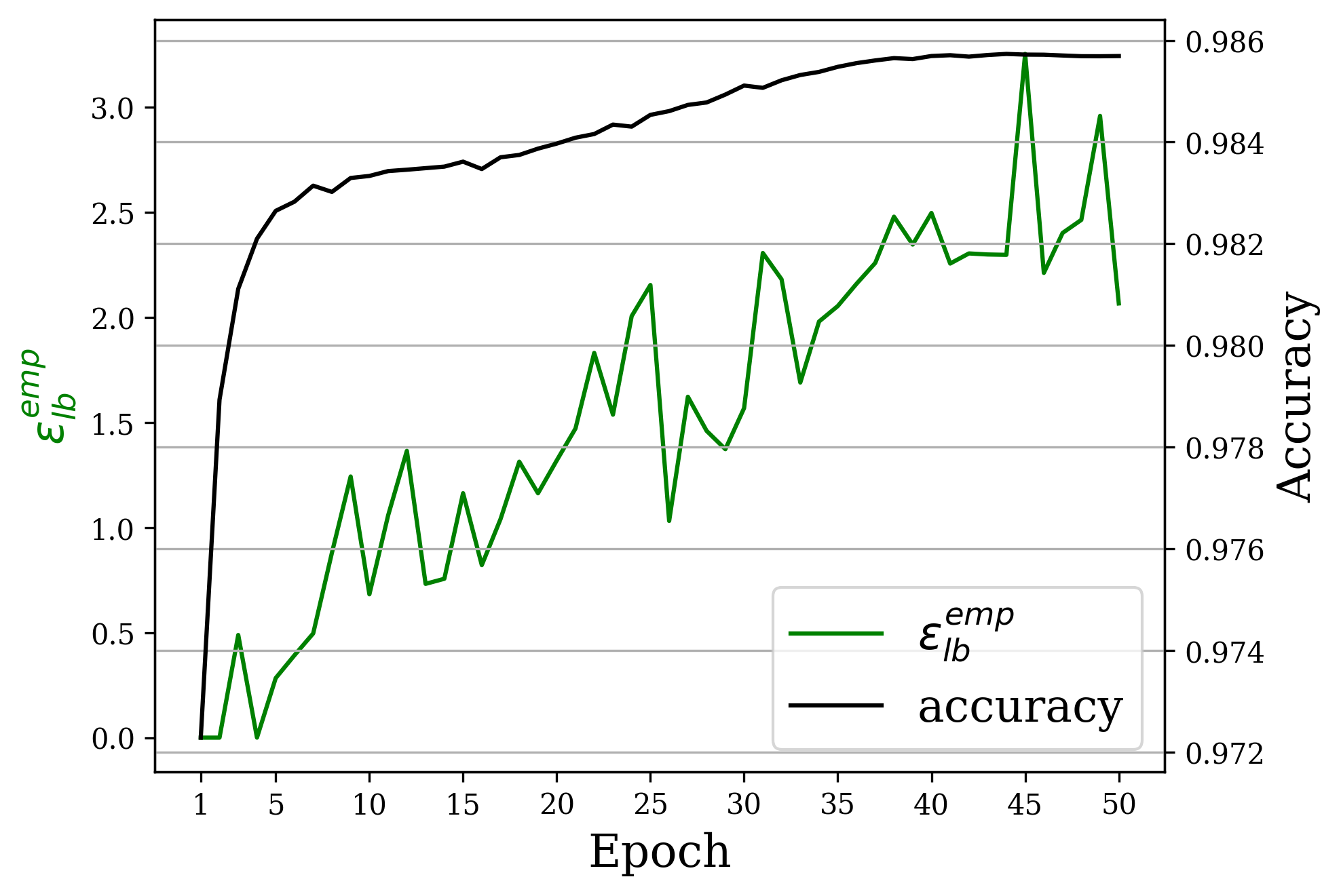}
\end{center}
\caption{A lower bound over the Differential Privacy of the LeNet5, SGLD-based training process over MNIST for a given $\delta=10^{-5}$, a learning rate of $0.001$, and a batch size of $4$.}
\label{fig::SGLD-empirical-epsilon-LeNet}
\end{figure}

In appendix \ref{sec::clipped_grads}, we show the results for SGLD with clipped gradients. We see that clipping the gradients protects the algorithm from our attack. As mentioned, an implementation that clips the gradients diverges from SGLD and, as such, has different sampling properties. Indeed, the results show that this version's accuracy degrades by $8\%$ compared to SGLD.

\section*{Acknowledgment}
We would like to express our gratitude to Dr. Or Sheffet for his help and advice throughout the development of this paper.

\bibliography{main.bib}{}
\bibliographystyle{IEEEtran}
\onecolumn
\appendices
\section{SGLD and Posterior Privacy}\label{app:SGLD}
Appendix \ref{app:SGLD} provides proofs for theorem \ref{MainTheorem} and the lemmas in subsections \ref{subseq-SGLD} and \ref{PosteriorSamplingPrivacySection}. As such, it uses the notations defined in section \ref{Method} and subsections \ref{subseq-SGLD} and \ref{PosteriorSamplingPrivacySection}. To ease the proof's reading, we repeat these notations here.

$\alpha, \beta$ and  $\theta$ are parameters of the linear model defined in eq. \ref{app:LinearModel} (originally defined in eq. \ref{LinearModel}), and $p(y|x)$ is the model likelihood.
\begin{align}\label{app:LinearModel}
\begin{split}
{}& y = \theta{x} + \xi\\
{}& \xi \sim \mathcal{N}(0,\beta^{-1})\\
{}& \theta \sim \mathcal{N}(0,\alpha^{-1})\\
{}& p(y|x) \sim \mathcal{N}(\theta x, \beta^{-1})
\end{split}
\end{align}
$x_h, x_l, c, n$, and $\gamma_1$ are defined as part of domain $\mathcal{D}(n, \gamma_1, x_h, x_l, c)$ definition (originally defined in eq. \ref{eq-Domain}): 
\begin{align}\label{app:eq-Domain}\begin{split}
\mathcal{D}(n, \gamma_1, x_h, x_l, c) & = 
\{(x_i, y_i)| |\frac{y_i}{x_i} - c| \leq n^{\gamma_1};\ x_i, y_i, c, \gamma_1 \in \mathbb{R}_{>0}; n\in\mathbb{Z}_{>0}; x_l \leq x_i \leq x_h \}_{i=1}^n
\end{split}
\end{align}
where $x_h^2\beta >3$ and $\gamma_1 < \frac12$. The datasets $D_1, D_2 \in \mathcal{D}(n, \gamma_1, x_h, x_l, c)$ (originally defined in eq. \ref{N-DBs}) are defined in eq. \ref{app:N-DBs}. 
\begin{align} \label{app:N-DBs}
\begin{split}
{}& D_1 = \{x_i,y_i: x_i=x_h, y_i = c\cdot x_h\}_{i=1}^n\\
{}& D_2 = \{x_i,y_i: x_i=x_h, y_i = c\cdot x_h\}_{i=1}^{n-1} \cup \{\frac{x_h}2, c\cdot\frac{x_h}2\}
\end{split}
\end{align}

The \textit{Bayesian Linear Regression Problem} (originally defined in section \ref{Method}) refers to the problem of sampling (or approximately sampling via SGLD) from the posterior for the model described in eq. \ref{app:LinearModel}.

We look at cyclic-SGLD with a batch size of $1$ and mark by $\theta_j, \hat\theta_j$ the samples on the $j$'th SGLD step when using datasets $D_1$ and ${D}_2$ accordingly. $\mu_j, \sigma_j^2 \in \R$ are the mean and variance of $\theta_j$. For $D_2$, there is only one sample different from the rest. We mark by $r$ the index in which this sample is used in the cyclic-SGLD and call this order $r$-order. We mark by $\hat\theta_j^r$ the sample on the $j$'th SGLD step when using dataset $D_2$ and $r$-order. $\hat\mu_j^r, (\hat\sigma_j^r)^2 \in \R$ are the mean and variance of $\hat\theta_j^r$. $\eta = \frac{2}{(\alpha + nx_h^2\beta)^2}$ is the SGLD learning rate. $\eta, \theta_j, \hat{\theta}_j, \mu_j, \hat{\mu}^r_j, \sigma_j, \hat{\sigma}^r_j$ were originally defined in subsection \ref{subseq-SGLD}.

\subsection{Theorem \ref{MainTheorem} Proof}
\begin{proof}[Proof of Theorem \ref{MainTheorem}]
We first define several parameters used to configure the \textit{Bayesian linear regression problem} on domain $\mathcal{D}(n, \gamma_1, x_h, x_l, c)$.
\begin{align*}
\frac12 & > \gamma_1 > 0;\ \frac32 > \gamma_2 > 1 + \gamma_1 ;\ x_l = \frac{x_h}2
\\ \nu_1 & = \frac{2\ln(\frac1\delta)}\epsilon + 1
\\ n_1 & = \max\left\{\frac{1}{2\alpha x_h^2\beta} - \frac{1}{x_h^2\beta}, \frac\alpha{x_h^2\beta},  \frac{\alpha}{x_h^2\beta}(e^{\frac2{x_h^2\beta}} - 2) + \frac1{2x^2\beta}\right\} + 1
\\ n_2 & = \max\Big\{1 + \frac{x_h^2}{x_l^2}\frac{8}{\epsilon},
1 + \nu\frac{x_h^2}{x_l^2}\left(1 + 8\frac{\left(\nu - 1\right)}{\epsilon}\right), 
\left(\frac{16\nu\beta x_h^4}{\frac9{10}\epsilon x_l^2}\right)^{\frac1{1 - 2\gamma_1}},
\\& \left(\frac{16\nu}{\epsilon}\cdot\frac{x_h^4\left(\alpha + x_h^2\beta\right)}{\frac9{10}x_l^4}\left(1 + \frac1{\left(1 + 10\frac{x_h^2}{x_l^2}\frac\nu\beta\right)^{\gamma_2 - \gamma_1}}\right)\right)^{\frac1{2 - \gamma_1-\gamma_2}}, 
\\& \left(\frac{4\nu}{\epsilon}\cdot\frac{\left(x_h^2\alpha + x_h^4\beta\right)^2}{\frac9{10}x_l^6\beta}\left(1 + \frac1{\left(1 + 10\frac{x_h^2}{x_l^2}\frac\nu\beta\right)^{\gamma_2 - \gamma_1}}\right)^2\right)^\frac1{3 - 2\gamma_2}\Big\}
\\ n_3 & = \max\left\{1 +  10\frac{x_h^2}{x_l^2}\frac{\nu_1}{\beta}, 1 + \nu_1\frac{x_h^2}{x_l^2}\right\} + 1
\\ n_p & = \max\left\{n_1, n_2, n_3, \left(\frac{2v_1}{\alpha}\left(\frac{32x_h^2\beta}{3}\right)^2\left(\epsilon' - \ln\left(0.5 - \delta\right)\right)e^{\frac2{x_h^2\beta}}\right)^{\frac1{2\left(\gamma_2 - 1\right)}}\right\}
\\ v_1 & = \max\left\{6, 1 + 2e^{\frac1{x_h^2\beta}}\right\}
\\ c_p & = n_p^{\gamma_2}.
\end{align*}

By looking at the \textit{Bayesian linear regression problem} on domain $\mathcal{D}(n_p, \gamma_1, x_h, x_l, c_p)$, we next show that sampling from the posterior is $(\epsilon, \delta)$ differentially private, although there is an SGLD step for which approximate sampling from the posterior using SGLD is not $(\epsilon', \delta)$ differentially private.

Given dataset $D_1 \in \mathcal{D}(n_p, \gamma_1, x_h, x_l, c_p)$ (as defined in eq. \ref{app:N-DBs}, with $n = n_p$ and $c=c_p$), as $n_p \geq n_1$, the problem holds the constraints of Lemma $\ref{Lemma-SGLD-3}$. Consequently, there exists an SGLD step that is not $(\epsilon'', \delta)$ private for all $\epsilon'' \leq e^{-\frac2{x_h^2\beta}}\frac{\alpha}{2v_1}(\frac{3}{32x_h^2\beta})^2(\frac{c_p}{n_p})^2 + \ln(0.5 - \delta)$. From eq. \ref{eq-theorem1-proof}, the choice of $n_p$ promises that $\epsilon' \leq e^{-\frac2{x_h^2\beta}}\frac{\alpha}{2v_1}(\frac{3}{32x_h^2\beta})^2(\frac{c_p}{n_p})^2 + \ln(0.5 - \delta)$. Therefore, approximate sampling from the posterior using SGLD is not $(\epsilon', \delta)$ differentially private.

Since $n_p \geq n_2$ and $n_p \geq n_3$, the problem holds the constraints of Claim \ref{Claim-Post-9}; therefore, one sample from the posterior is $(\epsilon, \delta)$ differentially private.

\begin{equation}\label{eq-theorem1-proof}
\begin{split}
{}& e^{-\frac2{x_h^2\beta}}\frac{\alpha}{2v_1}\left(\frac{3}{32x_h^2\beta}\right)^2\left(\frac{c_p}{n_p}\right)^2 + \ln\left(0.5 - \delta\right) \geq \epsilon'
\\{}& \left(\frac{c_p}{n_p}\right)^2 \geq \left(\epsilon' - \ln\left(0.5 - \delta\right)\right)e^{\frac2{x_h^2\beta}}\left(\frac{32x_h^2\beta}{3}\right)^2\frac{2v_1}{\alpha}
\\{}& n_p^{2\left(\gamma_2 - 1\right)} \geq \left(\epsilon' - \ln\left(0.5 - \delta\right)\right)e^{\frac2{x_h^2\beta}}\left(\frac{32x_h^2\beta}{3}\right)^2\frac{2v_1}{\alpha}
\\{}& n_p \geq \left(\left(\epsilon' - \ln\left(0.5 - \delta\right)\right)e^{\frac2{x_h^2\beta}}\left(\frac{32x_h^2\beta}{3}\right)^2\frac{2v_1}{\alpha}\right)^{\frac1{2\left(\gamma_2 - 1\right)}}
\end{split}
\end{equation}
\end{proof}

\subsection{Posterior Sampling Privacy}\label{subseq-PSP}
This subsection provides proofs for the lemmas provided in subsection \ref{PosteriorSamplingPrivacySection}, along with a supporting lemma.

\begin{proof}[Proof of Lemma \ref{PosteriorDistribution}]
Eq. \ref{Eq-App-BishopDistribution} is a known result for the Bayesian inference problem for a linear 1D model with Gaussian noise with a known precision parameter $(\beta)$ and a conjugate prior (see \citet{Bishop} - 3.49-4.51. for details). By choosing the basis function to be $\phi(x)=x$, working in one dimension, and choosing $\vm_0 = 0, \mS_0 = \alpha^{-1}$, we get the linear model defined in eq. \ref{LinearModel} and the matching posterior described in Lemma \ref{PosteriorDistribution}. 

\begin{equation}
\label{Eq-App-BishopDistribution}
p(\theta|D) = \mathcal{N}(\theta; m_N, S_N);\ m_N = S_N(S_0^{-1}m_0 + \beta\Phi^T t);\  S_N^{-1} = S_0^{-1} + \beta\Phi^T\Phi
\end{equation}
\end{proof}

\begin{lemma}
\label{PosteriorRenyiFull}
For a Bayesian linear regression problem on domain $\mathcal{D}(n, \gamma_1, x_h, x_l, c)$, such that $n > \max\{1 +  10\frac{x_h^2}{x_l^2}\frac{\nu}{\beta}, 1 + \nu\frac{x_h^2}{x_l^2}\}$, one sample from the posterior is $(\nu, \epsilon_1)$-R\'enyi differentially private, and $\epsilon_1$ is
\begin{align*}
\epsilon_1 & = \frac{x_h^2}{2(n - 1)x_l^2} + \frac{\left(\nu-1\right)\nu x_h^2}{2\left(\left(n - 1\right)x_l^2 - \nu x_h^2\right)}
+ \frac{20\nu\beta x_h^4}{9n^{1 - 2\gamma_1}x_l^2} \\ & +
\frac{20\nu x_h^4(\alpha + x_h^2\beta)}{9x_l^4}\cdot\frac{(c + n^{\gamma_1})}{n^{2-\gamma_1}}
+ \frac{5\nu x_h^4(\alpha + x_h^2\beta)^2}{9x_l^6\beta}\cdot\frac{(c + n^{\gamma_1})^2}{n^3}.
\end{align*}
\end{lemma}

\begin{proof}[Proof of Lemma \ref{PosteriorRenyiFull}]
By Definition \ref{Def-Renyi-DP}, for a single sample from the posterior to be $(\nu, \epsilon')$ RDP, the R\'enyi divergence of order $\nu$ between any adjacent datasets needs to be bounded. Therefore, we consider two adjacent datasets, $D, \hat{D} \in \mathcal{D}(n, \gamma_1, x_h, x_l, c)$, and w.l.o.g, define that they differ in the last sample (where it is also allowed to be $(0, 0)$ for one of them, which saves us the need to consider also a neighbouring dataset with a size smaller by 1). To ease the already complex and detailed calculations, we use definitions in eq. \ref{Eq-App-PosteriorAdjacentDBs}.

\begin{equation}
\label{Eq-App-PosteriorAdjacentDBs}
\begin{split}
{}& D = \{x_i,y_i\}_{i=1}^{n - 1}\cup\{x_n, y_n\},\ \hat{D} = \{x_i,y_i\}_{i=1}^{n-1}\cup\{\hat{x}_n, \hat{y}_n\}\\
{}& z = \sum_{i=1}^{n-1}x_i^2,\ q = \sum_{i=1}^{n-1}y_i x_i
\end{split}
\end{equation}

According to Lemma \ref{PosteriorDistribution} and with the definitions of eq. \ref{Eq-App-PosteriorAdjacentDBs}, the posterior distributions are

\begin{equation}
\label{Eq-App-PosteriorDis2}
\begin{split}
{}& p(\theta|D) = \mathcal{N}(\theta; \mu, \sigma^2);\ \mu = \frac{\beta(q + x_ny_n)}{\alpha + (z + x_n^2)\beta};\ \sigma^2 = \frac{1}{\alpha + (z + x_n)\beta}\\
{}& p(\theta|\hat{D}) = \mathcal{N}(\theta; \hat\mu, \hat\sigma^2);\ \hat\mu = \frac{\beta(q + \hat{x}_n\hat{y}_n)}{\alpha + (z + \hat{x}_n^2)\beta};\ \hat\sigma^2 = \frac{1}{\alpha + (z + \hat{x}_n)\beta}.
\end{split}
\end{equation}

Mark by $\mathrm{D}_\nu(f_1 || f_2)$ the R\'eyni divergence of order $\nu$ between $f_1$ and $f_2$ - uni-variate normal distributions with means $\mu_1, \mu_2$ and variances $\sigma_1, \sigma_2$ accordingly. By \citet{RenyiDivergence}, $\mathrm{D}_\nu(f_1 || f_2)$ is

\begin{equation*}
\begin{split}
{}& \mathrm{D}_\nu\left(f_1 || f_2\right) = \ln\frac{\sigma_1}{\sigma_2} + \frac12\left(\nu - 1\right)\ln\frac{\sigma_2^2}{\left(\sigma^2_{f_1,f_2}\right)^*_\nu} + \frac12\frac{\nu\left(\mu_1 - \mu_2\right)^2}{\left(\sigma^2\right)^*_\nu} \\
{}& \left(\sigma^2_{f_1,f_2}\right)^*_\nu = \nu\sigma_2^2 + \left(1 - \nu\right)\sigma_1^2 > 0.
\end{split}
\end{equation*}

Therefore, for $p(\theta|D)$ and $p(\theta|\hat{D})$, the R\'enyi divergence of order $\nu$ is given in eq. \ref{eq-renyi-div-post}, where we omit the subscript for $(\sigma^2)^*_\nu$ since it is clear from context to which distributions it applies. 
\begin{equation}\label{eq-renyi-div-post}\begin{split}
{}& \mathrm{D}_\nu\left(p\left(\theta|D\right) || p(\theta|\hat{D})\right) = \ln\frac{\sigma}{\hat\sigma} + \frac12\left(\nu - 1\right)\ln\frac{\hat\sigma^2}{\left(\sigma^2\right)^*_\nu} + \frac12\frac{\nu\left(\mu - \hat\mu\right)^2}{\left(\sigma^2\right)^*_\nu} \\
{}& \left(\sigma^2\right)^*_\nu = \nu\hat\sigma^2 + \left(1 - \nu\right)\sigma^2
\end{split}\end{equation}

According to Claim \ref{claim-Post-Div-2}, $(\sigma^2)^*_\nu > 0$; therefore, the value $\mathrm{D}_\nu(p(\theta|D), p(\theta|\hat{D}))$ exists. In order to prove R\'enyi differential privacy, each of the terms of $\mathrm{D}_\nu(p(\theta|D), p(\theta|\hat{D}))$ is bounded separately, so their sum will be equal to $\epsilon_1$. The bounds for each of the terms are proved at Claims \ref{claim-Post-Div-3}, \ref{claim-Post-Div-4}, and \ref{claim-Post-Div-5}.
\end{proof}

\begin{proof}[Proof of Lemma \ref{PosteriorRenyi}]
The Lemma is a direct corollary of Lemma \ref{PosteriorRenyiFull}
\end{proof}

\begin{proof}[Proof of Lemma \ref{PosteriorPrivacy}]
By Lemma \ref{PosteriorRenyiFull}, sampling from the posterior is $(\nu, \epsilon_1)$-RDP; therefore, by Lemma \ref{Lemma-RDP-to-ADP}, sampling from the posterior is also $(\epsilon_1 + \frac{\ln(\frac1\delta)}{\nu - 1}, \delta)$ differentially private.
\end{proof}

\subsection{Stochastic Gradient Langevin Dynamics Privacy}
This subsection provides proofs for the lemmas presented in subsection \ref{subseq-SGLD}. The proofs in this section rely heavily on the analysis of the SGLD behaviour for the \textit{Bayesian Linear Regression Problem} on datasets $D_1$, $D_2$. This analysis is provided in subsection \ref{subseq::sgld-detailed}.

\begin{proof}[Proof of Lemma \ref{SGLD-0}]
\begin{equation*}\label{proof::sgld-0}\begin{split}
p(\hat\theta_{j} > \mu_{j}|{D}_2) & =
\sum_{r=1}^{n}p(\hat\theta_{j}^r > \mu_{j}|{D}_2)p(\hat{\theta}_{j} = \hat{\theta}^{r}_{j}|{D}_2) = 
\sum_{r=1}^{n}p(\hat\theta_{j} - \hat{\mu}^{r}_{j} > \mu_{j} - \hat\mu^{r}_{j}|{D}_2)p(\hat{\theta}_{j} = \hat{\theta}^{r}_{j}|{D}_2)
\\& = \frac1n\sum_{r=1}^{n}p(\hat\theta_{j} - \hat\mu^{r}_{j} > \mu_{j} - \hat\mu^{r}_{j}|{D}_2) \leq
\frac1n\sum_{r=1}^{n}\exp\left({-\frac{(\mu_{j} - \hat\mu^{r}_{j})^2}{2(\hat{\sigma}^{r}_{j})^2}}\right)
\end{split}\end{equation*}
where the inequality holds due to Chernoff bound (For further details about Chernoff bound, see \citet{ChernoffBound}).
\end{proof}

\begin{proof}[Proof of Lemma \ref{SGLD-1}] 
By Lemma \ref{Lemma-SGLD-2.4}, for $n > \max\{\frac\alpha{x_h^2\beta}, \frac{\alpha}{x_h^2\beta}(e^{\frac2{x^2\beta}} - 2) + \frac1{2x_h^2\beta}, \frac{1}{2\alpha x_h^2\beta} - \frac{1}{x_h^2\beta}\}$ and $\dot{k} \in \R_{>0}$, eq. \ref{eq-proof-sgld-1} holds. We can see that the lower bound described in eq. \ref{eq-proof-sgld-1} is dominated by $\frac{c^2}{n^2}$, thus proving Lemma \ref{SGLD-1}.
\begin{equation}\label{eq-proof-sgld-1}\begin{split}
{}& \frac{(\mu_{(\lceil\dot{k}\rceil + 1)n} - \hat\mu_{(\lceil\dot{k}\rceil + 1)n}^r)^2}{(\hat{\sigma}^{r}_{(\lceil\dot{k}\rceil + 1)n})^2} \geq e^{-\frac2{x_h^2\beta}}\frac{\alpha}{v_1}(\frac{3}{32x_h^2\beta})^2(\frac{c}{n})^2
\\{}& v_1 = \max\{6, 1 + 2e^{\frac1{x_h^2\beta}}\}
\end{split}
\end{equation}
\end{proof}

\begin{proof}[Proof of Lemma \ref{SGLD-2}]
Define $\epsilon'$ as
\begin{equation*}\label{eq-sgld-2-proof}
\begin{split}
{}& \epsilon' = e^{-\frac2{x_h^2\beta}}\frac{\alpha}{2v_1}\left(\frac{3}{32x_h^2\beta}\right)^2\left(\frac{c}{n}\right)^2 + \ln\left(0.5 - \delta\right)
\\{}& v_1 = \max\{6, 1 + 2e^{\frac1{x_h^2\beta}}\}.
\end{split}
\end{equation*}
By Lemma \ref{Lemma-SGLD-3}, for $n > \max\{\frac\alpha{x_h^2\beta},  \frac{\alpha}{x_h^2\beta}(e^{\frac2{x_h^2\beta}} - 2) + \frac1{2x_h^2\beta}, \frac{1}{2\alpha x_h^2\beta} - \frac{1}{x_h^2\beta}\}$, 
there exists $T \in \mathbb{Z}_{>0}$ (marked in Lemma \ref{Lemma-SGLD-3} as $(\lceil\dot{k}\rceil + 1)n$) such that running SGLD for the \textit{Bayesian linear regression problem} over $D_1$ for $T$ steps will not be $(\epsilon, \delta)$ differentially private for $\epsilon < \epsilon'$. As $\epsilon'$ is dominated by $\frac{c^2}{n^2}$, Lemma \ref{SGLD-2} is proved.
\end{proof}
\newpage
\subsection{Stochastic Gradient Langevin Dynamics Detailed Analysis}\label{subseq::sgld-detailed}
This subsection provides an analysis of SGLD behaviour for the \textit{Bayesian Linear Regression Problem} on datasets $D_1$, $D_2$. We advise the reader to read the lemmas by order and provide here a summary of the analysis: Lemma \ref{Lem-App-SGLD-1} provides an expression for the sample at the $(k+1)n$ SGLD step when using datasets $D_1, D_2$. Lemmas \ref{Lem-App-SGLD-2} and \ref{Lemma-SGLD-2.2} use this expression to get a lower bound on the difference in means and an upper bounds on the variance, respectively. In turn, Lemma \ref{Lemma-SGLD-2.4} uses these lower and upper bounds to find a lower bound over $ {\left(\mu_{{\lceil\dot{k}\rceil}n+n} - \hat\mu_{{\lceil\dot{k}\rceil}n+n}^r\right)^2}/{\left(\hat{\sigma}^{r}_{{\lceil\dot{k}\rceil}n+n}\right)^2}$. This lower bound is used by Lemma \ref{P1} to upper bound the probability mass of the SGLD process running on dataset $D_2$ in $S = \{s| s > \mu_j\}$. The difference in probability masses in $S$ between the weights of an SGLD running on datasets $D1$ and $D_2$ leads to a breach of privacy, as shown in Lemma \ref{Lemma-SGLD-3}.

In order to ease the analysis of the SGLD process for the \textit{Bayesian linear regression problem} on domain $\mathcal{D}(n, \gamma_1, x_h, x_l, c)$, we use the markings in eq. \ref{Eqq-App-SGLD-markings}.
\begin{equation}\label{Eqq-App-SGLD-markings}
\lambda = \left(1 - \frac{\eta}{2}\left(\alpha + nx_h^2\beta\right)\right), \hat{\lambda} = \left(1 - \frac{\eta}{2}\left(\alpha + n\left(\frac{x_h}{2}\right)^2\beta\right)\right) , \rho = \frac{\eta}{2}ncx_h^2\beta, \hat{\rho} = \frac{\eta}{2}nc\left(\frac{x_h}{2}\right)^2\beta
\end{equation}

\begin{lemma}
\label{Lem-App-SGLD-1}
$\forall k \in \mathbb{Z}_{> 0}$, $\hat{\theta}^r_{(k+1)n}$ has the following forms: 

\begin{equation*}\begin{split}
{}& \hat{\theta}^1_{(k+1)n} = \theta_0\hat\lambda^{k+1}\lambda^{(n-1)(k+1)} + \sum_{j=0}^{k}\left(\hat\lambda\lambda^{n-1}\right)^j\left(\hat\rho\lambda^{n-1} + \rho\sum_{i=0}^{n-2}\lambda^i + \sqrt\eta\sum_{i=0}^{n-1}\lambda^i\xi_i\right) \\
{}& \hat{\theta}^{r > 1}_{(k+1)n} = 
\theta_0\left(\hat\lambda\lambda^{n-1}\right)^{k+1} + \left(\sum_{i=1}^{r-1}\left(\rho + \sqrt\eta\xi\right)\hat\lambda\lambda^{n-i-1} + \left(\hat\rho + \sqrt\eta\xi\right)\lambda^{n-r} + \sum_{j=r+1}^n\left(\rho + \sqrt\xi\eta\right)\lambda^{n-j}\right)\sum_{l=0}^{k}\left(\hat\lambda\lambda^{n-1}\right)^l.
\end{split}\end{equation*}
\end{lemma}

\begin{proof}[Proof of Lemma \ref{Lem-App-SGLD-1}]
We can apply the SGLD update rule, defined in eq. \ref{Eq-SGLD-Update-Rule}, to the \textit{Bayesian linear regression problem} over datasets $D_1$ and  $D_2$ as follows: First, $p(\theta_j) = \mathbb{N}(\theta_j; 0, \alpha^{-1})$, and therefore
$$\nabla_{\theta_j}\ln{p(\theta_j)} = \nabla_{\theta_j}\ln\left(\frac{1}{\sqrt{2\pi\alpha^{-1}}}\right) - \nabla_{\theta_j}\frac{1}{2}\theta_j^2\alpha = -\theta_j\alpha.$$
In a similar manner, 
$$
\nabla_{\theta_j}\ln{p(y_i|\theta_j)} = \nabla_{\theta_j}\ln\left(\frac{1}{\sqrt{2\pi\beta^{-1}}}\right) - \nabla_{\theta_j}\frac{1}{2}(y_i-\theta_jx_i)^2\beta = (y_i-\theta_jx_i)x_i\beta.
$$
Inserting these expressions to the SGLD update rule yields 
\begin{equation}\begin{split}\label{Eq-App-SGLD-Step-E}
\theta_{j+1} & =
\theta_j + \frac\eta2\left(-\theta_j\alpha + n\left(y_i - \theta_jx_j\right)x_i\beta\right) + \sqrt\eta_j\xi_i = 
\theta_j\left(1 - \frac\eta2\left(\alpha + nx_j^2\beta\right)\right) + \frac\eta2ny_ix_i\beta + \sqrt\eta\xi_j
\\& = \theta_j\left(1 - \frac\eta2\left(\alpha + nx_j^2\beta\right)\right) + \frac\eta2ncx_i^2\beta + \sqrt\eta\xi_j.
\end{split}\end{equation}

By using standard tools for solving first-order non-homogeneous recurrence relations with variable coefficients, the value of $\hat\theta^1_n$ can be found:
\begin{equation*}\begin{split}
\hat\theta_n^1 & = \hat\lambda\lambda^{n-1}\left(\frac{\theta_0\hat\lambda + \hat\rho + \sqrt\eta\xi}{\hat\lambda} + \sum_{i=2}^n\frac{\rho + \sqrt\eta\xi}{\hat\lambda\lambda^{i-1}}\right) =
\theta_0\hat\lambda\lambda^{n-1} + \left(\hat\rho+\sqrt\eta\xi\right)\lambda^{n-1} + \left(\rho + \sqrt\eta\xi\right)\sum_{i=2}^n\lambda^{n-1-\left(i-1\right)} \\
& = \theta_0\hat\lambda\lambda^{n-1} + \left(\hat\rho + \sqrt\eta\xi\right)\lambda^{n-1} + \left(\rho + \sqrt\eta\xi\right)\sum_{i=2}^n\lambda^{n-1-\left(i-1\right)} =
\theta_0\hat\lambda\lambda^{n-1} + \left(\hat\rho + \sqrt\eta\xi\right)\lambda^{n-1} + \left(\rho + \sqrt\eta\xi\right)\sum_{i=0}^{n-2}\lambda^{i} \\
& = \theta_0\hat\lambda\lambda^{n-1} + \hat\rho\lambda^{n-1} + \rho\sum_{i=0}^{n-2}\lambda^{i} + \sqrt\eta\xi\sum_{i=0}^{n-1}\lambda^i.
\end{split}\end{equation*}
Thus, we can define a new series, $\hat\theta^1_{(k+1)n} = c_1\hat\theta^1_{kn} + c_2$ where $c_1 = \hat\lambda\lambda^{n-1}, c_2 = \hat\rho\lambda^{n-1} + \rho\sum_{i=0}^{n-2}\lambda^{i} + \sqrt\eta\xi\sum_{i=0}^{n-1}\lambda^i$. Using tools for solving first order non-homogeneous recurrence relations with constant coefficients, the value of $\hat\theta^1_{kn}$ can be found:
\begin{equation*}\begin{split}
\hat\theta^1_{kn} & = c_1^k\left(\frac{\hat\theta^1_n}{c_1} + \sum_{i=2}^k\frac{c_2}{c_1^{i}}\right) =
\theta^1_nc_1^{k-1} + \sum_{i=2}^{k}c_2c_1^{k-i} =
\theta^1_nc_1^{k-1} + c_2\sum_{i=0}^{k-2}c_1^{i} =
\left(\theta_0c_1 + c_2\right)c_1^{k-1} + c_2\sum_{i=0}^{k-2}c_1^i \\
& = \theta_0\left(\hat\lambda\lambda^{n-1}\right)^k + \left(\hat\rho\lambda^{n-1} + \rho\sum_{i=0}^{n-2}\lambda^i + \sqrt\eta\xi\sum_{i=0}^{n-1}\lambda^i\right)\sum_{j=0}^{k-1}\left(\hat\lambda\lambda^{n-1}\right)^j.
\end{split}\end{equation*}
The proof for $\hat\theta_{kn}^r$ is done in a similar manner:
\begin{equation*}\begin{split}
\hat{\theta}_n^{r > 1} & = \hat\lambda\lambda^{n-1}\left(\theta_0 + \sum_{i=1}^{r-1}\frac{\rho + \sqrt{\eta}\xi}{\lambda^i} + \frac{\hat\rho + \sqrt\eta\xi}{\lambda^{r-1}\hat\lambda} + \sum_{j = r + 1}^{n}\frac{\rho + \sqrt\eta\xi}{\lambda^{j-1}\hat\lambda}\right)
\\& = \hat\lambda\lambda^{n-1}\theta_0 + \sum_{i=1}^{r-1}\left(\rho + \sqrt{\eta}\xi\right)\hat\lambda\lambda^{n-i-1} + \lambda^{n-r}\left(\hat\rho + \sqrt\eta\xi\right) + \sum_{j = r + 1}^{n}\left(\rho + \sqrt\eta\xi\right)\lambda^{n - j}
\end{split}\end{equation*}
Thus, we can define a new series, $\hat\theta^{r>1}_{(k+1)n} = c_3\hat\theta^{r>1}_{kn} + c_4$, where $c_3 = \hat\lambda\lambda^{n-1}, c_4 = \sum_{i=1}^{r-1}\left(\rho + \sqrt{\eta}\xi\right)\hat\lambda\lambda^{n-i-1} + \lambda^{n-r}\left(\hat\rho + \sqrt\eta\xi\right) + \sum_{j = r + 1}^{n}\left(\rho + \sqrt\eta\xi\right)\lambda^{n - j}$. Similarly to $\hat\theta^{1}_{kn}$ derivation, we can use first order non-homogeneous recurrence relations with constant coefficients to get:
\begin{equation*}
\hat{\theta}^{r>1}_{kn} = \theta_0\left(\hat\lambda\lambda^{n-1}\right)^{k} + \left(\sum_{i=1}^{r-1}\left(\rho + \sqrt\eta\xi\right)\hat\lambda\lambda^{n-i-1} + \left(\hat\rho + \sqrt\eta\xi\right)\lambda^{n-r} + \sum_{j=r+1}^n\left(\rho + \sqrt\xi\eta\right)\lambda^{n-j}\right)\sum_{l=0}^{k-1}\left(\hat\lambda\lambda^{n-1}\right)^l.
\end{equation*}
\end{proof}
\begin{lemma}\label{lem-app-sgld-theta}
$\forall k \in \mathbb{Z}_{> 0}$, $\theta_{(k + 1)n}$ has the following form:
$$
\theta_{(k+1)n} = \theta_0\lambda^{(k+1)n} + \sum_{j = 0}^k\lambda^{jn}\left(\rho\lambda^{n-1} + \rho\sum_{i=0}^{n-2}\lambda^i + \sqrt{\eta}\sum_{i = 0}^{n-1}\lambda^i\xi_i\right)
$$
\end{lemma}
\begin{proof}[Proof of Lemma \ref{lem-app-sgld-theta}]
First, notice that equation \ref{Eq-App-SGLD-Step-E} applies for SGLD using dataset $D_1$. By using standard tools for solving first-order non-homogeneous recurrence relations, an expression for $\theta_{n}$ can be found:
\begin{align*}
\theta_{n} = \lambda^{n}\left(\frac{\theta_0\lambda + \rho + \sqrt{\eta}\xi}{\lambda} + \sum_{i=2}^n\frac{\rho + \sqrt{\eta}\xi}{\lambda^{i}}\right) = \theta_0\lambda^n + \rho\lambda^{n-1} + \rho\sum_{i=0}^{n-2}\lambda^i + \sqrt{\rho}\xi\sum_{i=0}^{n-1}\lambda^i
\end{align*}
By defining a new series $\theta_{(k+1)n} = c_5\theta_{kn} + c_6$ and using tools for solving first-order non-homogeneous recurrence relations with constant coefficients, the value of $\theta_{kn}$ can be found:
\begin{align*}
\theta_{kn} = c_5^k\left(\frac{\theta_n}{c_5} + \sum_{i=2}^k\frac{c_6}{c_5}\right) = \theta_0\lambda^{nk} + \left(\rho\lambda^{n-1} + \rho\sum_{i=0}^{n-2}\lambda^i + \sqrt{\eta}\xi\sum_{i=0}^{n-1}\lambda^i\right)\sum_{j=0}^{k-1}\lambda^{jn}
\end{align*}
\end{proof}

\begin{lemma}\label{Lem-App-SGLD-2}
$\forall k\in\mathbb{Z}_{>0}$, the value of $\mu_{kn + n} - \hat\mu^r_{kn + n}$ can be lower bounded:
$$\mu_{kn + n} - \hat\mu^r_{kn + n} \geq \lambda^{n-1}\frac{ncx_h^2\beta}{\alpha + nx_h^2\beta}\lambda^{k(n-1)}\left(\hat\lambda^{k+1} - \lambda^{k+1}\right).$$\end{lemma}
\begin{proof}[Proof of Lemma \ref{Lem-App-SGLD-2}] The proof of this lemma is separated into two cases, for $r = 1$ and for $r > 1$. For $r = 1$, using $\E[\theta_0] = 0$ and $\E[\xi] = 0$, it is easy to derive eq. \ref{Eq-App-mean_r1} from Lemmas \ref{Lem-App-SGLD-1} and \ref{lem-app-sgld-theta}.
\begin{equation}\begin{split}\label{Eq-App-mean_r1}
{}&\hat\mu^1_{(k+1)n} =
\rho\sum_{i=0}^{n-2}\lambda^i\sum_{j=0}^{k}\left(\hat\lambda\lambda^{(n-1)}\right)^j + \hat\rho\lambda^{n-1}\sum_{j=0}^{k}\left(\hat\lambda\lambda^{(n-1)}\right)^j
\\{}&\mu_{kn+n} = \rho\sum_{i=0}^{n-2}\lambda^i\sum_{j=0}^{k}\lambda^{jn} + \rho\lambda^{n-1}\sum_{r=0}^{k}\lambda^{rn}
\end{split}\end{equation}
We use the sum of a geometric sequence to get
\begin{align*}
\hat\mu^1_{(k+1)n} = \rho\sum_{i=0}^{n-2}\lambda^i\sum_{j=0}^{k}\left(\hat\lambda\lambda^{\left(n-1\right)}\right)^j + \hat\rho\lambda^{n-1}\sum_{j=0}^{k}\left(\hat\lambda\lambda^{\left(n-1\right)}\right)^j = \left(\rho\left(\frac{1-\lambda^{n-1}}{1 -\lambda}\right) + \hat\rho\lambda^{n-1}\right)\frac{1 - \left(\hat\lambda\lambda^{n-1}\right)^{k+1}}{1 - \hat\lambda\lambda^{n-1}}
\end{align*}
and
\begin{align*}
\mu_{(k+1)n} = \rho\sum_{i=0}^{n-2}\lambda^i\sum_{j=0}^{k}\lambda^{nj} + \rho\lambda^{n-1}\sum_{j=0}^{k}\lambda^{nj} = \left(\rho\left(\frac{1-\lambda^{n-1}}{1 -\lambda}\right) + \rho\lambda^{n-1}\right)\frac{1 - \lambda^{(k + 1)n}}{1 - \lambda^{n}}.
\end{align*}
Therefore the difference between the means can be lower bounded:
\begin{equation*}\begin{split}
\mu_{kn + n} - \hat\mu^1_{kn+n} & =
\frac{1 - \lambda^{\left(k+1\right)n}}{1 - \lambda^n}\left(\rho\left(\frac{1-\lambda^{n-1}}{1 - \lambda}\right) + \rho\lambda^{n-1}\right) - \frac{1 - \left(\hat\lambda\lambda^{\left(n-1\right)}\right)^{k+1}}{1 - \hat\lambda\lambda^{n-1}}\left(\rho\left(\frac{1-\lambda^{n-1}}{1-\lambda}\right) + \hat\rho\lambda^{n-1}\right) 
\\& =^* \frac{1 - \lambda^{\left(k+1\right)n}}{1 - \lambda^n}\frac{ncx_h^2\beta}{\alpha + nx_h^2\beta}\left(1 - \lambda^n\right) - 
\frac{1 - \left(\hat\lambda\lambda^{\left(n-1\right)}\right)^{k+1}}{1 - \hat\lambda\lambda^{n-1}}\frac{ncx_h^2\beta}{\alpha + nx_h^2\beta}\left(1 - \lambda^{n-1}\left(\frac14\lambda +\frac34\right)\right)
\\& = \left(1 - \lambda^{\left(k+1\right)n}\right)\frac{ncx_h^2\beta}{\alpha + nx_h^2\beta} - \frac{1 - \left(\hat\lambda\lambda^{\left(n-1\right)}\right)^{k+1}}{1 - \hat\lambda\lambda^{n-1}}\frac{ncx_h^2\beta}{\alpha + nx_h^2\beta}\left(1 - \lambda^{n-1}\left(\frac14\lambda +\frac34\right)\right)
\\&= \frac{ncx_h^2\beta}{\alpha + nx_h^2\beta}\left(\left(1 - \lambda^{\left(k+1\right)n}\right) - \frac{1 - \left(\hat\lambda\lambda^{\left(n-1\right)}\right)^{k+1}}{1 - \hat\lambda\lambda^{n-1}}\left(1 - \lambda^{n-1}\left(\frac14\lambda +\frac34\right)\right)\right)
\\&=^{**} \frac{ncx_h^2\beta}{\alpha + nx_h^2\beta}\left(\frac{\lambda^{n-1}\frac34\frac\eta2\alpha\left(1 - \hat\lambda^{k+1}\lambda^{\left(k+1\right)\left(n-1\right)}\right)  + \lambda^{\left(k+1\right)\left(n-1\right)}\left(\hat\lambda^{k+1} - \lambda^{k+1}\right)\left(1 - \lambda^{n-1}\hat\lambda\right)}{1 - \lambda^{n-1}\hat\lambda}\right) 
\\&\geq \frac{ncx_h^2\beta}{\alpha + nx_h^2\beta}\left(\frac{\lambda^{\left(k+1\right)\left(n-1\right)}\left(\hat\lambda^{k+1} - \lambda^{k+1}\right)\left(1 - \lambda^{n-1}\hat\lambda\right)}{1 - \lambda^{n-1}\hat\lambda}\right)
\\&= \frac{ncx_h^2\beta}{\alpha + nx_h^2\beta}\lambda^{\left(k+1\right)\left(n-1\right)}\left(\hat\lambda^{k+1} - \lambda^{k+1}\right)
\\&= \lambda^{n-1}\frac{ncx_h^2\beta}{\alpha + nx_h^2\beta}\lambda^{k\left(n-1\right)}\left(\hat\lambda^{k+1} - \lambda^{k+1}\right)
\end{split}\end{equation*}
where equality * holds from Claims \ref{M1}, \ref{M2}, and \ref{M3}, equality ** holds from Claim \ref{M5.f}, and the inequality holds because $\lambda < \hat\lambda < 1$. This proves Lemma \ref{Lem-App-SGLD-2} for $r=1$.

For $r > 1$, from Lemma \ref{Lem-App-SGLD-1}, it is easy to see that
\begin{align*}
\hat\theta_{(k+1)n}^{r > 1} & = 
\left(\left(\theta_0\lambda^{r-1} + \rho\sum_{i=0}^{r-2}\lambda^{i} + \sqrt\eta\sum_{i=0}^{r-2}\lambda^i\xi_i\right)\hat\lambda^{k}\lambda^{k(n-1)} +
\sum_{j=0}^{k-1}(\hat\lambda\lambda^{n-1})^j\left(\hat\rho\lambda^{n-1} + \rho\sum_{i=0}^{n-2}\lambda^{i} + \sqrt\eta\sum_{i=0}^{n-1}\lambda^i\xi_i\right)\right)\hat\lambda\lambda^{n-r}
\\&+ \hat\rho\lambda^{n-r} + \rho\sum_{j=0}^{n-r-1}\lambda^j + \sqrt\eta\sum_{j=0}^{n-r}\xi\lambda^j.
\end{align*}
Therefore, 
\begin{align*}
{}& \hat\mu_{(k+1)n}^{r > 1} = \left(\left(\rho\sum_{i=0}^{r-2}\lambda^{i}\right)\hat\lambda^{k}\lambda^{k(n-1)} +
\sum_{j=0}^{k-1}(\hat\lambda\lambda^{n-1})^j\left(\hat\rho\lambda^{n-1} + \rho\sum_{i=0}^{n-2}\lambda^{i}\right)\right)\hat\lambda\lambda^{n-r} + 
\hat\rho\lambda^{n-r} + \rho\sum_{j=0}^{n-r-1}\lambda^j.
\end{align*}
Consequently, the difference between the means, for $r > 1$, can be lower bounded:
\begin{align*}
\mu_{kn + n} - \hat\mu^r_{kn + n} & =
\lambda^{n-r}\big(\lambda\rho\lambda^k\lambda^{k(n-1)}\sum_{i=0}^{r-2}\lambda^i + \lambda\sum_{j=0}^{k-1}(\lambda\lambda^{n-1})^j(\rho\lambda^{n-1} + \rho\sum_{i=0}^{n-2}\lambda^i)
\\& - \hat\lambda\rho\hat\lambda^k\lambda^{k(n-1)}\sum_{i=0}^{r-2}\lambda^i - \hat\lambda\sum_{j=0}^{k-1}(\hat\lambda\lambda^{n-1})^j(\hat\rho\lambda^{n-1} + \rho\sum_{i=0}^{n-2}\lambda^i)\big) + 
\lambda^{n-r}(\rho - \hat\rho)
\\& = 
\lambda^{n-r}\lambda^{k(n-1)}\rho(\lambda^{k+1} - \hat\lambda^{k+1})\sum_{i=0}^{r-2}\lambda^i +
\\&+ \lambda^{n-r}\sum_{j=0}^{k-1}\lambda^{(n-1)j}(\lambda^{n-1}(\rho\lambda^{j+1} -\hat\rho\hat\lambda^{j+1})  + (\lambda^{j+1} - \hat\lambda^{j+1})\rho\sum_{i=0}^{n-2}\lambda^i) + \lambda^{n-r}(\rho - \hat\rho) 
\\&=^{*}
\lambda^{n-r}\lambda^{k(n-1)}\rho(\lambda^{k+1} - \hat\lambda^{k+1})\frac{1 - \lambda^{r-1}}{1 - \lambda} + \lambda^{n-r}\frac{ncx_h^2\beta}{\alpha + nx_h^2\beta}(\lambda(1- \lambda^{kn})
\\&- \hat\lambda\frac{1 - (\lambda^{n-1}\hat\lambda)^k}{1 - \lambda^{n-1}\hat\lambda}( 1 - \lambda^n(\frac34\lambda^{-1} + \frac14))) + \lambda^{n-r}(\rho - \hat\rho) 
\\&=^{**}
\lambda^{n-r}\lambda^{k(n-1)}\rho(\lambda^{k+1} - \hat\lambda^{k+1})\frac{1 - \lambda^{r-1}}{1 - \lambda} + 
\lambda^{n-r}\frac{ncx_h^2\beta}{\alpha + nx_h^2\beta}((\lambda - \hat\lambda)
\\& + \frac{\lambda^{n-1}(\frac34\frac\eta2\alpha(1 - \hat\lambda^k\lambda^{k(n-1)}))}{1 - \hat\lambda\lambda^{n-1}} + \lambda^{k(n-1)}(\hat\lambda^{k+1} - \lambda^{k+1})) + 
\lambda^{n-r}(\rho - \hat\rho) 
\\&=^{***}
\lambda^{n-r}\frac{ncx_h^2\beta}{\alpha + nx_h^2\beta}\lambda^{k(n-1)}(\lambda^{k+1} - \hat\lambda^{k+1})(1 - \lambda^{r-1}) + \lambda^{n-r}\frac{ncx_h^2\beta}{\alpha + nx_h^2\beta}((\lambda - \hat\lambda) 
\\& + \frac{\lambda^{n-1}(\frac34\frac\eta2\alpha(1 - \hat\lambda^k\lambda^{k(n-1)}))}{1 - \hat\lambda\lambda^{n-1}} + \lambda^{k(n-1)}(\hat\lambda^{k+1} - \lambda^{k+1})) + 
\lambda^{n-r}(\rho - \hat\rho)
\\&= 
\lambda^{n-r}\frac{ncx_h^2\beta}{\alpha + nx_h^2\beta}\lambda^{k(n-1)}(\lambda^{k+1} - \hat\lambda^{k+1})( 1 - \lambda^{r-1} - 1)\\& +
\lambda^{n-r}\frac{ncx_h^2\beta}{\alpha + nx_h^2\beta}(\lambda - \hat\lambda + \frac{\lambda^{n-1}(\frac34\frac\eta2\alpha(1 - \hat\lambda^k\lambda^{k(n-1)}))}{1 - \hat\lambda\lambda^{n-1}}) + \lambda^{n-r}(\rho - \hat\rho)
\\&= \lambda^{n-r}\frac{ncx_h^2\beta}{\alpha + nx_h^2\beta}\lambda^{k(n-1)}(\hat\lambda^{k+1} - \lambda^{k+1})\lambda^{r-1}
\\&+ \lambda^{n-r}\frac{ncx_h^2\beta}{\alpha + nx_h^2\beta}(\lambda - \hat\lambda + \frac{\lambda^{n-1}(\frac34\frac\eta2\alpha(1 - \hat\lambda^k\lambda^{k(n-1)}))}{1 - \hat\lambda\lambda^{n-1}}) + \lambda^{n-r}(\rho - \hat\rho)
\\& =
\lambda^{n-1}\frac{ncx_h^2\beta}{\alpha + nx_h^2\beta}\lambda^{k(n-1)}(\hat\lambda^{k+1} - \lambda^{k+1})
\\& + \lambda^{n-r}\frac{ncx_h^2\beta}{\alpha + nx_h^2\beta}(\lambda - \hat\lambda + \frac{\lambda^{n-1}(\frac34\frac\eta2\alpha(1 - \hat\lambda^k\lambda^{k(n-1)}))}{1 - \hat\lambda\lambda^{n-1}}) + \lambda^{n-r}(\rho - \hat\rho) 
\\& >^{****}\lambda^{n-1}\frac{ncx_h^2\beta}{\alpha + nx_h^2\beta}\lambda^{k(n-1)}(\hat\lambda^{k+1} - \lambda^{k+1})
\end{align*}
where equality * holds from Claims \ref{M4.a} and \ref{M4.b}, equality ** holds from Claim \ref{M5.e}, equality *** holds from Claim \ref{M1}, and inequality **** holds from Claim \ref{M6} and because $\hat\lambda > \lambda$.
\end{proof}

\begin{lemma}\label{Lemma-SGLD-2.2}
For all $k\in\mathbb{Z}_{>0}$ such that $0 < k \leq \frac{1}{2n}\log_{\lambda}(\frac{1}{1 + \frac{1}{\alpha\eta}(1 - \lambda^{2})}) - 1$, and $x_h^2\beta > 3, n > \max\{\frac{1}{2\alpha x_h^2\beta} - \frac{1}{x_h^2\beta}, \frac{\alpha}{x_h^2\beta}(e^{\frac2{x_h^2\beta}} - 2) + \frac1{2x_h^2\beta}\}$, the values of $(\hat{\sigma}^{r}_{(k+1)n})^2$ can be upper bounded as following:
\begin{equation}\begin{split}\label{Eq-App-Lemma-SGLD-2.2}
{}& (\hat{\sigma}^{1}_{(k+1)n})^2 \leq 2(\hat\lambda\lambda^{n-1})^2\frac1\alpha(\hat\lambda\lambda^{n-1})^{2{k}}
\\
{}& (\hat{\sigma}^{r > 1}_{(k + 1)n})^2 \leq 6(\hat\lambda\lambda^{n-r})^2\frac1\alpha(\hat\lambda\lambda^{n-1})^{2{k}}.
\end{split}\end{equation}
\end{lemma}
\begin{proof}[Proof of Lemma \ref{Lemma-SGLD-2.2}]
The proof will be separated into two cases: $r=1$ and $r > 1$. Starting from the case of $r=1$, since the noise and the prior have Normal distributions, $(\hat\sigma_{kn + n}^{1})^2$ could be easily computed from Lemma \ref{Lem-App-SGLD-1}. Eq. \ref{Eq-App-sigma_1} yields a first general upper bound on $(\hat\sigma_{kn + n}^{1})^2$ applicable for all $k\in\mathbb{Z}_{>0}$.

\begin{equation}\begin{split}\label{Eq-App-sigma_1}
(\hat{\sigma}^{1}_{kn+n})^2 & = 
\frac{1}{\alpha}\left(\hat\lambda\lambda^{\left(n-1\right)}\right)^{2\left(k+1\right)} + \eta\sum_{i=0}^{n-1}\lambda^{2i}\sum_{j=0}^{k}\left(\hat\lambda^2\lambda^{2\left(n-1\right)}\right)^j
\\& = 
\left(\hat\lambda\lambda^{\left(n-1\right)}\right)^{2\left(k+1\right)}\left(\frac{1}{\alpha} + \eta\sum_{i=0}^{n-1}\lambda^{2i}\sum_{j=0}^{k}\frac{\left(\hat\lambda\lambda^{\left(n-1\right)}\right)^{2j}}{\left(\hat\lambda\lambda^{\left(n-1\right)}\right)^{2\left(k+1\right)}}\right) 
\\& = 
\left(\hat\lambda\lambda^{\left(n-1\right)}\right)^{2\left(k+1\right)}\left(\frac{1}{\alpha} + \eta\sum_{i=0}^{n-1}\lambda^{2i}\sum_{j=0}^{k}\left(\hat\lambda\lambda^{\left(n-1\right)}\right)^{2\left(j-\left(k+1\right)\right)}\right) 
\\& \leq 
\left(\hat\lambda\lambda^{\left(n-1\right)}\right)^{2\left(k+1\right)}\left(\frac{1}{\alpha} + \eta\sum_{i=0}^{n-1}\lambda^{2i}\sum_{j=0}^{k}\lambda^{2n\left(j-\left(k+1\right)\right)}\right) 
\\& = 
\left(\hat\lambda\lambda^{\left(n-1\right)}\right)^{2\left(k+1\right)}\left(\frac{1}{\alpha} + \eta\sum_{i=1}^{\left(k+1\right)n}\lambda^{-2i}\right)
\\& = \left(\hat\lambda\lambda^{\left(n-1\right)}\right)^{2\left(k+1\right)}\left(\frac{1}{\alpha} + \eta\lambda^{-2}\frac{1 - \lambda^{-2\left(k+1\right)n}}{1 - \lambda^{-2}}\right)
\end{split}\end{equation}
where the inequality holds because $\lambda < \hat\lambda$.

By Claim \ref{V1.1}, this upper bound can be further refined for $0 < {k} \leq \frac{1}{2n}\log_{\lambda}(\frac{1}{1 + \frac{1}{\alpha\eta}(1 - \lambda^{2})}) - 1$: 
\begin{equation}\label{Eq-App-sigma_1.2}\begin{split}
{}& \left(\hat\lambda\lambda^{\left(n-1\right)}\right)^{2\left({k}+1\right)}\left(\frac{1}{\alpha} + \eta\lambda^{-2}\frac{1 - \lambda^{-2\left({k}+1\right)n}}{1 - \lambda^{-2}}\right) \leq 2\left(\hat\lambda\lambda^{\left(n-1\right)}\right)^{2\left({k}+1\right)}\frac{1}{\alpha}.
\end{split}\end{equation}
This proves the lemma for $r = 1$.

For $r > 1$, $(\hat\sigma_{kn + n}^{r>1})^2$ can be bounded as follows:
\begin{equation}\label{Eq-App-sigma-big}
\begin{split}
(\hat{\sigma}^{r>1}_{kn+n})^2 & = 
\left(\hat\lambda\lambda^{n-r}\right)^2\eta\left(\left(\hat\lambda^k\lambda^{k\left(n-1\right)}\right)^2\sum_{i=0}^{r-2}\lambda^{2i} + \sum_{j=0}^{k-1}\left(\hat\lambda\lambda^{n-1}\right)^{2j}\sum_{i=0}^{n-1}\lambda^{2i}\right) + \eta\sum_{i=0}^{n-r}\lambda^{2i} + \frac1\alpha\left(\hat\lambda\lambda^{n-1}\right)^{2k}\left(\hat\lambda\lambda^{n-1}\right)^2
\\& \leq^*
\left(\hat\lambda\lambda^{n-r}\right)^2\eta\left(\left(\hat\lambda^k\lambda^{k\left(n-1\right)}\right)^2\sum_{i=0}^{r-2}\lambda^{2i} + \sum_{j=0}^{k-1}\left(\hat\lambda\lambda^{n-1}\right)^{2j}\sum_{i=0}^{n-1}\lambda^{2i}\right) + \eta\sum_{i=0}^{n-r}\lambda^{2i} + \frac1\alpha\left(\hat\lambda\lambda^{n-1}\right)^{2k}\left(\hat\lambda\lambda^{n-r}\right)^2
\\&=
\left(\hat\lambda\lambda^{n-r}\right)^2\left(\frac1\alpha\left(\hat\lambda\lambda^{n-1}\right)^{2k} + \eta\left(\hat\lambda\lambda^{n-1}\right)^{2k}\sum_{i=0}^{r-2}\lambda^{2i} + \eta\sum_{j=0}^{k-1}\left(\hat\lambda\lambda^{n-1}\right)^{2j}\sum_{i=0}^{n-1}\lambda^{2i}\right) + \eta\sum_{i=0}^{n-r}\lambda^{2i} 
\\&\leq^{**}
\left(\hat\lambda\lambda^{n-r}\right)^2\left(\frac1\alpha\left(\hat\lambda\lambda^{n-1}\right)^{2k} + \eta\left(\hat\lambda\lambda^{n-1}\right)^{2k}\sum_{i=0}^{n-1}\lambda^{2i} + \eta\sum_{j=0}^{k-1}\left(\hat\lambda\lambda^{n-1}\right)^{2j}\sum_{i=0}^{n-1}\lambda^{2i}\right) + \eta\sum_{i=0}^{n-r}\lambda^{2i} 
\\& \leq^{***}
2\left(\hat\lambda\lambda^{n-r}\right)^2\left(\frac1\alpha\left(\hat\lambda\lambda^{n-1}\right)^{2\dot{k}} + \eta\left(\hat\lambda\lambda^{n-1}\right)^{2k}\sum_{i=0}^{n-1}\lambda^{2i} + \eta\sum_{j=0}^{k-1}\left(\hat\lambda\lambda^{n-1}\right)^{2j}\sum_{i=0}^{n-1}\lambda^{2i}\right)
\end{split}\end{equation}
where inequality * is true because $\lambda < 1$ and $r > 1$, inequality ** is true because $r \leq n$, and inequality *** is true because of Claim \ref{V2.2}.

For $0 < k \leq \frac{1}{2n}\log_\lambda(\frac{1}{1 + \frac{1}{\alpha\eta}(1 - \lambda^{2})}) - 1$, this bound can be further refined:
\begin{equation}\label{Eq-App-sigma-big-2}
\begin{split}
{}& 2(\hat\lambda\lambda^{n-r})^2\left(\frac1\alpha\left(\hat\lambda\lambda^{n-1}\right)^{2{k}} + \eta\left(\hat\lambda\lambda^{n-1}\right)^{2k}\sum_{i=0}^{n-1}\lambda^{2i} + \eta\sum_{j=0}^{k-1}\left(\hat\lambda\lambda^{n-1}\right)^{2j}\sum_{i=0}^{n-1}\lambda^{2i}\right)\leq
6(\hat\lambda\lambda^{n-r})^2\frac1\alpha(\hat\lambda\lambda^{n-1})^{2{k}}
\end{split}\end{equation}
where the inequality is true because of Claims \ref{V1.1} and \ref{V2.1}, which provide the bound for $r > 1$.
\end{proof}

\begin{lemma}\label{Lemma-SGLD-2.3}
Mark $\dot{k} = \frac{1}{2n}\log_{\lambda}(\frac{1}{1 + \frac{1}{\alpha\eta}(1 - \lambda^{2})}) - 1$. For the conditions of Lemma \ref{Lemma-SGLD-2.2}, $\dot{k} >0 $ and the values of $\hat\sigma^r_{{\lceil\dot{k}\rceil}n + n}$ can be upper bounded as following:
\begin{equation*}
\begin{split}
{}&(\hat{\sigma}^{1}_{{\lceil\dot{k}\rceil}n+n})^2 \leq \left(1 + 2e^\frac{1}{x_h^2\beta}\right)\left(\hat\lambda\lambda^{n-1}\right)^2\frac{1}{\alpha}\left(\hat\lambda\lambda^{\left(n-1\right)}\right)^{2\lceil\dot{k}\rceil}
\\
{}&(\hat{\sigma}^{r > 1}_{{\lceil\dot{k}\rceil}n + n})^2 \leq 6\left(\hat\lambda\lambda^{n-r}\right)^2\frac1\alpha\left(\hat\lambda\lambda^{n-1}\right)^{2{\lceil\dot{k}\rceil}}.
\end{split}\end{equation*}
\end{lemma}

\begin{proof}[Proof of Lemma \ref{Lemma-SGLD-2.3}] We prove this lemma by augmenting the proof of Lemma \ref{Lemma-SGLD-2.2}. We begin with $r > 1$. Lemma \ref{Lemma-SGLD-2.2} first proved, in eq. \ref{Eq-App-sigma-big}, a bound over $(\hat{\sigma}^{r > 1}_{(k + 1)n})^2$ applicable for $k>0$. As this bound is applicable for all $k>0$, it also applies for $\lceil\dot{k}\rceil$. Then, in eq. \ref{Eq-App-sigma-big-2}, Lemma \ref{Lemma-SGLD-2.2} refined the bound for $0< k \leq\dot{k}$ using Claims \ref{V2.1}, \ref{V1.1} and, \ref{V2.2}. Therefore, if these claims also hold for $k = \lceil\dot{k}\rceil$, then the result of Lemma \ref{Lemma-SGLD-2.2} also applies to $k = \lceil\dot{k}\rceil$. Claims \ref{V2.1} and \ref{V1.1} apply to all $k \leq \frac{1}{2n}\log_\lambda(\frac{1}{1 + \frac{1}{\alpha\eta}(1 - \lambda^{2})})$. Since $\lceil\dot{k}\rceil \leq \dot{k} + 1 = \frac{1}{2n}\log_{\lambda}(\frac{1}{1 + \frac{1}{\alpha\eta}(1 - \lambda^{2})})$, the claims also apply to $\lceil\dot{k}\rceil$. Claim \ref{V2.2} was proved for all $k$, thus also applies to $\lceil\dot{k}\rceil$.


For $r = 1$, the bound found at eq. \ref{Eq-App-sigma_1} is applicable for all $k$, hence
\begin{equation*}
(\hat{\sigma}^1_{(\lceil\dot{k}\rceil + 1)n})^2 \leq \left(\hat\lambda\lambda^{n-1}\right)^{2\left(\lceil\dot{k}\rceil+1\right)}\left(\frac1\alpha + \eta\lambda^{-2}\frac{1 - \lambda^{-2\left(\lceil\dot{k}\rceil + 1\right)n}}{1 - \lambda^{-2}}\right) \leq \left(\hat\lambda\lambda^{n-1}\right)^{2\left(\lceil\dot{k}\rceil+1\right)}\frac1\alpha\left(1 + 2e^{\frac1{x_h^2\beta}}\right)
\end{equation*}
where the last inequality is true because of Claim \ref{V2.4}.

All that is left is to prove that $\dot{k} = \frac{1}{2n}\log_\lambda(\frac{1}{1 + \frac{1}{\alpha\eta}(1 - \lambda^{2})}) - 1 > 0$, which is done in Claim \ref{G4}.
\end{proof}

\begin{lemma}\label{Lemma-SGLD-2.4}
For $\dot{k}$ defined in Lemma \ref{Lemma-SGLD-2.3}, the conditions of Lemma \ref{Lemma-SGLD-2.2}, and $n > \frac\alpha{x_h^2\beta}$:
\begin{equation*}\label{Eq-App-SGLD-2.4}\begin{split}
{}& \frac{\left(\mu_{{\lceil\dot{k}\rceil}n+n} - \hat\mu_{{\lceil\dot{k}\rceil}n+n}^r\right)^2}{\left(\hat{\sigma}^{r}_{{\lceil\dot{k}\rceil}n+n}\right)^2} \geq e^{-\frac2{x_h^2\beta}}\frac{\alpha}{v_1}\left(\frac{3}{32x_h^2\beta}\right)^2\left(\frac{c}{n}\right)^2
\\{}& v_1 = \max\{6, 1 + 2e^{\frac1{x_h^2\beta}}\}.
\end{split}
\end{equation*}
\end{lemma}

\begin{proof}[Proof of Lemma \ref{Lemma-SGLD-2.4}]
\begin{equation*}\begin{split}
\frac{\left(\mu_{{\lceil\dot{k}\rceil}n+n} - \hat\mu_{{\lceil\dot{k}\rceil}n+n}^r\right)^2}{\hat{\sigma}^{r}_{{\lceil\dot{k}\rceil}n+n}} & \geq 
\frac{\left(\lambda^{n-1}\frac{ncx_h^2\beta}{\alpha + nx_h^2\beta}\lambda^{\lceil\dot{k}\rceil\left(n-1\right)}\left(\hat\lambda^{\lceil\dot{k}\rceil+1} - \lambda^{\lceil\dot{k}\rceil+1}\right)\right)^2}{v_1\left(\hat\lambda\lambda^{n-r}\right)^2\frac1\alpha\left(\hat\lambda\lambda^{n-1}\right)^{2{\lceil\dot{k}\rceil}}}
\\& =
\frac{\lambda^{2\lceil\dot{k}\rceil\left(n-1\right)}\left(\lambda^{n-1}\frac{ncx_h^2\beta}{\alpha + nx_h^2\beta}\left(\hat\lambda^{\lceil\dot{k}\rceil+1} - \lambda^{\lceil\dot{k}\rceil+1}\right)\right)^2}{v_1\left(\hat\lambda\lambda^{n-r}\right)^2\frac1\alpha\left(\hat\lambda\lambda^{n-1}\right)^{2{\lceil\dot{k}\rceil}}} = 
\frac{\left(\lambda^{n-1}\frac{ncx_h^2\beta}{\alpha + nx_h^2\beta}\left(\hat\lambda^{\lceil\dot{k}\rceil+1} - \lambda^{\lceil\dot{k}\rceil+1}\right)\right)^2}{v_1\left(\hat\lambda\lambda^{n-r}\right)^2\frac1\alpha\hat\lambda^{2{\lceil\dot{k}\rceil}}}
\\&=
\frac{\alpha\lambda^{2\left(r-1\right)}}{v_1}\left(\frac{ncx_h^2\beta}{\alpha + nx_h^2\beta}\right)^2\frac{\left(\hat\lambda^{\lceil\dot{k}\rceil+1} - \lambda^{\lceil\dot{k}\rceil+1}\right)^2}{\hat\lambda^{2{\lceil\dot{k}\rceil}+1}} = 
\frac{\alpha\lambda^{2\left(r-1\right)}}{v_1}\left(\frac{ncx_h^2\beta}{\alpha + nx_h^2\beta}\right)^2\left(1 - \frac{\lambda^{\lceil\dot{k}\rceil+1}}{\hat\lambda^{\lceil\dot{k}\rceil+1}}\right)^2
\\& \geq
\frac{\alpha\lambda^{2\left(r-1\right)}}{v_1}\left(\frac{ncx_h^2\beta}{\alpha + nx_h^2\beta}\right)^2\left(1 - \left(1 - \frac{\frac34nx^2\beta}{\left(\alpha + nx_h^2\beta\right)^2 - \left(\alpha + \frac14nx^2\beta\right)}\right)\right)^2
\\& =
\frac{\alpha\lambda^{2\left(r-1\right)}}{v_1}\left(\frac{ncx_h^2\beta}{\alpha + nx_h^2\beta}\right)^2\left(\frac{\frac34nx_h^2\beta}{\left(\alpha + nx_h^2\beta\right)^2 - \left(\alpha + \frac14nx^2\beta\right)}\right)^2 
\\&\geq
\frac{\alpha\lambda^{2\left(r-1\right)}}{v_1}\left(\frac{ncx_h^2\beta}{\alpha + nx_h^2\beta}\right)^2\left(\frac{\frac34nx^2\beta}{\left(\alpha + nx_h^2\beta\right)^2}\right)^2 \geq
\frac{\alpha\lambda^{2\left(r-1\right)}}{v_1}\left(\frac{ncx^2\beta}{2nx_h^2\beta}\right)^2\left(\frac{\frac34nx^2\beta}{\left(2nx^2\beta\right)^2}\right)^2
\\& =
\frac{\alpha\lambda^{2\left(r-1\right)}}{v_1}\left(\frac{c}{2}\right)^2\left(\frac{\frac34}{4nx_h^2\beta}\right)^2 = 
\frac{\alpha\lambda^{2\left(r-1\right)}}{v_1}\left(\frac{3}{32x_h^2\beta}\right)^2\left(\frac{c}{n}\right)^2
\\& \geq
\frac{\alpha\lambda^{2\left(n-1\right)}}{v_1}\left(\frac{3}{32x_h^2\beta}\right)^2\left(\frac{c}{n}\right)^2 \geq
e^{-\frac2{x_h^2\beta}}\frac{\alpha}{v_1}\left(\frac{3}{32x_h^2\beta}\right)^2\left(\frac{c}{n}\right)^2
\end{split}\end{equation*}
where first inequality holds from Lemmas \ref{Lem-App-SGLD-2} and \ref{Lemma-SGLD-2.3}, and the definition of $v_1$, the second inequality is true because of Claim \ref{V2.4} and Claim \ref{G4}, fourth inequality holds under the assumption of $nx_h^2\beta > \alpha$, and the last inequality holds from Claim \ref{G1.b}.
\end{proof}

\begin{claim}\label{P1}
For $\dot{k}$, defined in Lemma \ref{Lemma-SGLD-2.3}, and the conditions of Lemma \ref{Lemma-SGLD-2.4}:
\begin{equation*}
p\left(\hat\theta_{\left(\lceil\dot{k}\rceil + 1\right)n} > \mu_{\left(\lceil\dot{k}\rceil + 1\right)n}|D_2\right) \leq e^{-e^{-\frac2{x_h^2\beta}}\frac{\alpha}{2v_1}\left(\frac{3}{32x_h^2\beta}\right)^2\left(\frac{c}{n}\right)^2}.
\end{equation*}
\end{claim}
\begin{proof}[Proof of Claim \ref{P1}]
\begin{equation*}
\begin{split}
p\left(\hat\theta_{\left(\lceil\dot{k}\rceil + 1\right)n} > \mu_{\left(\lceil\dot{k}\rceil + 1\right)n}|\hat{D}\right) & \leq
\frac1n\sum_{r=1}^{n}\exp\left({-\frac{\left(\mu_{\left(\lceil\dot{k}\rceil + 1\right)n} - \hat\mu^{r}_{\left(\lceil\dot{k}\rceil + 1\right)n}\right)^2}{2\left(\sigma^{r}_{\left(\lceil\dot{k}\rceil + 1\right)n}\right)^2}}\right) 
\\&\leq
\frac1n\sum_{r=1}^n\exp\left({-e^{-\frac2{x_h^2\beta}}\frac{\alpha}{2v_1}\left(\frac{3}{32x_h^2\beta}\right)^2\left(\frac{c}{n}\right)^2}\right)
=
\exp\left({-e^{-\frac2{x_h^2\beta}}\frac{\alpha}{2v_1}\left(\frac{3}{32x_h^2\beta}\right)^2\left(\frac{c}{n}\right)^2}\right)
\end{split}\end{equation*}
where the first inequality holds due to Lemma \ref{SGLD-0} and second inequality holds due to Lemma \ref{Lemma-SGLD-2.4}.
\end{proof}

\begin{lemma}\label{Lemma-SGLD-3}
For the \textit{Bayesian linear regression problem} over $\mathcal{D}(n, \gamma_1, x_h, x_l, c)$, the conditions of Lemma \ref{Lemma-SGLD-2.4}, and $\dot{k}$ defined in Lemma \ref{Lemma-SGLD-2.3}, approximate sampling by running SGLD for $(\lceil\dot{k}\rceil + 1)n$ steps will not be ($\epsilon, \delta$) differentially private for 
\begin{equation*}\label{Eq-Lemma-SGLD-3}\begin{split}
{}& \delta < 0.5, \epsilon < e^{-\frac2{x_h^2\beta}}\frac{\alpha}{2v_1}\left(\frac{3}{32x_h^2\beta}\right)^2\left(\frac{c}{n}\right)^2 + \ln\left(0.5 - \delta\right)
\\{}& v_1 = \max\{6, 1 + 2e^{\frac1{x_h^2\beta}}\}.
\end{split}\end{equation*}
\end{lemma}

\begin{proof}[Proof of Lemma \ref{Lemma-SGLD-3}] 
According to Definition \ref{Def-ADP}, if there exists a group, $S$, such that 
\begin{equation}\label{equation::Lemma-sgld-3}
p\left(\theta_ {\left(\lceil\dot{k}\rceil + 1\right)n}\in S|D_1\right) > e^\epsilon p\left(\hat{\theta}_{\left(\lceil\dot{k}\rceil + 1\right)n}\in S|D_2\right) + \delta
\end{equation} 
then releasing $\theta_{(\lceil\dot{k}\rceil + 1)n}$ is not $(\epsilon, \delta)$ differentially private. We will show that eq. \ref{equation::Lemma-sgld-3} is true for $S = \{s| s > \mu_{(\lceil\dot{k}\rceil + 1)n}\}$ and the conditions of the lemma. First, notice that eq. \ref{equation::Lemma-sgld-3} can be rearranged as 
$$
e^\epsilon p\left(\hat{\theta}_{\left(\lceil\dot{k}\rceil + 1\right)n} \in S | D_2\right) + \delta - p\left(\hat{\theta}_{\left(\lceil\dot{k}\rceil + 1\right)n} \in S | D_1\right) \leq 0
$$
By Claim \ref{P1}, and since $\theta_{(\lceil\dot{k}\rceil + 1)n} \sim \mathbb{N}(\theta_{(\lceil\dot{k}\rceil + 1)n}; \mu_{(\lceil\dot{k}\rceil + 1)n}, \sigma^2_{(\lceil\dot{k}\rceil + 1)n})$,
\begin{equation}\label{Eq-App-Lemma-SGLD-3-1}
e^\epsilon p\left(\hat{\theta}_{\left(\lceil\dot{k}\rceil + 1\right)n} \in S | D_2\right) + \delta - p\left(\hat{\theta}_{\left(\lceil\dot{k}\rceil + 1\right)n} \in S | D_1\right) \leq 
e^\epsilon e^{-e^{-\frac2{x_h^2\beta}}\frac{\alpha}{2v_1}\left(\frac{3}{32x_h^2\beta}\right)^2\left(\frac{c}{n}\right)^2} + \delta - 0.5.
\end{equation}
Therefore, if
$$
e^\epsilon e^{-e^{-\frac2{x_h^2\beta}}\frac{\alpha}{2v_1}\left(\frac{3}{32x_h^2\beta}\right)^2\left(\frac{c}{n}\right)^2} + \delta - 0.5 < 0
$$
then eq. \ref{equation::Lemma-sgld-3} is true and the lemma is proved. As shown in eq. \ref{Eq-App-Lemma-SGLD-3-2}, this inequality holds under the conditions of Lemma \ref{Lemma-SGLD-3}.


\begin{equation}\label{Eq-App-Lemma-SGLD-3-2}\begin{split}
{}& e^{\epsilon -e^{-\frac2{x_h^2\beta}}\frac{\alpha}{2v_1}(\frac{3}{32x_h^2\beta})^2(\frac{c}{n})^2} + \delta - 0.5 < 0 \iff
\\{}& e^{\epsilon -e^{-\frac2{x_h^2\beta}}\frac{\alpha}{2v_1}(\frac{3}{32x_h^2\beta})^2(\frac{c}{n})^2} < 0.5 - \delta \iff
\\{}& \epsilon -e^{-\frac2{x_h^2\beta}}\frac{\alpha}{2v_1}(\frac{3}{32x_h^2\beta})^2(\frac{c}{n})^2 < \ln(0.5 - \delta) \iff
\\{}& \epsilon < e^{-\frac2{x_h^2\beta}}\frac{\alpha}{2v_1}(\frac{3}{32x_h^2\beta})^2(\frac{c}{n})^2 + \ln(0.5 - \delta)
\end{split}\end{equation}
\end{proof}

\subsection{Figure \ref{fig::sgld-figure} derivation}\label{fig-deriv}
The lower bound depicted in figure \ref{fig::sgld-figure} is derived by analysing the distributions of SGLD running on datasets $D_1$ and $D_2$ (defined in \ref{N-DBs}). Given these distributions, and using notations $\theta_j, \mu_j, \hat{\theta}_j, \hat{\mu}_j, \hat{\sigma}^r_j$ (defined in subsection \ref{subseq-SGLD}), the lower bound over $\epsilon$ for a given $\delta$ can be deduced as follows:

By definition \ref{Def-ADP}, for SGLD running on datasets $D_1, D_2$ to be $(\epsilon, \delta)$-DP, it must hold
\begin{equation*}
p\left(\theta_j > \mu_j\right) \leq \exp\left(\epsilon\right)p(\hat{\theta}_j > \mu_j) + \delta.
\end{equation*}
This condition can be easily translated to a lower bound over $\epsilon$:
\begin{equation*}
\ln\left(p\left(\theta_j > \mu_j\right) - \delta\right) - \ln(p(\hat{\theta}_j > \mu_j)) \leq \epsilon.
\end{equation*}
By Lemma \ref{lemma:diff_of_means}, 
\begin{equation*}
p(\hat{\theta}_j > \mu_j) \leq \frac1n\sum_{r=1}^n\exp\left(-\frac{\left(\mu_j - \hat{\mu}_j^r\right)^2}{2\left(\hat\sigma_j^r\right)^2}\right).
\end{equation*}
As $ln$ is monotonically increasing, this induces a necessary condition for SGLD to be $(\epsilon, \delta)$-DP:
\begin{equation*}
\ln\left(p\left(\theta_j > \mu_j\right) - \delta\right) - \ln\left(\frac1n\sum_{r=1}^n\exp\left(-\frac{\left(\mu_j - \hat{\mu}_j^r\right)^2}{2\left(\hat\sigma_j^r\right)^2}\right)\right) \leq \epsilon.
\end{equation*}

In figure \ref{fig::sgld-figure} we plot the value 
$$\max\left\{0, \ln\left(p\left(\theta_j > \mu_j\right) - \delta\right) - \ln\left(\frac1n\sum_{r=1}^n\exp\left(-\frac{\left(\mu_j - \hat{\mu}_j^r\right)^2}{2\left(\hat\sigma_j^r\right)^2}\right)\right)\right\}.$$

\newpage
\section{Propose Test Sample Supplementary}
$n_{min}$, which is used in Algorithm \ref{A-PTR}, is defined as follows:
\begin{equation}\label{alg-n_min}\begin{split}
\nu & = \frac{2\ln(\frac1\delta)}\epsilon + 1
\\n_{b1} &= \max\Big\{1 + \frac{x_h^2}{x_l^2}\frac{8}{\epsilon}, 
1 + \nu\frac{x_h^2}{x_l^2}\left(1 + 8\frac{\nu - 1}{\epsilon}\right), 
\left(\frac{16\nu\beta x_h^4}{\frac9{10}\epsilon x_l^2}\right)^{\frac1{1 - 2\gamma_1}},
\\& \left(\frac{32\nu}{\epsilon}\cdot\frac{x_h^4\left(\alpha + x_h^2\beta\right)}{\frac9{10}x_l^4}\cdot\breve{m}\right)^{\frac1{2 - \gamma_1}},
\left(\frac{32\nu}{\epsilon}\cdot\frac{x_h^4\left(\alpha + x_h^2\beta\right)}{\frac9{10}x_l^4}\right)^{\frac1{2 - 2\gamma_1}}, 
\\& \left(\frac{4\nu}{\epsilon}\cdot\frac{x_h^4\left(\alpha + x_h^2\beta\right)^2}{\frac9{10}x_l^6\beta}\right)^\frac13\cdot\left(2\breve{m}\right)^\frac23,
\left(\frac{16\nu}{\epsilon}\cdot\frac{x_h^4\left(\alpha + x_h^2\beta\right)^2}{\frac9{10}x_l^6\beta}\right)^\frac1{3 - 2\gamma_1}\Big\}
\\n_{b2} & = \max\left\{1 +  \frac{x_h^2}{x_l^2}\frac{10\nu}{\beta}, 1 + \nu\frac{x_h^2}{x_l^2}\right\}
\\n_{min} & = \max\left\{n_{b1}, n_{b2}, {n_1}^\frac{\rho_2}{\gamma_1}\right\}
\end{split}\end{equation}
where $\epsilon, \delta, x_l, x_h, \alpha, \beta, \gamma_1, \rho_1, \rho_2, n_1$ are parameters of Algorithm \ref{A-PTR}.

\section{Propose Test Sample Privacy}\label{app:pts-privacy}
Appendix \ref{app:pts-privacy} provides auxiliary claims and proof for Claim \ref{claim-A-PTR-not-private}, along with auxiliary claims which support the proof of Claim \ref{claim-A-PTR-Posterior-private}. Appendix \ref{app:pts-privacy} uses definitions and notations defined in subsection \ref{ProposeReleaseTest}, specifically: $D_3, D_4$ (defined in eq. \ref{DB34}), $\epsilon, \delta, x_l, x_h, \alpha, \beta, \gamma_1, \rho_1, \rho_2$ - parameters of algorithm \ref{A-PTR}, $D$ input dataset to algorithm \ref{A-PTR} of size $n_1$, $V, n_2, m, \breve{m}, n_W, W, p(\theta|W), l_1, l_2, l_3, l_4$ (defined in algorithm \ref{A-PTR}), and $n_{min}$ (defined in eq. \ref{alg-n_min}).
\begin{proof}[Proof of Claim \ref{claim-A-PTR-not-private}]
We analyze Algorithm \ref{A-PTR}, running on datasets $D_3$ and $D_4$ defined in eq. \ref{DB34}, with parameters values:
\begin{equation}\begin{split}\label{claim-atpr-not-private-params}
{}& \rho_3 = 1.15; \rho_2 = 0.45; \rho_1 = 1.25, \gamma_1 = 0.49;
\\{}& \beta = 3; x_h = 1; x_l = 0.5; \alpha = 1
\\{}& n_1 > \max\left\{\frac4\epsilon\log\frac1{2\delta}, 2^{10\rho_1}, 2^{9 + 10\rho_2}\right\}.
\end{split}\end{equation}
Note that we only define a lower bound over $n_1$, which will be updated later on.

Mark the return value of the algorithm as $r$, the event of the algorithm running on dataset $D_3$ and $W = D_3$ as $A_{D_3}$, the event of the algorithm running on dataset $D_4$ and $W = D_4$ as $A_{D_4}$, and $S = \{s | s > \mu_{i}\}$, where $\mu_i$ is the mean of the sample distribution at the SGLD $i$'th step given dataset $D_3$ (similarly to the definition of $S$ in subsection \ref{subseq-SGLD}). We will show that $\forall \epsilon \in \R_{>0}, \delta < \frac16$ there exists $n_1$ such that 
\begin{equation}\label{apts-0-0}
p(r \in S | D_3) > e^\epsilon p(r \in S| D_4) + \delta.
\end{equation}

We first show that
\begin{equation}\label{apts-0-1}\begin{split}
{}& p(r \in S \land A_{D_3}^c|D_3) = 0
\\{}& p(r \in S \land A_{D_4}^c|D_4) = 0.
\end{split}\end{equation}
Notice that $r \in S$ only if the algorithm reached line 25. Consider an event where the algorithm reached line 25 and $A_{D_3}^c$. Because $A_{D_3}^c$, $\exists (x_i, y_i) \in D_3$ such that $|\frac{y_i}{x_i} - \breve{m}| \geq n_2^{\rho_2}$. However, since $\forall (x_i, y_i) \in D_3: \frac{y_i}{x_i} = n_1^{\rho_3}$ then $\forall (x_i, y_i) \in D_3: |\frac{y_i}{x_i} - \breve{m}| > n_2^{\rho_2}$ and therefore $|W| = 0$. Under the assumption that a sample from $p(\theta | \{\})$ returns $null$, the algorithm, in this case, also returns $null$ and therefore $P(r \in S \land A_{D_3}^c|D_3) = 0$. Same arguments hold for $D_4$.

Because we showed eq. \ref{apts-0-1} is true, then to prove eq. \ref{apts-0-0}, it is enough to show that eq. \ref{eq-aptr-np} is true.
\begin{equation}\label{eq-aptr-np}\begin{split}
p(r \in S| D_3, A_{D_3})p(A_{D_3} | D_3) & \geq^* 
p(r \in S| D_3, A_{D_3}) - 5\delta > ^{**}
e^{\epsilon} p(r \in S| D_4, A_{D_4}) + \delta 
\\& \geq
e^{\epsilon} p(r \in S| D_4, A_{D_4})p(A_{D_4} | D_4) + \delta =
e^{\epsilon} p(r \in S \land A_{D_4}| D_4) + \delta
\end{split}\end{equation}

From Claim \ref{claim-pts-0}, $\exists n_{bound_1}$ such that $\forall n_1 > n_{bound_1}$ inequality * holds. From Lemma \ref{SGLD-2}, for $n_1$ big enough, $\exists T \in \mathbb{Z}_{>0}$ such that eq. \ref{eq-aptr-0-1} hold (Where $6\delta < 0.5$ according to the claim's conditions). Therefore, $\exists k, n_{bound2} \in \R_{>0}$ such that $\forall n_1 > n_{bound_2}: \epsilon' > kn_1^{2(1 - \rho_3)}$ and eq. \ref{eq-aptr-0-1} hold. As $\rho_3 > 1$, by choosing $n_1 > \max\{n_{bound2}, (\frac\epsilon k)^{\frac1{2(\rho_3 - 1)}}\}$ we get that $\epsilon' > \epsilon$.
Consequently, by choosing $n_1 > \max\{n_{bound_1}, n_{bound_2}, (\frac\epsilon k)^{\frac1{2(\rho_3 - 1)}}\}$, inequalities * and ** hold, and the claim is proved.

\begin{equation}\label{eq-aptr-0-1}\begin{split}
{}& \epsilon' = \Omega(n_1^{2(\rho_3 - 1)})
\\{}& p(r \in S| D_3, A_{D_3}) > e^{\epsilon'} p(r \in S| D_4, A_{D_4}) + 6\delta
\end{split}\end{equation}

\end{proof}

\begin{claim}\label{claim-pts-0}
Given a run of Algorithm \ref{A-PTR} on dataset $D_3$, mark by $A$ the event of the algorithm reaching line 25 with $W = D_3$; the following holds:
$$
\exists n_{bound_1} \in \mathbb{Z}_{>0}\ s.t.\ \forall n_1 > n_{bound_1}: p(A) \geq 1 - 5\delta.
$$
\end{claim}
\begin{proof}
For abbreviation, mark the event of $n_W > n_{min} \land \breve{m} \in [m - n_2^{\rho_2}, m + n_2^{\rho_2}] \land n_2^{1 + \rho_2 - 0.1} > \breve{n}_1^{\rho_1} \land \breve{n_1} \leq n_1 \land V = D$ as $B$. Since $P(A|D_3, B) = 1$, then $P(A | D_3) \geq P(A \land B| D_3) = P(A|B, D_3)P(B|D_3) = P(B| D_3)$. Therefore, we can prove the claim by showing the existence of $n_{lb}$ such that $\forall n_1 > n_{lb}: P(B | D_3) \geq 1 - 5\delta$. We do so in eq. \ref{eq-claim-pts-0}:
%

\begin{equation}\label{eq-claim-pts-0}\begin{split}
p(B | D_3) &= 
p(\breve{m} \in [m - n_2^{\rho_2}, m + n_2^{\rho_2}] \land n_W > n_{min} | D_3, V = D, n_2^{1 + \rho_2 - 0.1} > \breve{n}_1^{\rho_1}, \breve{n_1} \leq n_1)
\\&\cdot 
p(n_2^{1 + \rho_2 - 0.1} > \breve{n}_1^{\rho_1}| V = D, \breve{n_1} \leq n_1, D_3)P(V = D, \breve{n_1} \leq n_1|D_3) 
\\& \geq 
p(\breve{m} \in [m - n_2^{\rho_2}, m + n_2^{\rho_2}] \land n_W > n_{min} | D_3, V = D, n_2^{1 + \rho_2 - 0.1} > \breve{n}_1^{\rho_1}, \breve{n_1} \leq n_1) - 3\delta 
\\& =
p(n_W > n_{min} | D_3, V = D, n_2^{1 + \rho_2 - 0.1} > \breve{n}_1^{\rho_1}, \breve{n_1} \leq n_1, \breve{m} \in [m - n_2^{\rho_2}, m + n_2^{\rho_2}])
\\& \cdot p(\breve{m} \in [m - n_2^{\rho_2}, m + n_2^{\rho_2}]| D_3, V = D, n_2^{1 + \rho_2 - 0.1} > \breve{n}_1^{\rho_1}, \breve{n_1} \leq n_1) - 3\delta
\\&\geq p(n_W > n_{min} | D_3, V = D, n_2^{1 + \rho_2 - 0.1} > \breve{n}_1^{\rho_1}, \breve{n_1} \leq n_1, \breve{m} \in [m - n_2^{\rho_2}, m + n_2^{\rho_2}]) - 4\delta
\\&\geq 1 - 5\delta
\end{split}\end{equation}
where by Corollary \ref{cor-pts-1} and Claim \ref{claim-pts-2}, for $n_1$ big enough, first inequality holds. By Claim \ref{claim-pts-3}, for $n_1$ big enough, second inequality holds. Lastly, by Claim \ref{claim-pts-4}, for $n_1$ big enough, third inequality holds. 
\end{proof}

\begin{claim}\label{claim-pts-6}
For $n_1 > \max\{2^{10\rho_1}, \frac4\epsilon\log\frac1{2\delta}\}$:
\begin{equation*}
p(\breve{n}_1^{\rho_1} \geq n_1^{\rho_3} \land \breve{n}_1 \leq n_1 | D_3) \geq 1 - 2\delta.
\end{equation*}
\end{claim}
\begin{proof}[Proof of Claim \ref{claim-pts-6}]
\begin{equation}\label{eq::claim-pts-6}\begin{split}
p(\breve{n}_1^{\rho_1} \geq n_1^{\rho_3} \land \breve{n}_1 \leq n_1 | D_3) &= 
p((n_1 + l_1 - \frac1\epsilon\log\frac1{2\delta})^{\rho_1} \geq n_1^{\rho_3} \land n_1 + l_1 - \frac1\epsilon\log\frac1{2\delta} \leq n_1 | D_3)
\\&= p((n_1 + l_1 - \frac1\epsilon\log\frac1{2\delta})^{\rho_1} \geq n_1^{\rho_3} \land l_1 \leq \frac1\epsilon\log\frac1{2\delta} | D_3)
\\&= p((n_1 + l_1 - \frac1\epsilon\log\frac1{2\delta})^{\rho_1} \geq n_1^{\rho_3} \land |l_1| \leq \frac1\epsilon\log\frac1{2\delta} | D_3)
\\&+ p((n_1 + l_1 - \frac1\epsilon\log\frac1{2\delta})^{\rho_1} \geq n_1^{\rho_3} \land l_1 \leq -\frac1\epsilon\log\frac1{2\delta} | D_3)
\\&\geq p((n_1 + l_1 - \frac1\epsilon\log\frac1{2\delta})^{\rho_1} \geq n_1^{\rho_3} \land |l_1| \leq \frac1\epsilon\log\frac1{2\delta} | D_3)
\\&= p((n_1 + l_1 - \frac1\epsilon\log\frac1{2\delta})^{\rho_1} \geq n_1^{\rho_3} ||l_1| \leq \frac1\epsilon\log\frac1{2\delta}, D_3)p(|l_1| \leq \frac1\epsilon\log\frac1{2\delta} | D_3)
\end{split}\end{equation}
As 
\begin{align*}
p(|l_1| \leq \frac1\epsilon\log\frac1{2\delta}) = 1 - 2p(l_1 \leq -\frac1\epsilon\log\frac1{2\delta}) = 1 - \exp(-\frac{\frac1\epsilon\log\frac1{2\delta}}{\frac1\epsilon}) =1 - 2\delta,
\end{align*}
we can further develop eq. \ref{eq::claim-pts-6}:
\begin{align*}
{}& p((n_1 + l_1 - \frac1\epsilon\log\frac1{2\delta})^{\rho_1} \geq n_1^{\rho_3} ||l_1| \leq \frac1\epsilon\log\frac1{2\delta}, D_3)p(|l_1| \leq \frac1\epsilon\log\frac1{2\delta} | D_3) \geq 
p((n_1  - 2\frac1\epsilon\log\frac1{2\delta})^{\rho_1} \geq n_1^{\rho_3} | D_3) - 2\delta
\end{align*}
Finally, because $n_1 > 2^{10\rho_1}$, then $n_1^{-0.1} < (\frac12)^{\rho_1}$. Therefore, $(n_1 - 2\frac1\epsilon\log\frac1{2\delta})^{\rho_1} > (\frac12n_1)^{\rho_1} > n_1^{\rho_1 - 0.1} = n_1^{\rho_3}$. Which leads to 
$$
p((n_1  - 2\frac1\epsilon\log\frac1{2\delta})^{\rho_1} \geq n_1^{\rho_3} | D_3) - 2\delta = 1 - 2\delta
$$

\end{proof}

Noticing that the conditions of corollary \ref{cor-pts-1} includes the conditions of Claim \ref{claim-pts-6}, and that given that the algorithm runs on dataset $D_3$ and $\breve{n}_1^{\rho_1} \geq n_1^{\rho_3}$ then $V = D_3$, we get that Corollary \ref{cor-pts-1} is a direct result of Claim \ref{claim-pts-6}.
\begin{corollary}\label{cor-pts-1}
$\forall n_1 > \max\{2^{10\rho_1}, 4\frac1\epsilon\log\frac1{2\delta}\}$, when running Algorithm \ref{A-PTR} on dataset $D_3$, the following inequality holds:
$$ P(V = D_3 \land \breve{n_1} \leq n_1) \geq 1 - 2\delta.$$
\end{corollary}


\begin{claim}\label{claim-pts-2}
$\forall n_1 > \max\{\frac12^{-(9 + 10\rho_2)}, 4\frac1\epsilon\log\frac1{2\delta}\}$, when running Algorithm \ref{A-PTR} on dataset $D_3$ the following inequality holds:
$$p(n_2^{1 + \rho_2 - 0.1} > \breve{n}_1^{\rho_1}|D_3, V = D_3, \breve{n}_1 \leq n_1) \geq 1 - \delta.$$
\end{claim}
\begin{proof}[Proof of Claim \ref{claim-pts-2}]
\begin{equation*}\begin{split}
p(n_2^{0.9 + \rho_2} > \breve{n}_1^{\rho_1}|D_3, V = D_3, \breve{n}_1 \leq n_1) & \geq
p(n_2^{0.9 + \rho_2} > n_1^{\rho_1}|D_3, V = D_3)
\\&\geq p((n_1 - 2\frac1\epsilon\log\frac1{2\delta})^{0.9 + \rho_2} > n_1^{\rho_1}|D_3)p(n_2 \geq |V| - 2\frac1\epsilon\log\frac1{2\delta})
\\&\geq p((n_1 - 2\frac1\epsilon\log\frac1{2\delta})^{0.9 + \rho_2} > n_1^{\rho_1}|D_3) - \delta = 1 - \delta
\end{split}\end{equation*}
where the third inequality holds since $p(\text{Lap}(\frac1\epsilon) < -\frac1\epsilon\log\frac1{2\delta}) < \delta$, and last equality holds since $(n_1 - 2\frac1\epsilon\log\frac1{2\delta})^{0.9 + \rho_2} > (\frac12n_1)^{0.9 + \rho_2} > n_1^{1 + \rho_2 - 0.2} = n_1^{\rho_1}$
\end{proof}

\begin{claim}\label{claim-pts-3}
$\exists n_{lb_1} \in \mathbb{Z}_{>0}$ such that $\forall n_1 > n_{lb_1}$, when running Algorithm \ref{A-PTR} on dataset $D_3$, the following inequality holds:
$$p(\breve{m} \in [m - n_2^{\rho_2}, m + n_2^{\rho_2}]| D_3, n_2^{0.9 + \rho_2} > \breve{n}_1^{\rho_1}) \geq 1 - \delta.$$
\end{claim}
\begin{proof}[Proof of Claim \ref{claim-pts-3}]
\begin{equation*}\begin{split}
p(\breve{m} \in [m - n_2^{\rho_2}, m + n_2^{\rho_2}]| D_3, n_2^{0.9 + \rho_2} > \breve{n}_1^{\rho_1}) &=
p(l_3 \in [-n_2^{\rho_2}, n_2^{\rho_2}] | D_3, n_2^{0.9 + \rho_2} > \breve{n}_1^{\rho_1})  
\\& \geq 
1 - 2(\frac12\exp(-n_2^{\rho_2}\frac{1}{\frac1\epsilon\breve{n}_1^{\rho_1}\frac{2(n_2 - 1)x_h^2x_l^2 + x_h^4}{n_2(n_2 - 1)x_l^4}}) 
\\&=
1 - \exp(-\frac{n_2^{1 + \rho_2}\epsilon(n_2 - 1)x_l^4}{\breve{n}_1^{\rho_1}(2(n_2 - 1)x_h^2x_l^2 + x_h^4)})
\end{split}\end{equation*}
Because $n_2^{0.9 + \rho_2} > \breve{n}_1^{\rho_1}$, then $\breve{n}_1^{\rho_1} = o(n_2^{1 + \rho_2})$, and therefore for $n_1$ big enough, the exponent is smaller than $\delta$.
\end{proof}


\begin{claim}\label{claim-pts-4}
$\exists n_{lb_2} \in \mathbb{Z}_{>0}$ such that $\forall n_1 > n_{lb_2}$, when running Algorithm \ref{A-PTR} on dataset $D_3$, the following inequality holds:
$$p(n_W > n_{min} | |V| = D_3 \land n_2^{0.9 + \rho_2} > \breve{n}_1^{\rho_1} \land \breve{m} \in [m - n_2^{\rho_2}, m + n_2^{\rho_2}], D_3) \geq 1 - \delta.$$  
\end{claim}
\begin{proof}[Proof of Claim \ref{claim-pts-4}]
For abbreviation, mark event $B$ as when the following apply $|V| = D_3 \land n_2^{0.9 + \rho_2} > \breve{n}_1^{\rho_1} \land \breve{m} \in [m - n_2^{\rho_2}, m + n_2^{\rho_2}]$. 

\begin{equation*}\begin{split}
p(n_W > n_{min} | B, D_3) & = 
p(n_1 - \frac1\epsilon\log\frac1{2\delta} + l_4  > n_{min} | B, D_3)
\\& > p(n_1 - \frac1{\epsilon}\log\frac1{2\delta} - \frac1\epsilon\log\frac1\delta > n_{min} | B, D_3)p(l_4 >  - \frac1{\epsilon}\log\frac1{\delta})
\\{}& +
p(n_1 - \frac1\epsilon\log\frac1{2\delta} + l_4  > n_{min} \land  l_4 < -\frac{1}{\epsilon}\log\frac1{\delta}| B, D_3) 
\\& > p(n_1 - \frac2\epsilon\log\frac1{2\delta} - \frac1\epsilon\log\frac1{\delta} > n_{min} | B, D_3)(1 - \frac\delta2) 
\\& > 
p(n_1 - \frac2\epsilon\log\frac1{2\delta} - \frac1\epsilon\log\frac1{\delta} > n_{min} | B, D_3) - \frac\delta2 
\\&=  p(n_1 - \frac2\epsilon\log\frac1{2\delta} - \frac1\epsilon\log\frac1{\delta} > n_{min} | l_2 < \frac1\epsilon\log\frac1\delta, B, D_3)p(l_2 < \frac1\epsilon\log\frac1\delta| B, D_3)  
\\&+ p(n_1 - \frac2\epsilon\log\frac1{2\delta} - \frac1\epsilon\log\frac1{\delta} > n_{min} \land l_2 \geq \frac1\epsilon\log\frac1\delta|B , D_3) - \frac\delta2 
\\&\geq  p(n_1 - \frac2\epsilon\log\frac1{2\delta} - \frac1\epsilon\log\frac1{\delta} > n_{min} | l_2 < \frac1\epsilon\log\frac1\delta, B, D_3)p(l_2 < \frac1\epsilon\log\frac1\delta| B, D_3) - \frac\delta2  
\\&\geq p(n_1 - \frac2\epsilon\log\frac1{2\delta} - \frac1\epsilon\log\frac1{\delta} > n_{min} | l_2 < \frac1\epsilon\log\frac1\delta, B, D_3) - \delta
\end{split}\end{equation*}

From $B$, $|m - \breve{m}| < n_2^{\rho_2}$, and therefore $\breve{m} < m + n_2^{\rho_2}$, and for the case of $l_2 < \frac1\epsilon\log\frac1\delta$ it holds that $$n_2 < n_1 - \frac1\epsilon\log\frac1{2\delta} + \frac1\epsilon\log\frac1{\delta} < n_1 + \frac1\epsilon\log\frac1{\delta}.$$ 
Therefore $\breve{m} \leq m + (n_1 + \frac1\epsilon\log\frac1{\delta})^{\rho_2}$. As $n_{min} = \mathcal{O}(\max\{\breve{m}^\frac23,n_1^\frac{\rho_2}{\gamma_1}\})$ then for the case of $l_2 < \frac1\epsilon\log\frac1\delta$ and $B$, it holds that $n_{min} = \mathcal{O}(\max\{(m + n_1^{\rho_2})^\frac23, n_1^\frac{\rho_2}{\gamma_1}\}) = \mathcal{O}(\max\{n_1^{\frac{2\rho_3}3}, n_1^\frac{\rho_2}{\gamma_1}\}) < o(n_1)$; therefore $\exists n_{lb_2}$ such that $\forall n_1 > n_{lb_2}$ : $p(n_1 - \frac2\epsilon\log\frac1{2\delta} - \frac1\epsilon\log\frac1{\delta} > n_{min} | l_2 < \frac1\epsilon\log\frac1\delta, B, D_3) = 1$.
\end{proof}

\begin{definition}
A randomized function $f(X, y): \chi^{n_1} \times \R^{n_2} \to \R$, is $(\epsilon, \delta)$-differentially private with respect to $X$ if $\forall S \subseteq \R$, and $\forall X,\hat{X} \in \chi^{n}: \|X - \hat{X}\| \leq 1$, eq. \ref{DifferentialPrivacyWithRespect} holds.
\begin{equation} \label{DifferentialPrivacyWithRespect}
p(f(X,y) \in S) \leq \exp(\epsilon)p(f(\hat{X},y) \in S) + \delta
\end{equation}
\end{definition}

\begin{definition}[$l_1$-sensitivity, \citet{AlgorithmicFundations}]
The $l_1$-sensitivity of a function $f:\mathbb{N}^{|\chi|} \to \R^k$ is:
$$
\Delta f = \max_{x,y\in\mathbb{N}^{|\chi|}, \|x - y\|_1 = 1}\|f(x) - f(y)\|_1
$$
\end{definition}

\begin{claim}\label{ro-1}
For Algorithm \ref{A-PTR}, calculating $\breve{n}_1, {n}_2$ is $(2\epsilon, 0)$ differentially private.
\end{claim}
\begin{proof}[Proof of Claim \ref{ro-1}]
$n_1$ $l_1$-sensitivity is 1; therefore, calculating $\breve{n}$ is $(\epsilon,0)$ DP by the Laplace mechanism's privacy guarantees. For a given $\breve{n}_1$ value, the $l_1$-sensitivity of $|V|$ is 1. Therefore, given $\breve{n}_1$, calculating $n_2$ is $(\epsilon, 0)$ DP by the Laplace mechanism's privacy guarantees. Consequently, by the sequential composition theorem, the composition is $(2\epsilon, 0)$ differentially private.
\end{proof}

\begin{claim}\label{ro-3}
For Algorithm \ref{A-PTR}, $p(n_2 \leq |V|| D, \breve{n}_1) = 1 - \delta$.
\end{claim}
\begin{proof}[Proof of Claim \ref{ro-3}]
\begin{equation*}\begin{split}
{}& p(n_2 \leq |V|| D, \breve{n}_1) = 
p(|V| - \frac1\epsilon\log\frac1{2\delta} + l_2 \leq |V||D, \breve{n}_1) =
p(l_2 \leq \frac1\epsilon\log\frac1{2\delta} |D, \breve{n}_1) = 
1 - \frac12\exp(-\frac{\frac1\epsilon\log\frac1{2\delta}}{\frac1\epsilon}) = 1-\delta
\end{split}\end{equation*}
\end{proof}

\begin{claim}\label{ro-4}
Given $\breve{n}_1, n_2$ and $n_2 < |V|$, calculating $\breve{m}$ in Algorithm \ref{A-PTR} is $(\epsilon, 0)$ differentially private with respect to $D$.
\end{claim}
\begin{proof}[Proof of Claim \ref{ro-4}]
Given two neighbouring datasets, $D_5$ and $\hat{D}_5$, mark by $V_5$ and $\hat{V}_5$ the realizations of $V$ (calculated at line 10) when Algorithm \ref{A-PTR} runs on each of the datasets, respectively. If $V_5=\hat{V_5}$ then the claim follows trivially. In case they differ, assume w.l.o.g that $|V_5| \geq |\hat{V_5}|$, and that if $|V_5| = |\hat{V_5}|$ then they differ in their last sample. Define $q = \sum_{(x_i, y_i) \in V_5/\{x_{|V_5|},y_{|V_5|}\}}x_iy_i$, $z = \sum_{(x_i, y_i) \in V_5/\{x_{|V_5|},y_{|V_5|}\}}x_i^2$. 
\begin{equation*}\label{eq-ro-4-1}\begin{split}
|\frac{q + x_{|V_5|}y_{|V_5|}}{z + x_{|V_5|}^2} - \frac{q + \hat{x}_{|V_5|}\hat{y}_{|V_5|}}{z + \hat{x}_{|V_5|}^2}| &= 
|\frac{q\hat{x}_{|V_5|}^2 + x_{|V_5|}y_{|V_5|}\hat{x}_{|V_5|}^2 + x_{|V_5|}y_{|V_5|}z - qx_{|V_5|}^2 - \hat{x}_{|V_5|}\hat{y}_{|V_5|}x_{|V_5|}^2 - \hat{x}_{|V_5|}\hat{y}_{|V_5|}z}{(z + x_{|V_5|}^2)(z + \hat{x}_{|V_5|}^2)}|  
\\& \leq 
\frac{qx_h^2 + \breve{n}_1^{\rho_1}x_h^2z + \breve{n}_1^{\rho_1}x_h^4}{(z + x_l^2)z} \leq 
\breve{n}_1^{\rho_1}\frac{2zx_h^2 + x_h^4}{(z + x_l^2)z} = 
\breve{n}_1^{\rho_1}(\frac{2x_h^2}{z + x_l} + \frac{x_h^4}{(z + x_l^2)z}) 
\\& \leq  
\breve{n}_1^{\rho_1}(\frac{2x_h^2}{|V_5|x_l^2} + \frac{x_h^4}{|V_5|(|V_5| - 1)x_l^4}) \leq 
\breve{n}_1^{\rho_1}(\frac{2x_h^2}{n_2x_l^2} + \frac{x_h^4}{n_2(n_2 - 1)x_l^4}) 
\\& =
\breve{n}_1^{\rho_1}\frac{2(n_2 - 1)x_h^2x_l^2 + x_h^4}{n_2(n_2 - 1)x_l^4}
\end{split}\end{equation*}
Therefore, by the Laplace mechanism's privacy guarantees, calculating $\breve{m}$ is $(\epsilon, 0)$ differentially private.
\end{proof}

\begin{claim}\label{ro-5}
Lines 8-18 of Algorithm \ref{A-PTR} are $(3\epsilon, \delta)$ differentially private.
\end{claim}
\begin{proof}[Proof of Claim \ref{ro-5}]
Define $\hat{D}$ as a neighbouring dataset to $D$.
\begin{equation*}
\begin{split}
p(\breve{m} \in S|D) &= 
\int_{r_1,r_2\in \R_{>0} \times \R_{>0}}p(\breve{m} \in S|D, \breve{n}_1 = r_1, n_2 = r_2)p(\breve{n}_1 = r_1, n_2 = r_2|D)dr_1dr_2
\\&= \int_{r_1,r_2\in \R_{>0} \times [1, |V|]}p(\breve{m} \in S|D, \breve{n}_1 = r_1, n_2 = r_2)p(\breve{n}_1 = r_1, n_2 = r_2|D)dr_1dr_2
\\{}&+ 
\int_{r_1,r_2\in \R_{>0} \times (|V|,\infty]}p(\breve{m} \in S|D, \breve{n}_1 = r_1, n_2 = r_2)p(\breve{n}_1 = r_1, n_2 = r_2|D)dr_1dr_2 
\\&\leq^* \int_{r_1,r_2\in \R_{>0} \times [1, |V|]}p(\breve{m} \in S|D, \breve{n}_1 = r_1, n_2 = r_2)p(\breve{n}_1 = r_1, n_2 = r_2|D)dr_1dr_2 + \delta 
\\&\leq^{**} \int_{r_1,r_2\in \R_{>0} \times [1, |V|]}e^{2\epsilon} p(\breve{m} \in S|\hat{D}, \breve{n}_1 = r_1, n_2 = r_2)p(\breve{n}_1 = r_1, n_2 = r_2|\hat{D})dr_1dr_2 + \delta 
\\&\leq \int_{r_1,r_2\in \R_{>0} \times \R_{>0}}e^{2\epsilon} p(\breve{m} \in S|\hat{D}, \breve{n}_1 = r_1, n_2 = r_2)p(\breve{n}_1 = r_1, n_2 = r_2|\hat{D})dr_1dr_2 + \delta
\\&= e^{2\epsilon}p(\breve{m} \in S|\hat{D}) + \delta
\end{split}
\end{equation*}
where inequality * follows Claim \ref{ro-3} and inequality ** follows Claims \ref{ro-4} and \ref{ro-1}.
\end{proof}


\begin{claim}\label{pts-1}
Given $n_2, \breve{m}$ and $|W| < n_{min}$, lines 19-25 of Algorithm \ref{A-PTR} are $(\epsilon, \delta)$ differentially private with respect to $D$.
\end{claim}
\begin{proof}[Proof of Claim \ref{pts-1}]
Mark $l \sim \text{Lap}(\frac1\epsilon)$, and $\hat{D}$ as a neighbouring dataset to $D$. Eq. \ref{eq-pts-1-2} proves the claim.

\begin{equation}\label{eq-pts-1-2}
\begin{split}
p(S | D, |W| < n_{min}, \breve{m}, n_2) &= 
p(S \cap \{null\} | D, |W| < n_{min}, \breve{m}, n_2) + 
p(S \cap \{null\}^c |  D, |W| < n_{min}, \breve{m}, n_2)
\\&\leq e^\epsilon p(S \cap \{null\} | \hat{D}, |W| < n_{min}, \breve{m}, n_2) + \delta \leq
e^\epsilon p(S | \hat{D}, |W| < d, \breve{m}, n_2) + \delta
\end{split}
\end{equation}
where the first inequality is true from eq. \ref{eq-pts-1-1} and the Laplace mechanism's privacy guarantees for $n_W$.

\begin{equation}\label{eq-pts-1-1}
\begin{split}
{}& p(null|  D, |W| < n_{min}, \breve{m}, n_2) = 
p(n_W < n_{min} + \frac1\epsilon\log(\frac1{2\delta}) | D, |W| < n_{min}, \breve{m}, n_2) \geq  
p(l < \frac1\epsilon\log(\frac1{2\delta})) \geq 1 - \delta
\end{split}\end{equation}
\end{proof}

\begin{claim}\label{pts-2}
Line 25 of Algorithm \ref{A-PTR} is $(\epsilon, \delta)$ differentially private with respect to $D$ for $|W| \geq n_{min}$ and given $n_2, \breve{m}$.
\end{claim}
\begin{proof}[Proof of Claim \ref{pts-2}]
For a given $n_2$ and $\breve{m}$ the group $W$ can change by up to one sample for a neighbouring dataset. Mark $n = |W|$ and $c = \breve{m}$. As $n \geq n_2^{\frac{\rho_2}{\gamma_1}}$, then $n^\frac12 > n^{\gamma_1} \geq n_2^{\rho_2}$, and therefore $W \in \mathcal{D}(n, \gamma_1, x_h, x_l, c)$, as defined in eq. \ref{eq-Domain}.

Because $W \in \mathcal{D}(n, \gamma_1, x_h, x_l, c)$, $n \geq n_{b1}$, and $n \geq n_{b2}$, the problem of sampling from $p(\theta|W)$ for $|W| \geq n_{min}$ holds the constraints of Claim \ref{claim-Post-Div-9-1}. Therefore one sample from $p(\theta|W)$ is $(\epsilon, \delta)$ differentially private.
\end{proof}

\begin{claim}\label{ptr-3}
Lines 19-24 of Algorithm \ref{A-PTR} are $(\epsilon, 0)$ differentially private with respect to $D$ for $|W| \geq n_{min}$ and given $\breve{m}, n_2$.
\end{claim}
\begin{proof}[Proof of Claim \ref{ptr-3}]
The only data released in lines 19-24 is $n_W$. Since the $l_1$-sensitivity of $|W|$ given $\breve{m}, n_2$ is 1, then the Laplace mechanism ensures $(\epsilon, 0)$ differential privacy.
\end{proof}

\begin{corollary}\label{cor-pts-3}
Lines 19-25 of Algorithm \ref{A-PTR} are $(2\epsilon, \delta)$ differentially private with respect to $D$ for $|W| \geq n_{min}$ and given $\breve{m}, n_2$.
\end{corollary}
Corollary \ref{cor-pts-3} follows directly from Claims \ref{ptr-3} and \ref{pts-2}.

\begin{corollary}\label{cor-pts-4}
Lines 19-25 of Algorithm \ref{A-PTR} are $(2\epsilon, \delta)$ differentially private with respect to $D$ given $\breve{m}, n_2$.
\end{corollary}
Corollary \ref{cor-pts-4} follows directly from Claims \ref{cor-pts-3} and \ref{pts-1}.

\newpage
\section{Auxiliary Claims}\label{sec::auxiliary_claims}
Appendix \ref{sec::auxiliary_claims} contains claims used to simplify the otherwise complex proofs throughout the paper. 
\subsection{Stochastic Gradient Langevin Dynamics Privacy}
This subsection provides auxiliary claims for SGLD privacy analysis performed in subsection \ref{subseq::sgld-detailed}. It uses the notations defined in section \ref{Method}, subsection \ref{subseq-SGLD}, and subection \ref{subseq::sgld-detailed}, specifically: $\alpha, \beta, \theta, p(y|x, \theta)$ (defined in eq. \ref{LinearModel}), $\mathcal{D}(n, \gamma_1, x_h, x_l, c)$ (defined in eq. \ref{eq-Domain}), $D_1, D_2$ (defined in eq. \ref{N-DBs}), $\eta, \theta_j, \hat{\theta}_j, \mu_j, \hat{\mu}^r_j, \sigma_j, \hat{\sigma}^r_j$ (defined in subsection \ref{subseq-SGLD}), and $\lambda, \hat{\lambda}, \rho, \hat{\rho}$ (defined in eq. \ref{Eqq-App-SGLD-markings}).

\begin{claim}\label{M1}
$\rho\frac{1}{1 - \lambda} = \frac{ncx_h^2\beta}{\alpha + nx_h^2\beta}$.
\end{claim}

\begin{proof}[Proof of Claim \ref{M1}]
\begin{equation*}
\rho\frac{1}{1 - \lambda} = \frac{\eta}{2}ncx_h^2\beta\frac{1}{1 - \left(1 - \frac{\eta}{2}\left(\alpha + n\left(\frac{x_h}{2}\right)^2\beta\right)\right)} = ncx_h^2\beta\frac{1}{\alpha + nx_h^2\beta} = \frac{ncx_h^2\beta}{\alpha + nx_h^2\beta}.
\end{equation*}
\end{proof}

\begin{claim}\label{M2}
$\rho\frac{1 - \lambda^{n-1}}{1 - \lambda} + \rho\lambda^{n-1} = \frac{ncx_h^2\beta}{\alpha + nx_h^2\beta}\left(1 -\lambda^n\right)$.
\end{claim}

\begin{proof}[Proof of Claim \ref{M2}]
\begin{equation*}
\rho\frac{1 - \lambda^{n-1}}{1 - \lambda} + \rho\lambda^{n-1} = 
\rho\left(\frac{1 - \lambda^{n-1}+\lambda^{n-1} - \lambda^n}{1 - \lambda}\right) = 
\rho\left(\frac{1  - \lambda^n}{1 - \lambda}\right) =
\frac{ncx_h^2\beta}{\alpha + nx_h^2\beta}\left(1 -\lambda^n\right)
\end{equation*}
where the last equality holds from Claim \ref{M1}.
\end{proof}

\begin{claim}\label{M3}
$\rho\left(\frac{1-\lambda^{n-1}}{1-\lambda}\right) + \hat\rho\lambda^{n-1} = \frac{ncx_h^2\beta}{\alpha + nx_h^2\beta}\left(1 - \lambda^n\left(\frac{3}{4}\lambda^{-1} + \frac{1}{4}\right)\right)$.
\end{claim}

\begin{proof}[Proof of Claim \ref{M3}]
\begin{equation*}\begin{split}
\rho\left(\frac{1-\lambda^{n-1}}{1-\lambda}\right) + \hat\rho\lambda^{n-1} & = 
\rho\left(\frac{1-\lambda^{n-1}}{1-\lambda}\right) + \rho\frac{1}{4}\lambda^{n-1} =
\rho\left(\frac{1 - \frac{3}{4}\lambda^{n-1} - \frac{1}{4}\lambda^n}{1-\lambda}\right)
\\{}& =
\rho\left(\frac{1 - \lambda^n\left(\frac{3}{4}\lambda^{-1} + \frac{1}{4}\right)}{1-\lambda}\right) =
\frac{ncx_h^2\beta}{\alpha + nx_h^2\beta}\left(1 - \lambda^n\left(\frac{3}{4}\lambda^{-1} + \frac{1}{4}\right)\right)
\end{split}\end{equation*}
where the last equality holds from Claim \ref{M1}.
\end{proof}

\begin{claim}\label{M5.a}
$\frac{1}{4}\lambda + \frac34 - \hat\lambda = \frac34\frac\eta2\alpha$.
\end{claim}
\begin{proof}[Proof of Claim \ref{M5.a}]
\begin{equation*}\begin{split}
\frac{1}{4}\lambda + \frac34 - \hat\lambda &= \frac{1}{4}\left(1 - \frac\eta2\left(\alpha + nx^2\beta\right)\right) + \frac34  - \left(1 - \frac\eta2\left(\alpha + \frac{1}{4}nx^2\beta\right)\right)
\\&= \frac\eta2\left(\alpha + \frac14nx^2\beta - \frac14\left(\alpha + nx^2\beta\right)\right) = \frac34\frac\eta2\alpha.
\end{split}\end{equation*}
\end{proof}

\begin{claim}\label{M5.f}
$\forall k\in\mathbb{Z}_{>0}:$
\begin{equation*}\begin{split}
\left(1 - \lambda^{kn}\right)&\left(1 - \hat\lambda\lambda^{n-1}\right) - \left(1 - \left(\hat\lambda\lambda^{\left(n-1\right)}\right)\right)^k\left(1 - \lambda^{n-1}\left(\frac14\lambda +\frac34\right)\right) 
\\&=
\lambda^{n-1}\frac34\frac\eta2\alpha\left(1 - \hat\lambda^k\lambda^{k\left(n-1\right)}\right)  + \lambda^{k\left(n-1\right)}\left(\hat\lambda^k - \lambda^k\right)\left(1 - \lambda^{n-1}\hat\lambda\right).
\end{split}\end{equation*}
\end{claim}
\begin{proof}[Proof of Claim \ref{M5.f}]
\begin{equation*}\begin{split}
\left(1 - \lambda^{kn}\right)&\left(1 - \hat\lambda\lambda^{n-1}\right) - \left(1 - \left(\hat\lambda\lambda^{\left(n-1\right)}\right)\right)^k\left(1 - \lambda^{n-1}\left(\frac14\lambda +\frac34\right)\right) 
\\& =
\lambda^{n-1}\left(\frac14\lambda + \frac34 - \hat\lambda\right) + \lambda^{kn}\left( \hat\lambda\lambda^{n-1} - 1\right) + \left(\hat\lambda\lambda^{n-1}\right)^k\left(1 - \lambda^{n-1}\left(\frac14\lambda + \frac34\right)\right)
\\&= \lambda^{n-1}\left(\frac14\lambda + \frac34 - \hat\lambda\right) + \lambda^{k\left(n-1\right)}\left(\lambda^k\left( \hat\lambda\lambda^{n-1} - 1\right) + \hat\lambda^k\left(1 - \lambda^{n-1}\left(\frac14\lambda + \frac34\right)\right)\right)
\\&= \lambda^{n-1}\left(\frac14\lambda + \frac34 - \hat\lambda\right) + \lambda^{k\left(n-1\right)}\left(\hat\lambda^k\left(1 - \lambda^{n-1}\left(\frac14\lambda + \frac34\right)\right) - \lambda^k\left(1 - \hat\lambda\lambda^{n-1}\right)\right)
\\&= \lambda^{n-1}\left(\frac14\lambda + \frac34 - \hat\lambda\right) + \lambda^{k\left(n-1\right)}\left(\hat\lambda^k\left(1 - \lambda^{n-1}\left(\frac14\lambda + \frac34\right)\right) - \lambda^k\left(1 - \lambda^{n-1}\hat\lambda\right)\right) 
\\& =^* \lambda^{n-1}\frac\eta2\frac34\alpha + \lambda^{k\left(n-1\right)}\left(\hat\lambda^k\left(1 - \lambda^{n-1}\left(\frac14\lambda + \frac34\right)\right) - \lambda^k\left(1 - \lambda^{n-1}\hat\lambda\right)\right) 
\\&=^* \lambda^{n-1}\frac\eta2\frac34\alpha + \lambda^{k\left(n-1\right)}\left(\hat\lambda^k\left(1 - \lambda^{n-1}\left(\hat\lambda + \frac34\frac\eta2\alpha\right)\right) - \lambda^k\left(1 - \lambda^{n-1}\hat\lambda\right)\right)
\\&= \lambda^{n-1}\frac\eta2\frac34\alpha - \hat\lambda^k\lambda^{n-1}\lambda^{k(n-1)}\frac34\frac\eta2\alpha  + \lambda^{k\left(n-1\right)}\left(\hat\lambda^k\left(1 - \lambda^{n-1}\hat\lambda\right) - \lambda^k\left(1 - \lambda^{n-1}\hat\lambda\right)\right) 
\\&= \lambda^{n-1}\frac34\frac\eta2\alpha\left(1 - \hat\lambda^k\lambda^{k\left(n-1\right)}\right)  + \lambda^{k\left(n-1\right)}\left(\hat\lambda^k - \lambda^k\right)\left(1 - \lambda^{n-1}\hat\lambda\right).
\end{split}\end{equation*}
where equality signs marked by * hold from Claim \ref{M5.a}.
\end{proof}

\begin{claim}\label{M4.a}
$\forall k\in\mathbb{Z}_{>0}:$ $\lambda\sum_{j=0}^{k-1}\lambda^{\left(n-1\right)j}\lambda^j\left(\lambda^{n-1}\rho + \rho\sum_{i=0}^{n-2}\lambda^i\right) = \lambda\left(1 - \lambda^{kn}\right)\frac{ncx_h^2\beta}{\alpha + nx_h^2\beta}$.
\end{claim}

\begin{proof}[Proof of Claim \ref{M4.a}]
\begin{equation*}\begin{split}
{}& \lambda\sum_{j=0}^{k-1}\lambda^{\left(n-1\right)j}\lambda^j\left(\lambda^{n-1}\rho + \rho\sum_{i=0}^{n-2}\lambda^i\right) = 
\rho\lambda\sum_{i=0}^{kn-1}\lambda^i = \rho\lambda\frac{1 - \lambda^{kn}}{1 - \lambda} =^* \lambda\frac{ncx_h^2\beta}{\alpha + nx_h^2\beta}\left(1 - \lambda^{kn}\right)
\end{split}\end{equation*}
where equality * follows from Claim \ref{M1}.
\end{proof}

\begin{claim}\label{M4.b}
$\forall k\in\mathbb{Z}_{>0}:$ $\hat\lambda\sum_{j=0}^{k-1}\lambda^{\left(n-1\right)j}\hat\lambda^j\left(\lambda^{n-1}\hat\rho + \rho\sum_{i=0}^{n-2}\lambda^i\right) = \hat\lambda\frac{1 - \left(\lambda^{n-1}\hat\lambda\right)^k}{1 - \lambda^{n-1}\hat\lambda}\frac{ncx_h^2\beta}{\alpha + nx_h^2\beta}\left(1 - \lambda^n\left(\frac{3}{4}\lambda^{-1} + \frac{1}{4}\right)\right)$.
\end{claim}

\begin{proof}[Proof of Claim \ref{M4.b}]
\begin{equation*}\begin{split}
\hat\lambda\sum_{j=0}^{k-1}\lambda^{\left(n-1\right)j}\hat\lambda^j\left(\lambda^{n-1}\hat\rho + \rho\sum_{i=0}^{n-2}\lambda^i\right) 
&=
\hat\lambda\frac{1 - \left(\lambda^{n-1}\hat\lambda\right)^k}{1 - \lambda^{n-1}\hat\lambda}\left(\lambda^{n-1}\hat\rho + \rho\frac{1-\lambda^{n-1}}{1-\lambda}\right) 
\\&=^*
\hat\lambda\frac{1 - \left(\lambda^{n-1}\hat\lambda\right)^k}{1 - \lambda^{n-1}\hat\lambda}\frac{ncx_h^2\beta}{\alpha + nx_h^2\beta}\left(1 - \lambda^n\left(\frac{3}{4}\lambda^{-1} + \frac{1}{4}\right)\right)
\end{split}\end{equation*}
where equality * follows from Claims \ref{M1} and \ref{M3}.
\end{proof}

\begin{claim}\label{M5.c}
$\forall k\in\mathbb{R}_{>0}:$
$\lambda\lambda^k  - \lambda^k\lambda^{n}\hat\lambda - \hat\lambda\hat\lambda^k+\hat\lambda\hat\lambda^k\lambda^n\left(\frac34\lambda^{-1} + \frac14\right) = \left(1 - \hat\lambda\lambda^{n-1}\right)\left(\lambda^{k+1} - \hat\lambda^{k+1}\right) + \hat\lambda^{k+1}\lambda^{n-1}\left(\frac34\frac\eta2\alpha\right)$.
\end{claim}

\begin{proof}[Proof of Claim \ref{M5.c}]
\begin{equation*}\begin{split}
\lambda\lambda^k  - \lambda^k\lambda^{n}\hat\lambda - \hat\lambda\hat\lambda^k+\hat\lambda\hat\lambda^k\lambda^n\left(\frac34\lambda^{-1} + \frac14\right) &= 
\lambda^{k+1}\left(1 - \hat\lambda\lambda^{n-1}\right) - \hat\lambda^{k+1}\left(1 - \lambda^{n-1}\left(\frac14\lambda + \frac34\right)\right) 
\\&=^*
\lambda^{k+1}\left(1 - \hat\lambda\lambda^{n-1}\right) - \hat\lambda^{k+1}\left(1 - \lambda^{n-1}\left(\hat\lambda + \frac34\frac\eta2\alpha\right)\right) 
\\&= 
\left(1 - \hat\lambda\lambda^{n-1}\right)\left(\lambda^{k+1} - \hat\lambda^{k+1}\right) + \hat\lambda^{k+1}\lambda^{n-1}\left(\frac34\frac\eta2\alpha\right)
\end{split}\end{equation*}
where equality * holds from Claim \ref{M5.a}.
\end{proof}

\begin{claim}\label{M5.d}
$\forall k\in\mathbb{Z}_{>0}:$
\begin{align*}
\lambda\left(1 - \lambda^{kn}\right)&\left(1 - \lambda^{n-1}\hat\lambda\right) - \hat\lambda\left(1 - \left(\lambda^{n-1}\hat\lambda\right)^k\right)\left(1 - \lambda^n\left(\frac34\lambda^{-1} + \frac14\right)\right)
\\&= 
\left(\lambda - \hat\lambda\right)\left(1 - \hat\lambda\lambda^{n-1}\right) + \lambda^{n-1}\hat\lambda\left(\frac34\frac\eta2\alpha\left(1 - \hat\lambda^k\lambda^{k\left(n-1\right)}\right)\right) + \lambda^{k\left(n-1\right)}\left(\left(1 - \hat\lambda\lambda^{n-1}\right)\left({\hat\lambda}^{k+1} - \lambda^{k+1}\right)\right).
\end{align*}
\end{claim}
\begin{proof}[Proof of Claim \ref{M5.d}]
\begin{equation*}\begin{split}
\lambda\left(1 - \lambda^{kn}\right)&\left(1 - \lambda^{n-1}\hat\lambda\right) - \hat\lambda\left(1 - \left(\lambda^{n-1}\hat\lambda\right)^k\right)\left(1 - \lambda^n\left(\frac34\lambda^{-1} + \frac14\right)\right) 
\\&= 
\lambda -\hat\lambda - \lambda^n\hat\lambda\left(1 - \left(\frac34\lambda^{-1} + \frac14\right)\right) - \lambda^{k\left(n-1\right)}\left(\lambda\lambda^k  - \lambda^k\lambda^{n}\hat\lambda - \hat\lambda\hat\lambda^k+\hat\lambda\hat\lambda^k\lambda^n\left(\frac34\lambda^{-1} + \frac14\right)\right) 
\\&=^*
\lambda -\hat\lambda - \lambda^n\hat\lambda\left(1 - \left(\frac34\lambda^{-1} + \frac14\right)\right) - \lambda^{k\left(n-1\right)}\left(\left(1 - \hat\lambda\lambda^{n-1}\right)\left(\lambda^{k+1} - \hat\lambda^{k+1}\right) + \hat\lambda^{k+1}\lambda^{n-1}\left(\frac34\frac\eta2\alpha\right)\right) 
\\&=
\lambda -\hat\lambda - \lambda^{n-1}\hat\lambda\left(\lambda - \left(\frac34 + \frac14\lambda\right)\right) - \lambda^{k\left(n-1\right)}\left(\left(1 - \hat\lambda\lambda^{n-1}\right)\left(\lambda^{k+1} - \hat\lambda^{k+1}\right) + \hat\lambda^{k+1}\lambda^{n-1}\left(\frac34\frac\eta2\alpha\right)\right) 
\\&=^{**}
\lambda -\hat\lambda - \lambda^{n-1}\hat\lambda\left(\lambda - \left(\hat\lambda + \frac34\frac\eta2\alpha\right)\right) - \lambda^{k\left(n-1\right)}\left(\left(1 - \hat\lambda\lambda^{n-1}\right)\left(\lambda^{k+1} - \hat\lambda^{k+1}\right) + \hat\lambda^{k+1}\lambda^{n-1}\left(\frac34\frac\eta2\alpha\right)\right) 
\\&=
\left(\lambda - \hat\lambda\right)\left(1 - \hat\lambda\lambda^{n-1}\right) + \lambda^{n-1}\hat\lambda\left(\frac34\frac\eta2\alpha\left(1 - \hat\lambda^k\lambda^{k\left(n-1\right)}\right)\right) + \lambda^{k\left(n-1\right)}\left(\left(1 - \hat\lambda\lambda^{n-1}\right)\left({\hat\lambda}^{k+1} - \lambda^{k+1}\right)\right)
\end{split}\end{equation*}
where equality * follows from Claim \ref{M5.c} and equality ** follows from Claim \ref{M5.a}.
\end{proof}

\begin{claim}\label{M5.e}
$\forall k\in\mathbb{Z}_{>0}:$
\begin{equation*}\begin{split}
\lambda\left(1 - \lambda^{kn}\right) &- \hat\lambda\frac{1 - \left(\lambda^{n-1}\hat\lambda\right)^k}{1 - \lambda^{n-1}\hat\lambda}\left(1 - \lambda^n\left(\frac34\lambda^{-1} + \frac14\right)\right)
\\&= \left(\lambda - \hat\lambda\right) + \frac{\lambda^{n-1}\left(\frac34\frac\eta2\alpha\left(1 - \hat\lambda^k\lambda^{k\left(n-1\right)}\right)\right)}{1 - \hat\lambda\lambda^{n-1}} + \lambda^{k\left(n-1\right)}\left(\hat\lambda^{k+1} - \lambda^{k+1}\right).
\end{split}\end{equation*}
\end{claim}

\begin{proof}[Proof of Claim \ref{M5.e}]
\begin{equation*}\begin{split}
\lambda\left(1 - \lambda^{kn}\right) & - \hat\lambda\frac{1 - \left(\lambda^{n-1}\hat\lambda\right)^k}{1 - \lambda^{n-1}\hat\lambda}\left(1 - \lambda^n\left(\frac34\lambda^{-1} + \frac14\right)\right)
\\&= \frac{\left(\lambda - \hat\lambda\right)\left(1 - \hat\lambda\lambda^{n-1}\right) + \lambda^{n-1}\hat\lambda\left(\frac34\frac\eta2\alpha\left(1 - \hat\lambda^k\lambda^{k\left(n-1\right)}\right)\right) + \lambda^{k\left(n-1\right)}\left(\left(1 - \hat\lambda\lambda^{n-1}\right)\left({\hat\lambda}^{k+1} - \lambda^{k+1}\right)\right)}{\left(1 - \hat\lambda \lambda^{n-1}\right)}
\\&= 
\left(\lambda - \hat\lambda\right) + \frac{\lambda^{n-1}\left(\frac34\frac\eta2\alpha\left(1 - \hat\lambda^k\lambda^{k\left(n-1\right)}\right)\right)}{1 - \hat\lambda\lambda^{n-1}} + \lambda^{k\left(n-1\right)}\left(\hat\lambda^{k+1} - \lambda^{k+1}\right)
\end{split}\end{equation*}
where first equality is true from Claim \ref{M5.d}.
\end{proof}

\begin{claim}\label{M6}
$\forall k\in\mathbb{Z}_{>0}:$ $\frac{ncx_h^2\beta}{\alpha + nx_h^2\beta}\left(\lambda - \hat\lambda + \frac{\lambda^{n-1}\left(\frac34\frac\eta2\alpha\left(1 - \hat\lambda^k\lambda^{k\left(n-1\right)}\right)\right)}{1 - \hat\lambda\lambda^{n-1}}\right) + \left(\rho - \hat\rho\right) > 0$.
\end{claim}

\begin{proof}[Proof of Claim \ref{M6}]
\begin{equation*}\begin{split}
\frac{ncx_h^2\beta}{\alpha + nx_h^2\beta}&\left(\lambda - \hat\lambda + \frac{\lambda^{n-1}\left(\frac34\frac\eta2\alpha\left(1 - \hat\lambda^k\lambda^{k\left(n-1\right)}\right)\right)}{1 - \hat\lambda\lambda^{n-1}}\right) + \left(\rho - \hat\rho\right) \\&= 
\frac{ncx_h^2\beta}{\alpha + nx_h^2\beta}\left(\lambda - \hat\lambda + \frac{\lambda^{n-1}\left(\frac34\frac\eta2\alpha\left(1 - \hat\lambda^k\lambda^{k\left(n-1\right)}\right)\right)}{1 - \hat\lambda\lambda^{n-1}}\right) + \frac{\eta}{2}ncx_h^2\beta\left(1 - \frac14\right)
\\&=
\frac{ncx_h^2\beta}{\alpha + nx_h^2\beta}(1 - \frac{\eta}{2}\left(\alpha + nx^2\beta\right) - \left(1 - \frac\eta2\left(\alpha + \frac14nx^2\beta\right)\right)
\\&+ \frac{\lambda^{n-1}\left(\frac34\frac\eta2\alpha\left(1 - \hat\lambda^k\lambda^{k\left(n-1\right)}\right)\right)}{1 - \hat\lambda\lambda^{n-1}}) + \frac34\frac{\eta}{2}ncx_h^2\beta 
\\&= 
\frac{ncx_h^2\beta}{\alpha + nx_h^2\beta}\left( - \frac34\frac{\eta}{2}\left(nx^2\beta\right)  + \frac{\lambda^{n-1}\left(\frac34\frac\eta2\alpha\left(1 - \hat\lambda^k\lambda^{k\left(n-1\right)}\right)\right)}{1 - \hat\lambda\lambda^{n-1}}\right) + \frac34\frac{\eta}{2}ncx_h^2\beta
\\& =
\frac{ncx_h^2\beta}{\alpha + nx_h^2\beta}\frac{\lambda^{n-1}\left(\frac34\frac\eta2\alpha\left(1 - \hat\lambda^k\lambda^{k\left(n-1\right)}\right)\right)}{1 - \hat\lambda\lambda^{n-1}} + 
ncx_h^2\beta\left(\frac34\frac\eta2  - \frac34\frac\eta2\frac{nx^2\beta}{\alpha + nx^2\beta}\right)
\\&=
\frac{ncx_h^2\beta}{\alpha + nx_h^2\beta}\frac{\lambda^{n-1}\left(\frac34\frac\eta2\alpha(1 - \hat\lambda^k\lambda^{k(n-1)})\right)}{1 - \hat\lambda\lambda^{n-1}} + ncx_h^2\beta\frac34\frac\eta2\left(1  - \frac{nx^2\beta}{\alpha + nx^2\beta}\right) \\&> 0
\end{split}\end{equation*}
where the last inequality holds because $\lambda,\hat\lambda < 1$ and $\alpha > 0$.
\end{proof}

\begin{claim}\label{V1.1}
$\frac1\alpha > \lambda^{-2}\eta\frac{1 - \lambda^{-2\left(k+1\right)n}}{1 - \lambda^{-2}}$ is true for all $k$ such that $0 \leq k \leq \frac{1}{2n}\log_\lambda\left(\frac{1}{1 + \frac{1}{\alpha\eta}\left(1 - \lambda^{2}\right)}\right) - 1$
\end{claim}

\begin{proof}[Proof of Claim \ref{V1.1}]
\begin{equation*}\begin{split}
{}& \frac1\alpha \geq \lambda^{-2}\eta\frac{1 - \lambda^{-2{k}n}}{1 - \lambda^{-2}} \iff
\\{}&\lambda^2\frac1\alpha\frac1\eta\left(1 - \lambda^{-2}\right) \leq {1 - \lambda^{-2{k}n}} \iff 
\\{}& \lambda^2\frac1\alpha\frac1\eta\left(\lambda^{-2} - 1\right) \geq {\lambda^{-2{k}n} - 1} \iff
\\{}&1 + \lambda^2\frac1\alpha\frac1\eta\left(\lambda^{-2} - 1\right) \geq \lambda^{-2{k}n} \iff
\\{}&
-{k} \geq \frac{1}{2n}\log_{\lambda}\left(1 + \frac{1}{\alpha\eta}\left(1 - \lambda^2\right)\right) \iff
\\{}&{k} \leq \frac{1}{2n}\log_\lambda\left(\frac{1}{1 + \frac{1}{\alpha\eta}\left(1 - \lambda^{2}\right)}\right)
\end{split}\end{equation*}
\end{proof}

\begin{claim}\label{V1.2}
$\frac{1}{\alpha}\left(\hat\lambda\lambda^{\left(n-1\right)}\right)^{2{k}} > \eta\sum_{i=0}^{n-1}\lambda^{2i}\sum_{j=0}^{{k}-1}\left(\hat\lambda^2\lambda^{2\left(n-1\right)}\right)^j$ is true for all $k \in \mathbb{Z}_{>0} : k \leq \frac{1}{2n}\log_\lambda\left(\frac{1}{1 + \frac{1}{\alpha\eta}\left(1 - \lambda^{2}\right)}\right)$.
\end{claim}

\begin{proof}[Proof of Claim \ref{V1.2}]
First, note that the inequality can also be written as $$\frac{1}{\alpha} > \eta\sum_{i=0}^{n-1}\lambda^{2i}\sum_{j=0}^{k-1}\left(\hat\lambda\lambda^{\left(n-1\right)}\right)^{2\left(j - k\right)}.$$ 
Second, the right-hand term of the inequality could be upper bound as in eq. \ref{Eq-App-V1.2}. Therefore, for the claim's inequality to hold, it is enough that $\frac{1}{\alpha} \geq \eta\lambda^{-2}\frac{1 - \lambda^{-2nk}}{1 - \lambda^{-2}}$, which is proved by Claim \ref{V1.1} to be true for $0 < {k} \leq \frac{1}{2n}\log_\lambda\left(\frac{1}{1 + \frac{1}{\alpha\eta}\left(1 - \lambda^{2}\right)}\right)$.

\begin{equation}\begin{split}\label{Eq-App-V1.2}
\eta\sum_{i=0}^{n-1}\lambda^{2i}\sum_{j=0}^{k-1}\left(\hat\lambda\lambda^{\left(n-1\right)}\right)^{2\left(j - k\right)} &= 
\eta\sum_{i=0}^{n-1}\lambda^{2i}\sum_{j=0}^{k-1}\frac1{\left(\hat\lambda\lambda^{\left(n-1\right)}\right)^{2\left(k - j\right)}} \\&<_{k > j}
\eta\sum_{i=0}^{n-1}\lambda^{2i}\sum_{j=0}^{k-1}\frac1{\left(\lambda\lambda^{\left(n-1\right)}\right)^{2\left(k - j\right)}} = 
\eta\sum_{i=0}^{n-1}\lambda^{2i}\sum_{j=0}^{k-1}\frac1{\lambda^{2n\left(k - j\right)}}
\\&= \eta\sum_{i=0}^{n-1}\sum_{j=0}^{k-1}\frac1{\lambda^{2\left(nk - nj - i\right)}} =_{r = nj+i}
\eta\sum_{r=0}^{nk-1}\frac1{\lambda^{2\left(nk - r\right)}} =_{r' = nk - r, 1 < r' < nk}
\eta\sum_{r'=1}^{nk}\frac1{\lambda^{2\left(r'\right)}}  
\\&= 
\eta\sum_{i=1}^{nk}\lambda^{-2i} =
\eta\frac{\lambda^{-2} - \lambda^{-2\left(nk + 1\right)}}{1 - \lambda^{-2}} = 
\eta\lambda^{-2}\frac{1 - \lambda^{-2nk}}{1 - \lambda^{-2}}
\end{split}\end{equation}
\end{proof}

\begin{claim}\label{V2.1}
$\frac1\alpha\left(\hat\lambda\lambda^{n-1}\right)^{2{k}} \geq \eta\left(\hat\lambda\lambda^{n-1}\right)^{2{k}}\sum_{i=0}^{n-1}\lambda^{2i}$ is true for all $k \in \mathbb{Z}_{>0} : {k} \leq \frac{1}{2n}\log_\lambda\left(\frac{1}{1 + \frac{1}{\alpha\eta}\left(1 - \lambda^{2}\right)}\right)$.
\end{claim}

\begin{proof}[Proof of Claim \ref{V2.1}] Eq. \ref{V2.1-1} holds because $\lambda, \hat\lambda < 1$. 
By multiplying both sides with $\sum_{i=0}^{n-1}\lambda^{2i}$, we get eq. \ref{V2.1-2}.
Then, noticing that the right term equals to the right term of Claim \ref{V1.2} inequality, and hence smaller than the left term of Claim \ref{V1.2} inequality, Claim \ref{V2.1} is proved.

\begin{equation}\label{V2.1-1}
\left(\hat\lambda\lambda^{n-1}\right)^{2k} < 1 < \sum_{i=0}^{k-1}\left(\hat\lambda\lambda^{n-1}\right)^{2j}    
\end{equation}
\begin{equation}\label{V2.1-2}
\eta\left(\hat\lambda\lambda^{n-1}\right)^{2\dot{k}}\sum_{i=0}^{n-1}\lambda^{2i} < \eta\sum_{j=0}^{\dot{k}-1}\left(\hat\lambda\lambda^{n-1}\right)^{2j}\sum_{i=0}^{n-1}\lambda^{2i}
\end{equation}
\end{proof}

\begin{claim}\label{V2.2}
The inequality 
\begin{equation*}
\left(\hat\lambda\lambda^{n-r}\right)^2\left(\frac1\alpha\left(\hat\lambda\lambda^{n-1}\right)^{2k} + \eta\sum_{j=0}^{k-1}\left(\hat\lambda\lambda^{n-1}\right)^{2j}\sum_{i=0}^{n-1}\lambda^{2i}\right) > \eta\sum_{i=0}^{n-r}\lambda^{2i}
\end{equation*} holds for all $k \in \mathbb{Z}_{>0}$ and $x_h^2\beta > 3, n > \frac{1}{2\alpha x_h^2\beta} - \frac{1}{x_h^2\beta}$. 
\end{claim}
\begin{proof}[Proof of Claim \ref{V2.2}]
We start with eq. \ref{Eq-App-V2.2-1}, where the first inequality holds because $\lambda < 1$ and $r > 1$, and the second inequality holds because $\lambda < \hat\lambda$. Using eq. \ref{Eq-App-V2.2-1}, we can lower-bound the left-hand side of the claim's inequality. We continue with eq. \ref{Eq-App-V2.2-2}, where the inequality holds because $\lambda < \hat\lambda$ and $r > 1$. This allows us to upper-bound the right side of the claim's inequality. Given these lower and upper bounds, it's enough to show that $\lambda^{2n}(\frac1\alpha\lambda^{2kn} + \eta\frac{1-\lambda^{2kn}}{1-\lambda^2}) > \eta\frac{1-\lambda^{2n}}{1-\lambda^2}$, which according to eq. \ref{Eq-App-V2.2-3} is equivalent to showing that $(2nx_h^2\beta - 1)\frac1\alpha\lambda^{2(k+1)n} + 2(2\lambda^{2n} - 1) > 0$. Since $n > \frac{1}{2\alpha x_h^2\beta} - \frac{1}{x_h^2\beta}$, Claim \ref{G1.b} applies, and therefore $\lambda^{2n} \geq e^{-\frac{2}{x_h^2\beta}}$. Consequently, it's enough to show that $(2nx_h^2\beta - 1)\frac1\alpha\lambda^{2(k+1)n} + 2(2e^{-\frac{2}{x_h^2\beta}} - 1) > 0$, which is true for $x_h^2\beta > 3$ by Claim \ref{V2.3}. 
\begin{equation}\label{Eq-App-V2.2-1}\begin{split}
\left(\hat\lambda\lambda^{n-r}\right)^2&\left(\frac1\alpha\left(\hat\lambda\lambda^{n-1}\right)^{2k} + \eta\sum_{j=0}^{k-1}\left(\hat\lambda\lambda^{n-1}\right)^{2j}\sum_{i=0}^{n-1}\lambda^{2i}\right) 
>
\left(\hat\lambda\lambda^{n-1}\right)^2\left(\frac1\alpha\left(\hat\lambda\lambda^{n-1}\right)^{2k} + \eta\sum_{j=0}^{k-1}\left(\hat\lambda\lambda^{n-1}\right)^{2j}\sum_{i=0}^{n-1}\lambda^{2i}\right)
\\&> \lambda^{2n}\left(\frac1\alpha\lambda^{2kn} + \eta\sum_{j=0}^{k-1}\lambda^{2jn}\sum_{i=0}^{n-1}\lambda^{2i}\right)=
\lambda^{2n}\left(\frac1\alpha\lambda^{2kn} + \eta\frac{1-\lambda^{2kn}}{1-\lambda^2}\right)
\end{split}\end{equation}

\begin{equation}\label{Eq-App-V2.2-2}
\eta\sum_{i=0}^{n-r}\lambda^{2i} < \eta\sum_{i=0}^{n-1}\lambda^{2i} = \eta\frac{1-\lambda^{2n}}{1-\lambda^2}
\end{equation}

\begin{equation}\begin{split}\label{Eq-App-V2.2-3}
{}&\lambda^{2n}\left(\frac1\alpha\lambda^{2kn} + \eta\frac{1-\lambda^{2kn}}{1-\lambda^2}\right) > \eta\frac{1-\lambda^{2n}}{1-\lambda^2}
\\{}&\lambda^{2n}\left(1-\lambda^2\right)\frac1\alpha\lambda^{2kn} + \eta\lambda^{2n}\left(1-\lambda^{2kn}\right) > \eta\left(1-\lambda^{2n}\right)
\\{}&\left(1-\lambda^2\right)\frac1\alpha\lambda^{2\left(k+1\right)n} + \eta\left(2\lambda^{2n}-\lambda^{2\left(k+1\right)n} - 1\right) > 0
\\{}&\left(\alpha + nx_h^2\beta\right)^2\left(1-\lambda^2\right)\frac1\alpha\lambda^{2\left(k+1\right)n} + 2\left(2\lambda^{2n}-\lambda^{2\left(k+1\right)n} - 1\right) > 0 
\\{}&\left(\alpha + nx_h^2\beta\right)^2\left(1-\left(1 - \frac1{\alpha + nx_h^2\beta}\right)^2\right)\frac1\alpha\lambda^{2\left(k+1\right)n} + 2\left(2\lambda^{2n}-\lambda^{2\left(k+1\right)n} - 1\right) > 0
\\{}&\left(2\left(\alpha + nx_h^2\beta\right) - 1\right)\frac1\alpha\lambda^{2\left(k+1\right)n} + 2\left(2\lambda^{2n}-\lambda^{2\left(k+1\right)n} - 1\right) > 0
\\{}&2\lambda^{2\left(k+1\right)n} + \left(2nx^2\beta - 1\right)\frac1\alpha\lambda^{2\left(k+1\right)n} + 2\left(2\lambda^{2n}-\lambda^{2\left(k+1\right)n} - 1\right) > 0
\\{}&\left(2nx_h^2\beta - 1\right)\frac1\alpha\lambda^{2\left(k+1\right)n} + 2\left(2\lambda^{2n} - 1\right) > 0
\end{split}\end{equation}
\end{proof}

\begin{claim}\label{V2.3}
For $x^2\beta > 3$, the inequality $(2e^{-\frac{2}{x^2\beta}} - 1) > 0$ holds.
\end{claim}
\begin{proof}[Proof of Claim \ref{V2.3}] It's easy to see that the inequality holds only if ${x^2\beta} \geq \frac{-2}{\ln\frac12}$. Since $\frac{-2}{\ln\frac12} < 3$, the claim is proved.
\end{proof}

\begin{claim}\label{V2.4}
For $\dot{k}$ as defined in Lemma \ref{Lemma-SGLD-2.3}, and the conditions of Claim \ref{G1.b}:
\begin{equation*}
\frac1\alpha\left(e^{\frac2{x_h^2\beta}} + \alpha\frac{\left(e^{\frac2{x_h^2\beta}} - 1\right)}{\left(\alpha + nx_h^2\beta\right) + \frac18}\right) > \lambda^{-2}\eta\frac{1 - \lambda^{-2\left(\lceil\dot{k}\rceil+1\right)n}}{1 - \lambda^{-2}}.
\end{equation*}
\end{claim}

\begin{proof}[Proof of Claim \ref{V2.4}]
\begin{equation*}
\begin{split}
\eta\frac{1 - \lambda^{-2\left(\lceil\dot{k}\rceil + 1\right)n}}{\lambda^2-1} &\leq
\eta\frac{\lambda^{-2\left(\dot{k} + 2\right)n} - 1}{1 - \lambda^2}
\\&= \eta\frac{\lambda^{-2\left(\frac{1}{2n}\log_\lambda\left(\frac{1}{1 + \frac{1}{\alpha\eta}\left(1 - \lambda^{2}\right)}\right) - 1 + 2\right)n} - 1}{1 - \lambda^2}
\\&= \eta\frac{\lambda^{-\log_\lambda\left(\frac{1}{1 + \frac{1}{\alpha\eta}\left(1 - \lambda^{2}\right)}\right)}\lambda^{-2n} - 1}{1 - \lambda^2} = 
\eta\frac{\left(1 + \frac{1}{\alpha\eta}\left(1 - \lambda^{2}\right)\right)\lambda^{-2n} - 1}{1 - \lambda^2} 
\\&= \eta\frac{\left(1-\lambda^2\right)\lambda^{-2n}\frac1{\alpha\eta}}{1 - \lambda^2} + \eta\frac{\lambda^{-2n}- 1}{1- \lambda^2} =
\frac1\alpha\lambda^{-2n} + \frac{\left(\lambda^{-2n} - 1\right)}{\left(\alpha + nx_h^2\beta\right) + \frac18} 
\\&\leq e^{\frac2{x_h^2\beta}}\frac1\alpha + \frac1\alpha\alpha\frac{\left(e^{\frac2{x_h^2\beta}} - 1\right)}{\left(\alpha + nx_h^2\beta\right) + \frac18} = 
\frac1\alpha\left(e^{\frac2{x_h^2\beta}} + \alpha\frac{\left(e^{\frac2{x_h^2\beta}} - 1\right)}{\left(\alpha + nx_h^2\beta\right) + \frac18}\right)
\end{split}\end{equation*}
where the fourth equality holds from eq. $\ref{Eq-App-V2.4-1}$ and the second inequality holds from \ref{G1.b}.

\begin{equation}\label{Eq-App-V2.4-1}\begin{split}
\frac{\eta}{\lambda^2 - 1} &= 
\eta\frac{1}{\left(1 - \frac\eta2\left(\alpha + nx_h^2\beta\right)\right)^2 - 1} = 
\eta\frac{1}{\eta\left(\alpha + nx_h^2\beta\right) + \left(\frac\eta2\left(\alpha + nx_h^2\beta\right)\right)^2}
\\&= \frac{1}{\left(\alpha + nx_h^2\beta\right) + \frac\eta4\left(\alpha + nx_h^2\beta\right)^2} = 
\frac{1}{\left(\alpha + nx_h^2\beta\right) + \frac18}    
\end{split}\end{equation}
\end{proof}

\begin{claim}\label{G1.b}
For the conditions of claim \ref{G1.d}:
\begin{equation*}
(1 - \frac1{\alpha + nx^2\beta})^{2n} \geq e^{-\frac2{x^2\beta}}.
\end{equation*}
\end{claim}
\begin{proof}[Proof of Claim \ref{G1.b}] The proof is easily deduced from Claims \ref{G1.b.1} and \ref{G1.d}
\end{proof}

\begin{claim}\label{G1.b.1}
\begin{equation*}
\lim_{n\to\infty}(1 - \frac1{\alpha + nx^2\beta})^{2n} = e^{-\frac2{x^2\beta}}.
\end{equation*}
\end{claim}

\begin{proof}[Proof of Claim \ref{G1.b.1}]

From eq. \ref{Eq-App-G1.b-1}, it is enough to find $\lim_{n\to\infty}\frac{\ln(1 - \frac1{\alpha + nx^2\beta})}{\frac1{2n}}$. Since $\lim_{n\to\infty}\frac{\ln(1 - \frac1{\alpha + nx^2\beta})}{\frac1{2n}} = \frac00$, and both the numerator and denominator are differentiable around $\infty$, the use of L'Hôpital's rule is possible as shown in eq. \ref{Eq-App-G1.b-2}, with the result proving the claim.

\begin{equation}\label{Eq-App-G1.b-1}
\left(1 - \frac1{\alpha + nx^2\beta}\right)^{2n} = 
e^{\ln\left(1 - \frac1{\alpha + nx^2\beta}\right)^{2n}} = 
e^{2n\ln\left(1 - \frac1{\alpha + nx^2\beta}\right)} = 
e^{\frac{\ln\left(1 - \frac1{\alpha + nx^2\beta}\right)}{\frac1{2n}}}
\end{equation}

\begin{equation}\label{Eq-App-G1.b-2}
\lim_{n\to\infty}\frac{\frac{d}{dn}\ln\left(1 - \frac1{\alpha + nx^2\beta}\right)}{\frac{d}{dn}\frac1{2n}} = \lim\frac{\frac{x^2\beta}{\left(\alpha + nx^2\beta - 1\right)\left(\alpha + nx^2\beta\right)}}{-\frac1{2n^2}} = -\lim\frac{2n^2x^2\beta}{\left(nx^2\beta\right)^2} = -\frac{2}{x^2\beta}
\end{equation}
\end{proof}

\begin{claim}\label{G1.d}
\begin{equation*}
\forall n > \frac{1}{2\alpha x^2\beta} - \frac1{x^2\beta} :\ \frac{d}{dn}(1 - \frac1{\alpha + nx^2\beta})^{2n} < 0.
\end{equation*}
\end{claim}

\begin{proof}[Proof of Claim \ref{G1.d}] First, we find a simplified term for the derivative:
\begin{equation}\label{Eq-App-G1.d-1}\begin{split}
\frac{d}{dn}\left(1 - \frac{1}{\alpha + nx^2\beta}\right)^{2n} &= 
\frac{d}{dn}e^{2n\ln\left(1 - \frac1{\alpha + nx^2\beta}\right)}
\\&= \left(1 - \frac{1}{\alpha + nx^2\beta}\right)^{2n}\left(2\ln\left(1 - \frac{1}{\alpha + nx^2\beta}\right) + 2n\frac{1}{1 - \frac{1}{\alpha + nx^2\beta}}\cdot\frac{x^2\beta}{\left(\alpha + nx^2\beta\right)^2}\right)
\\&= \left(1 - \frac{1}{\alpha + nx^2\beta}\right)^{2n}\left(2\ln\left(1 - \frac{1}{\alpha + nx^2\beta}\right) + \frac{2nx^2\beta}{\left(\alpha + nx^2\beta - 1\right)\left(\alpha + nx^2\beta\right)}\right). 
\end{split}
\end{equation}
A lower bound for the $ln$ term can be found using Taylor's theorem as shown in eq. \ref{Eq-App-G1.d-2}, where $0 \leq \xi \leq \frac1{\alpha + nx^2\beta}$.
\begin{equation}\label{Eq-App-G1.d-2}\begin{split}
{}& \ln\left(1 - \frac1{\alpha + nx^2\beta}\right) = 
-\frac{1}{\alpha + nx^2\beta} - \frac12\frac{1}{\left(1-\xi\right)^2}\left(\frac{1}{\alpha +nx^2\beta}\right)^2 \leq
-\frac{1}{\alpha + nx^2\beta} - \frac12\left(\frac{1}{\alpha +nx^2\beta}\right)^2 
\end{split}\end{equation}

From eq. \ref{Eq-App-G1.d-1} and \ref{Eq-App-G1.d-2}, it is enough to find the conditions for which $\frac{nx^2\beta}{(\alpha + nx^2\beta - 1)(\alpha + nx^2\beta)} < \frac1{\alpha + nx^2\beta} + \frac12\frac{1}{(\alpha + nx^2\beta)^2}$. A simplified version of this inequality is found at eq. \ref{Eq-App-G1.d-3}, and it can be easily seen that for $\alpha > \frac12(\frac1{nx^2\beta} + 1)$, and therefore also for $n > \frac{1}{2\alpha x^2\beta} - \frac{1}{x^2\beta}$, this inequality holds.

\begin{equation}\label{Eq-App-G1.d-3}\begin{split}
{}& \frac{nx^2\beta}{(\alpha + nx^2\beta - 1)(\alpha + nx^2\beta)} < \frac1{\alpha + nx^2\beta} + \frac12\frac{1}{(\alpha + nx^2\beta)^2}
\\{}& 0 < 2\alpha^2 + 2nx^2\beta\alpha -2\alpha  - 2nx^2\beta + \alpha + nx^2\beta - 1 
\\{}& 0 < nx^2\beta(2\alpha - 1) + \alpha(2\alpha - 1) - 1
\end{split}\end{equation}
\end{proof}

\begin{claim}\label{G4}
For $n > \frac{\alpha}{x_h^2\beta}(e^{\frac2{x_h^2\beta}} - 2) + \frac1{2x_h^2\beta}$ and the conditions of Claim \ref{G1.b}, $\dot{k}$ defined in Lemma \ref{Lemma-SGLD-2.3} is positive.
\end{claim}

\begin{proof}[Proof of Claim \ref{G4}] The claim's inequality is simplified at eq. \ref{Eq-App-G4.1}:
\begin{equation}\label{Eq-App-G4.1}\begin{split}
{}& \dot{k} > 0
\\{}& \frac{1}{2n}\log_\lambda\left(\frac{1}{1 + \frac{1}{\alpha\eta}\left(1 - \lambda^2\right)}\right) -1 > 0 \iff
\\{}& \log_\lambda\left(\frac{1}{1 + \frac{1}{\alpha\eta}\left(1 - \lambda^2\right)}\right) > 2n \iff
\\{}& \frac{\ln\left(\frac{1}{1 + \frac{1}{\alpha\eta}\left(1 - \lambda^2\right)}\right)}{\ln\lambda} > 2n \iff
\\{}& \ln\left(\frac{1}{1 + \frac{1}{\alpha\eta}\left(1 - \lambda^2\right)}\right) < 2n\ln\lambda \iff
\\{}& \ln\left(\frac{1}{1 + \frac{1}{\alpha\eta}\left(1 - \lambda^2\right)}\right) < \ln\lambda^{2n} \iff
\\{}& \frac{1}{1 + \frac{1}{\alpha\eta}\left(1 - \lambda^2\right)} < \lambda^{2n} \iff
\\{}& \lambda^{-2n} < 1 + \frac{1}{\alpha\eta}\left(1 - \lambda^2\right) \iff
\\{}& \lambda^{-2n} - 1 < \frac{1}{\alpha\eta}\left(1 - \lambda^2\right).
\end{split}\end{equation}
By Claim \ref{G1.b}, $\lambda^{-2n} - 1 < e^{\frac2{x_h^2\beta}} - 1$; therefore it is enough to find conditions for $e^{\frac2{x_h^2\beta}} - 1 < \frac{1}{\alpha\eta}(1 - \lambda^2)$, which is done at eq. \ref{Eq-App-G4.2}. As this condition matches the claim conditions, the claim is proved.
\begin{equation}\label{Eq-App-G4.2}\begin{split}
\\{}& e^{\frac2{x_h^2\beta}} - 1 < \frac{1}{\alpha\eta}\left(1 - \lambda^2\right) \iff
\\{}& \alpha\eta\left(e^{\frac2{x_h^2\beta}} - 1\right) < \left(1 - \lambda^2\right) \iff
\\{}& \alpha\eta\left(e^{\frac2{x_h^2\beta}} - 1\right) < 1 - \left(1 - \frac\eta2\left(\alpha + nx_h^2\beta\right)\right)^2 \iff
\\{}& \alpha\left(e^{\frac2{x_h^2\beta}} - 1\right) < \left(\alpha + nx_h^2\beta\right) - \frac\eta4\left(\alpha + nx^2\beta\right)^2 \iff
\\{}& \alpha\left(e^{\frac2{x_h^2\beta}} - 1\right) < \left(\alpha + nx_h^2\beta\right) - \frac12 \iff
\\{}& \alpha\left(e^{\frac2{x_h^2\beta}} - 2\right) + \frac12 < nx_h^2\beta \iff
\\{}& \frac\alpha{x_h^2\beta}\left(e^{\frac2{x_h^2\beta}} - 2\right) + \frac1{2x_h^2\beta} < n
\end{split}\end{equation}
\end{proof}


\subsection{Posterior Sampling Privacy}
This subsection provides auxiliary claims for the posterior sampling privacy analysis performed in subsection \ref{PosteriorSamplingPrivacySection}. It uses the notations defined in section \ref{Method}, subsection \ref{PosteriorSamplingPrivacySection}, and subection \ref{subseq-PSP}, specifically: $\alpha, \beta, \theta, p(y|x, \theta)$ (defined in eq. \ref{LinearModel}), $\mathcal{D}(n, \gamma_1, x_h, x_l, c)$ (defined in eq. \ref{eq-Domain}), $p(\theta|D), p(\theta| \hat{D}), \theta, \mu, \sigma, \hat\mu, \hat\sigma, (\sigma^2)^*_\nu, \nu$ (defined in subsection \ref{PosteriorSamplingPrivacySection}), $D, \hat{D}, x_n, y_n, \hat{x}_n, \hat{y}_n, z, q$ (defined in eq. \ref{Eq-App-PosteriorAdjacentDBs}).

\begin{claim}\label{claim-Post-Div-1}
For $n > 1 +  10\frac{x_h^2}{x_l^2}\frac{\nu}{\beta}$, the inequality $\frac1{10}(\alpha + (z + x_n^2)\beta) > \nu(\hat{x}_n^2 - x_n^2)$ holds.
\end{claim}
\begin{proof}[Proof Claim \ref{claim-Post-Div-1}]
Notice that $\frac1{10}(\alpha + (z + x_n^2)\beta) > \frac1{10}z\beta > \frac1{10}(n-1)x_l^2\beta$ and $\nu x_h^2 > \nu(\hat{x}_n^2 - x_n^2)$. Therefore a sufficient condition will be that $\frac1{10}(n-1)x_l^2\beta > \nu x_h^2$, which is equivalent to $n > 1 + \frac{x_h^2}{x_l^2}\frac{10\nu}{\beta}$.
\end{proof}

\begin{claim}\label{claim-Post-Div-2}
$(\sigma^2)^*_\nu$ is positive.
\end{claim}
\begin{proof}[Proof Claim \ref{claim-Post-Div-2}]
\begin{equation}\label{eq-claim-post-div-2}\begin{split}
\left(\sigma^2\right)^*_\nu &= \nu\sigma^2 + \left(1-\nu\right)\hat\sigma^2 = 
\frac{\nu}{\alpha + \left(z + x_n^2\right)\beta} + \frac{1 - \nu}{\alpha + \left(z + \hat{x}_n^2\right)\beta}
\\&= 
\frac{\nu\left(\alpha + \left(z + \hat{x}_n^2\right)\beta\right) + \left(1 - \nu\right)\left(\alpha + \left(z + x_n^2\right)\beta\right)}{\left(\alpha + \left(z + x_n^2\right)\beta\right)\left(\alpha + \left(z + \hat{x}_n^2\right)\beta\right)}
= \frac{\alpha + \left(z + x_n^2\right)\beta + \nu\left(x_n^2 - \hat{x}_n^2\right)}{\left(\alpha + \left(z + x_n^2\right)\beta\right)\left(\alpha + \left(z + \hat{x}_n^2\right)\beta\right)}
\end{split}\end{equation}
therefore, a sufficient condition is that $\alpha + (z + x_n^2)\beta + \nu(x_n^2 - \hat{x}_n^2) > 0$. Since the condition of Lemma \ref{PosteriorRenyiFull} dictates that $n > 1 + 10\frac{x_h^2}{x_l^2}\frac\nu\beta$, then Claim \ref{claim-Post-Div-1} holds, and therefore the condition is satisfied.
\end{proof}

\begin{claim}\label{claim-Post-Div-3}
The value $\ln\frac{\sigma}{\hat\sigma}$ can be bounded as following:
$$\ln\frac{\sigma}{\hat\sigma} \leq \frac{x_h^2}{2(n - 1)x_l^2}.$$
\end{claim}
\begin{proof}[Proof of Claim \ref{claim-Post-Div-3}]
For $\hat{x}_n \leq x_n$, the term $\ln\frac{\sigma}{\hat\sigma}$ is negative and the claim trivially holds. For $\hat{x}_n > x_n$, consider $c_1 = \frac{x_h^2}{(n - 1)x_l^2}$:
\begin{equation}\label{eq-claim-Post-Div-3-1}
c_1 = \frac{x_h^2}{(n-1)x_l^2} > \frac{\hat{x}_n^2 - x_n^2}{z + x_n^2} > \frac{\hat{x}_n^2\beta - x_n^2\beta}{\alpha + (z + x_n^2)\beta} =  \frac{\alpha + (z + \hat{x}_n^2)\beta}{\alpha + (z + {x}_n^2)\beta} - 1.
\end{equation}
From eq. \ref{eq-claim-Post-Div-3-1}, by Taylor theorem:
\begin{equation*}
e^{c_1} = 1 + c_1 + \frac{e^\zeta}2(c_1)^2 > 1 + c_1 > \frac{\alpha + (z + \hat{x}_n^2)\beta}{\alpha + (z + {x}_n^2)\beta}
\end{equation*}
where $0 \leq \zeta \leq c_1$. Consequently, because the natural logarithm is monotonically increasing, the following equation also holds:
\begin{equation*}
\frac12c_1 > \frac12\ln\frac{\alpha + (z + \hat{x}_n)\beta}{\alpha + (z + {x}_n)\beta} = \ln\frac{\sigma}{\hat\sigma}.
\end{equation*}
Therefore $\ln\frac{\sigma}{\hat\sigma} < \frac12\frac{x_h^2}{(n - 1)x_l^2}$.
\end{proof}

\begin{claim}\label{claim-Post-Div-4}
For the conditions of Lemma \ref{PosteriorRenyiFull}, the value of $\frac12\left(\nu - 1\right)\ln\frac{\hat{\sigma}^2}{\left(\sigma^2\right)^*_\nu}$ can be upper bounded as following:
$$\frac12\left(\nu - 1\right)\ln\frac{\hat{\sigma}^2}{\left(\sigma^2\right)^*_\nu} \leq \frac12\left(\nu-1\right)\frac{\nu x_h^2}{2\left(\left(n - 1\right)x_l^2 - \nu x_h^2\right)}.$$
\end{claim}
\begin{proof}[Proof of Claim \ref{claim-Post-Div-4}] 
Consider $c_1 = \frac{\nu x_h^2}{((n - 1)x_l^2 - \nu x_h^2)}$:
\begin{equation*}\begin{split}
c_1 &= \frac{\nu x_h^2}{\left(n - 1\right)x_l^2 - \nu x_h^2} \geq^* 
\frac{\nu\beta x_h^2}{\alpha + \left(n - 1\right)x_l^2\beta - \nu\beta x_h^2} \geq^*
\frac{\nu\beta \hat{x}_n^2}{\alpha + \left(z + x_n^2\right)\beta - \nu\beta x_n^2} 
\\& \geq
\frac{\nu\beta\left(\hat{x}_n^2 - x_n^2\right)}{\alpha + \left(z + x_n^2\right)\beta - \nu\beta\left(x_n^2 - \hat{x}_n^2\right)}
= \frac{\alpha + \left(z + x_n^2\right)\beta}{\alpha + \left(z + {x}_n^2\right)\beta + \nu\beta\left(x_n^2 - \hat{x}_n^2\right)} - 1 
\\&= \frac{1}{\alpha + \left(z + \hat{x}_n^2\right)\beta}\cdot
\frac{\left(\alpha + \left(z + x_n^2\right)\beta\right)\left(\alpha + \left(z + \hat{x}_n^2\right)\beta\right)}{\alpha + \left(z + {x}_n^2\right)\beta + \nu\beta\left(x_n^2 - \hat{x}_n^2\right)} - 1 = 
\frac{\hat{\sigma}^2}{\left(\sigma^2\right)^*_\nu} - 1
\end{split}\end{equation*}
where inequalities * holds under the assumption that $n > 1 + \nu\frac{x_h^2}{x_l^2}$, and last equality holds from eq. \ref{eq-claim-post-div-2}. Therefore, by using Taylor theorem:
\begin{equation*}
e^{c_1} = 1 + c_1 + \frac{e^\zeta}2(c_1)^2 > 1 + c_1 \geq \frac{\hat{\sigma}^2}{(\sigma^2)^*_\nu}.
\end{equation*}
where $0 \leq \zeta \leq c_1$. From this inequality, and because the natural logarithm is monotonically increasing, it is true that $\ln\frac{\hat{\sigma}^2}{(\sigma^2)^*_\nu} \leq c_1$. Therefore 
$$
\frac12(\nu - 1)\ln\frac{\hat{\sigma}^2}{(\sigma^2)^*_\nu} \leq \frac12(\nu - 1)c_1 = \frac12(\nu-1)\frac{\nu x_h^2}{((n - 1)x_l^2 - \nu x_h^2)}.$$
\end{proof}

\begin{claim}\label{claim-Post-Div-5}
For the conditions of Lemma \ref{PosteriorRenyiFull}, the value $\frac\nu2\frac{(\mu - \hat\mu)^2}{(\sigma^2)_\nu^*}$ is bounded by
\begin{equation*}\begin{split}
{}& 2\nu\beta\cdot\frac{x_h^4}{\frac9{10}n^{1 - 2\gamma_1}x_l^2} + 2\nu\cdot\frac{x_h^4\left(\alpha + x_h^2\beta\right)}{\frac9{10}x_l^4}\cdot\frac{\left(c + n^{\gamma_1}\right)}{n^{2-\gamma_1}} + 
\frac\nu2\cdot\frac{x_h^4\left(\alpha + x_h^2\beta\right)^2}{\frac9{10}x_l^6\beta}\cdot\frac{\left(c + n^{\gamma_1}\right)^2}{n^3}.
\end{split}\end{equation*}
\end{claim}
\begin{proof}[Proof of Claim \ref{claim-Post-Div-5}]
First, we bound $|\mu - \hat{\mu}|$:
\begin{equation*}
\begin{split}
|{\mu - \hat{\mu}}| &= 
\beta|{\frac{q+x_ny_n}{\alpha + \left(z + x_n^2\right)\beta} - \frac{q+\hat{x}_n\hat{y}_n}{\alpha + \left(z + \hat{x}_n^2\right)\beta}}| \\&=
|{\frac{\left(q+x_ny_n\right)\left(\alpha + \left(z + \hat{x}_n^2\right)\beta\right) - \left(q+\hat{x}_n\hat{y}_n\right)\left(\alpha + \left(z + x_n^2\right)\beta\right)}{\left(\alpha + \left(z + x_n^2\right)\beta\right)\left(\alpha + \left(z + \hat{x}_n^2\right)\beta\right)}}|
\\&=
\beta|{\frac{q\hat{x}_n^2\beta + x_ny_n\alpha + x_ny_nz\beta + x_ny_n\hat{x}_n^2\beta - 
qx_n^2\beta - \hat{x}_n\hat{y}_n\alpha - \hat{x}_n\hat{y}_nz\beta - \hat{x}_n\hat{y}_nx_n^2\beta}
{\left(\alpha + \left(z + x_n^2\right)\beta\right)\left(\alpha + \left(z + \hat{x}_n^2\right)\beta\right)}}|
\\&=
\beta|{\frac{\hat{x}_n^2z\left(\frac{q}{z} - \frac{\hat{y}_n}{\hat{x}_n}\right)\beta - x_n^2z\left(\frac{q}{z} - \frac{y_n}{x_n}\right)\beta + \alpha\left(x_ny_n - \hat{x}_n\hat{y}_n\right) + x_n\hat{x}_n\beta\left(y_n\hat{x}_n - \hat{y}_nx_n\right)}{\left(\alpha + \left(z + x_n^2\right)\beta\right)\left(\alpha + \left(z + \hat{x}_n^2\right)\beta\right)}}| 
\\&<
\beta|{\frac{\hat{x}_h^2z\left(2n^{\gamma_1}\right)\beta + \alpha x_h^2\left(c + {n}^{\gamma_1}\right) + x_h^4\beta\left(c + n^{\gamma_1}\right)}{\left(\alpha + \left(z + x_n^2\right)\beta\right)\left(\alpha + \left(z + \hat{x}_n^2\right)\beta\right)}}| =
\beta|{\frac{2\hat{x}_h^2\beta zn^{\gamma_1} + \left(x_h^2\alpha + x_h^4\beta\right)\left(c + n^{\gamma_1}\right)}{\left(\alpha + \left(z + x_n^2\right)\beta\right)\left(\alpha + \left(z + \hat{x}_n^2\right)\beta\right)}}|.
\end{split}
\end{equation*}
Therefore, 
\begin{equation*}
\begin{split}
\frac{\nu}{2}\frac{\left(\mu - \hat{\mu}\right)^2}{\left(\sigma^2\right)^*_\nu} &\leq 
\frac\nu2\beta^2\left(\frac{2\hat{x}_h^2\beta zn^{\gamma_1} + \left(x_h^2\alpha + x_h^4\beta\right)\left(c + n^{\gamma_1}\right)}{\left(\alpha + \left(z + x_n^2\right)\beta\right)\left(\alpha + \left(z + \hat{x}_n^2\right)\beta\right)}\right)^2 \cdot 
\left(\frac{\alpha + \left(z + x_n^2\right)\beta + \nu\left(x_n^2 - \hat{x}_n^2\right)}{\left(\alpha + \left(z + x_n^2\right)\beta\right)\left(\alpha + \left(z + \hat{x}_n^2\right)\beta\right)}\right)^{-1}
\\&= 
\frac{\nu}{2}\frac{\beta^2\left(2\hat{x}_h^2\beta zn^{\gamma_1} + \left(x_h^2\alpha + x_h^4\beta\right)\left(c + n^{\gamma_1}\right)\right)^2}{\left(\alpha + \left(z + x_n^2\right)\beta\right)\left(\alpha + \left(z + \hat{x}_n^2\right)\beta\right)\left(\alpha + \left(z + x_n^2\right)\beta + \nu\left(x_n^2 - \hat{x}_n^2\right)\right)}
\\& \leq^*
\frac{\nu}{2}\frac{\beta^2\left(2x_h^2\beta zn^{\gamma_1} + \left(x_h^2\alpha + x_h^4\beta\right)\left(c + n^{\gamma_1}\right)\right)^2}{\frac9{10}\left(\alpha + \left(z + x_n^2\right)\beta\right)\left(\alpha + \left(z + \hat{x}_n^2\right)\beta\right)\left(\alpha + \left(z + x_n^2\right)\beta\right)}
\\&=
\frac{\nu}{2}\beta^2\left(\frac{\left(2x_h^2\beta\right)^2z^2n^{2\gamma_1} + 2\left(2x_h^2\beta\right)\left(x_h^2\alpha + x_h^4\beta\right)zn^{\gamma_1}\left(c + n^{\gamma_1}\right) + \left(x_h^2\alpha + x_h^4\beta\right)^2\left(c + n^{\gamma_1}\right)^2}{\frac9{10}\left(\alpha + \left(z + x_n^2\right)\beta\right)^2\left(\alpha + \left(z + \hat{x}_n^2\right)\beta\right)}\right) 
\\&\leq 
\frac{\nu}{2}\beta^2\left(\frac{\left(2x_h^2\beta\right)^2z^2n^{2\gamma_1} + \left(4x_h^2\beta\right)\left(x_h^2\alpha + x_h^4\beta\right)zn^{\gamma_1}\left(c + n^{\gamma_1}\right) + \left(x_h^2\alpha + x_h^4\beta\right)^2\left(c + n^{\gamma_1}\right)^2}{\frac9{10}\left(\left(z + x_n^2\right)\beta\right)^2\left(\left(z + \hat{x}_n^2\right)\beta\right)}\right)
\\&\leq^{**} 
\frac{\nu}{2}\beta^2\left(\frac{\left(2x_h^2\beta\right)^2n^{2\gamma_1}}{\frac9{10}nx_l^2\beta^3}\right) +
\frac{\nu}{2}\beta^2\left(\frac{\left(4x_h^2\beta\right)\left(x_h^2\alpha + x_h^4\beta\right)n^{\gamma_1}\left(c + n^{\gamma_1}\right)}{\frac9{10}\left(nx_l^2\right)^2\beta^3}\right) +
\frac{\nu}{2}\beta^2\left(\frac{\left(x_h^2\alpha + x_h^4\beta\right)^2\left(c + n^{\gamma_1}\right)^2}{\frac9{10}\left(nx_l^2\beta\right)^3}\right)
\\&= 2\nu\beta\left(\frac{x_h^4}{\frac9{10}n^{1 - 2\gamma_1}x_l^2}\right) + 2\nu\beta\left(\frac{\left(x_h^2\beta\right)\left(x_h^2\alpha + x_h^4\beta\right)}{\frac9{10}\left(x_l^2\beta\right)^2}\right)\frac{\left(c + n^{\gamma_1}\right)}{n^{2-\gamma_1}} + 
\frac\nu2\left(\frac{\left(x_h^2\alpha + x_h^4\beta\right)^2}{\frac9{10}x_l^6\beta}\right)\frac{\left(c + n^{\gamma_1}\right)^2}{n^3}
\end{split}
\end{equation*}
where inequality * is true because the conditions of Lemma \ref{PosteriorRenyiFull} dictates that $n > 1 + \frac{x_h^2}{x_l^2}\frac{10\nu}\beta$, and according to Claim \ref{claim-Post-Div-1} this guarantee that $\frac1{10}(\alpha + (z + x_n^2)\beta) > \nu(\hat{x}_n^2 - x_n^2)$. Inequality ** follows from $n >> 1 \Rightarrow (n-1)x_l \approx nx_l$.
\end{proof}

\begin{claim}\label{claim-Post-Div-9-1}
For the conditions and definitions of Lemma \ref{PosteriorPrivacy}, one sample from the posterior is $(\epsilon, \delta)$ differentially private for the following conditions on $n$ and $\nu$:
\begin{equation*}\begin{split}
\nu &= 1 + \frac{2\ln(\frac1\delta)}\epsilon
\\n &\geq \max\Big\{1 + \frac{x_h^2}{x_l^2}\frac{8}{\epsilon}, 
1 + \nu\frac{x_h^2}{x_l^2}\left(1 + 8\frac{\left(\nu - 1\right)}{\epsilon}\right), 
\left(\frac{16\nu\beta x_h^4}{\frac9{10}\epsilon x_l^2}\right)^{\frac1{1 - 2\gamma_1}},
\\& \left(\frac{16\nu}{\epsilon}\cdot\frac{x_h^4\left(\alpha + x_h^2\beta\right)}{\frac9{10}x_l^4}\left(c + n^{\gamma_1}\right)\right)^{\frac1{2 - \gamma_1}}, 
\left(\frac{4\nu}{\epsilon}\cdot\frac{x_h^4\left(\alpha + x_h^2\beta\right)^2}{\frac9{10}x_l^6\beta}\right)^\frac13\left(c + n^{\gamma_1}\right)^\frac23
\Big\}.
\end{split}\end{equation*}
\end{claim}
\begin{proof}[Proof of Claim \ref{claim-Post-Div-9-1}]
By Lemma \ref{PosteriorRenyiFull}, one sample from the posterior is $(\epsilon_1 + \frac{\ln(\frac1\delta)}{\nu - 1}, \delta)$ differentially private. For each of the six terms of $\epsilon_1 + \frac{\ln(\frac1\delta)}{\nu - 1}$, the lower bounds on $n$ and $\nu$, found at equations \ref{Eq-1-claim-Post-Div-9-1}, \ref{Eq-2-claim-Post-Div-9-1}, \ref{Eq-3-claim-Post-Div-9-1}, \ref{Eq-4-claim-Post-Div-9-1}, \ref{Eq-5-claim-Post-Div-9-1}, and \ref{Eq-6-claim-Post-Div-9-1}, guarantee that the sum of terms is upper bounded by ${\epsilon}$. These bounds match the claim's guarantee over $n$ and $\nu$, thus proving the claim.

For $\frac{\ln(\frac1\delta)}{\nu - 1}$:
\begin{equation}\label{Eq-1-claim-Post-Div-9-1}
\begin{split}
{}& \frac{\ln(\frac1\delta)}{\nu - 1} = \frac\epsilon2
\\{}& \frac{2\ln(\frac1\delta)}\epsilon + 1 = \nu.
\end{split}
\end{equation}

For $\frac{x_h^2}{(n - 1)x_l^2}$:
\begin{equation}\label{Eq-2-claim-Post-Div-9-1}
\begin{split}
{}& \frac{x_h^2}{2(n - 1)x_l^2} \leq  \frac\epsilon{16}
\\{}& n \geq 1 + \frac{x_h^2}{x_l^2}\frac{8}{\epsilon}.
\end{split}
\end{equation}

For $\frac{(\nu-1)\nu x_h^2}{2((n - 1)x_l^2 - \nu x_h^2)}$:
\begin{equation}\label{Eq-3-claim-Post-Div-9-1}
\begin{split}
{}& \frac{(\nu-1)\nu x_h^2}{2((n - 1)x_l^2 - \nu x_h^2)} \leq \frac\epsilon{16}
\\{}& \frac12(\nu-1)\frac{\nu x_h^2}{\left(n - 1\right)x_l^2 - \nu x_h^2} \leq \frac\epsilon{16}
\\{}& \frac12\left(\nu-1\right)\frac{16\nu x_h^2}{\epsilon} \leq \left(n - 1\right)x_l^2 - \nu x_h^2
\\{}& n \geq 1 + \frac12\left(\nu-1\right)\frac{16\nu x_h^2}{\epsilon x_l^2} + \nu\frac{x_h^2}{x_l^2} =
1 + \nu\frac{x_h^2}{x_l^2}\left(1 + 8\frac{\left(\nu - 1\right)}{\epsilon}\right).
\end{split}
\end{equation}

For $\frac{20\nu\beta x_h^4}{9n^{1 - 2\gamma_1}x_l^2}$:
\begin{equation}\label{Eq-4-claim-Post-Div-9-1}
\begin{split}
{}& \frac{20\nu\beta x_h^4}{9n^{1 - 2\gamma_1}x_l^2} \leq \frac\epsilon8
\\{}& \frac{16}\epsilon\frac{\nu\beta x_h^4}{\frac{9}{10}x_l^2} \leq n^{1 - 2\gamma_1}
\\{}& n \geq \left(\frac{16\nu\beta x_h^4}{\frac9{10}\epsilon x_l^2}\right)^{\frac1{1 - 2\gamma_1}}.
\end{split}
\end{equation}

For $\frac{20\nu x_h^4(\alpha + x_h^2\beta)}{9x_l^4}\cdot\frac{(c + n^{\gamma_1})}{n^{2-\gamma_1}}$:
\begin{equation}\label{Eq-5-claim-Post-Div-9-1}
\begin{split}
{}& \frac{20\nu x_h^4(\alpha + x_h^2\beta)}{9x_l^4}\cdot\frac{\left(c + n^{\gamma_1}\right)}{n^{2-\gamma_1}} \leq \frac{\epsilon}{8} 
\\{}& n^{2 - \gamma_1} \geq \frac{16\nu}{\epsilon}\cdot\frac{x_h^4(\alpha + x_h^2\beta)}{\frac{9}{10}x_l^4}\cdot\left(c + n^{\gamma_1}\right)
\\{}& n \geq \left(\frac{16\nu}{\epsilon}\cdot\frac{x_h^4(\alpha + x_h^2\beta)}{\frac{9}{10}x_l^4}\cdot\left(c + n^{\gamma_1}\right)\right)^{\frac1{2 - \gamma_1}}.
\end{split}
\end{equation}

For term $\frac{5\nu x_h^4(\alpha + x_h^2\beta)^2}{9x_l^6\beta}\cdot\frac{(c + n^{\gamma_1})^2}{n^3}$
\begin{equation}\label{Eq-6-claim-Post-Div-9-1}
\begin{split}
{}& \frac{5\nu x_h^4(\alpha + x_h^2\beta)^2}{9x_l^6\beta}\cdot\frac{(c + n^{\gamma_1})^2}{n^3} \leq \frac{\epsilon}{8}
\\{}& n^3 \geq \frac{4\nu}{\epsilon}\cdot\frac{10 x_h^4(\alpha + x_h^2\beta)^2}{9x_l^6\beta}\cdot\left(c + n^{\gamma_1}\right)^2
\\{}& n \geq \left(\frac{4\nu}{\epsilon}\cdot\frac{10 x_h^4(\alpha + x_h^2\beta)^2}{9x_l^6\beta}\cdot\left(c + n^{\gamma_1}\right)^2\right)^{\frac13}
\\{}& n \geq \left(\frac{4\nu}{\epsilon}\cdot\frac{10 x_h^4(\alpha + x_h^2\beta)^2}{9x_l^6\beta}\right)^\frac13\left(c + n^{\gamma_1}\right)^{\frac23}
\end{split}
\end{equation}
\end{proof}

\newpage
\begin{claim}\label{Claim-Post-9}
For c = $n^{\gamma_2}, \gamma_1 < \gamma_2 < \frac32$, and the conditions and definitions of Lemma \ref{PosteriorPrivacy}, one sample from the posterior is $(\epsilon, \delta)$ differentially private for following terms on $n$ and $\nu$:
\begin{equation*}\begin{split}
\nu &= \frac{2\ln\left(\frac1\delta\right)}\epsilon + 1
\\n &\geq \max\Big\{1 + \frac{x_h^2}{x_l^2}\frac{8}{\epsilon},
1 + \nu\frac{x_h^2}{x_l^2}\left(1 + 8\frac{\left(\nu - 1\right)}{\epsilon}\right), 
\left(\frac{16\nu\beta x_h^4}{\frac9{10}\epsilon x_l^2}\right)^{\frac1{1 - 2\gamma_1}},
\\& \left(\frac{16\nu}{\epsilon}\cdot\frac{x_h^4\left(\alpha + x_h^2\beta\right)}{\frac9{10}x_l^4}\left(1 + \frac1{\left(1 + 10\frac{x_h^2}{x_l^2}\frac\nu\beta\right)^{\gamma_2 - \gamma_1}}\right)\right)^{\frac1{2 - \gamma_1-\gamma_2}}, 
\\& \left(\frac{4\nu}{\epsilon}\cdot\frac{\left(x_h^2\alpha + x_h^4\beta\right)^2}{\frac9{10}x_l^6\beta}\left(1 + \frac1{\left(1 + 10\frac{x_h^2}{x_l^2}\frac\nu\beta\right)^{\gamma_2 - \gamma_1}}\right)^2\right)^\frac1{3 - 2\gamma_2}\Big\}.
\end{split}\end{equation*}
\end{claim}
\begin{proof}[Proof of Claim \ref{Claim-Post-9}]
Claim \ref{claim-Post-Div-9-1} provides lower bounds on $n$ such that one sample from the posterior will be $(\epsilon, \delta)$ differential privacy. When $c = n^{\gamma_2}, \gamma_2 > \gamma_1$, these bounds can be refined.

The condition $n \geq (\frac{16\nu}{\epsilon}\cdot\frac{x_h^4(\alpha + x_h^2\beta)}{\frac9{10}x_l^4}(c + n^{\gamma_1}))^{\frac1{2 - \gamma_1}}$ can be refined as follows:

\begin{equation*}
\begin{split}
\left(\frac{16\nu}{\epsilon}\cdot\frac{x_h^4\left(\alpha + x_h^2\beta\right)}{\frac9{10}x_l^4}\left(c + n^{\gamma_1}\right)\right)^{\frac1{2 - \gamma_1}} &= 
\left(\frac{16\nu}{\epsilon}\cdot\frac{x_h^4\left(\alpha + x_h^2\beta\right)}{\frac9{10}x_l^4}n^{\gamma_2}\left(1 + \frac1{n^{\gamma_2 - \gamma_1}}\right)\right)^{\frac1{2 - \gamma_1}} 
\\&\leq
\left(\frac{16\nu}{\epsilon}\cdot\frac{x_h^4\left(\alpha + x_h^2\beta\right)}{\frac9{10}x_l^4}n^{\gamma_2}\left(1 + \frac1{\left(1 + 10\frac{x_h^2}{x_l^2}\frac\nu\beta\right)^{\gamma_2 - \gamma_1}}\right)\right)^{\frac1{2 - \gamma_1}}
\end{split}\end{equation*}
where the inequality holds because Lemma \ref{PosteriorPrivacy} dictates that $n > 1 + 10\frac{x_h^2}{x_l^2}\frac\nu\beta$. Consequently it's enough that 
\begin{equation*}
\begin{split}
n \geq \left(\frac{16\nu}{\epsilon}\cdot\frac{x_h^4\left(\alpha + x_h^2\beta\right)}{\frac9{10}x_l^4}\left(1 + \frac1{\left(1 + 10\frac{x_h^2}{x_l^2}\frac\nu\beta\right)^{\gamma_2 - \gamma_1}}\right)\right)^{\frac1{2 - \gamma_1-\gamma_2}}.
\end{split}\end{equation*}

Following the same considerations for condition $n \geq \left(\frac{4\nu}{\epsilon}\cdot\frac{x_h^4\left(\alpha + x_h^2\beta\right)^2}{\frac9{10}x_l^6\beta}\right)^\frac13\left(c + n^{\gamma_1}\right)^\frac23$:
\begin{equation*}
\left(\frac{4\nu}{\epsilon}\cdot\frac{x_h^4\left(\alpha + x_h^2\beta\right)^2}{\frac9{10}x_l^6\beta}\right)^\frac13\left(c + n^{\gamma_1}\right)^\frac23 \leq
\left(\frac{4\nu}{\epsilon}\cdot\frac{x_h^4\left(\alpha + x_h^2\beta\right)^2}{\frac9{10}x_l^6\beta}\right)^\frac13n^{\frac{2\gamma_2}3}\left(1 + \frac1{\left(1 + 10\frac{x_h^2}{x_l^2}\frac\nu\beta\right)^{\gamma_2 - \gamma_1}}\right)^\frac23
\end{equation*}
Therefore, it is enough that
\begin{equation*}
n \geq \left(\frac{4\nu}{\epsilon}\cdot\frac{x_h^4\left(\alpha + x_h^2\beta\right)^2}{\frac9{10}x_l^6\beta}\left(1 + \frac1{\left(1 + 10\frac{x_h^2}{x_l^2}\frac\nu\beta\right)^{\gamma_2 - \gamma_1}}\right)^2\right)^\frac1{3 - 2\gamma_2}.
\end{equation*}
\end{proof}

\newpage
\section{Clipped Gradients}\label{sec::clipped_grads}
Previous work \citep{ConnectMCMCtoDP} suggested training machine learning models using an SGLD-inspired learning step with clipped gradients to get differentially private models. Given the definitions in section \ref{SGLD-bkgrnd} and a gradient clipping threshold $C\in\R^{>0}$, the SGLD-inspired learning step with clipped gradient is

\begin{equation}\label{Eq-SGLD-Clipped-Update-Rule}
\begin{split}
\theta_{j+1} &= \theta_j + \frac{\eta_j}{2}\left(\nabla_{\theta_j}\ln{p(\theta_j)}+\frac{n}{b}\sum_{i=1}^{b}\nabla_{\theta_j}\ln{p(y_{i_j}|\theta_j, x_{i_j})}/\max(1, \frac{\|\nabla_{\theta_j}\ln{p(y_{i_j}|\theta_j, x_{i_j})\|_2}}{C})\right) + \sqrt{\eta_j}\xi_j 
\\i_j&\sim uniform\{1,...,n\}
\\\xi_j &\sim \mathcal{N}(0,1).
\end{split}
\end{equation}

We repeated the attack described in subsection \ref{Empirical} with models that were trained with the learning step described in eq. \ref{Eq-SGLD-Clipped-Update-Rule}. The models were trained with clipping threshold $C=0.2$, a learning rate of $0.001$\footnote{Effective learning rate after multiplication by SGLD's normalization factor, i.e. $\eta\frac{n}{2b}$. See learning step in eq. \ref{Eq-SGLD-Clipped-Update-Rule} for details.}, and a batch size of $4$. We created a novel sample, $(x^*, y^*)$, and used $200$ models to train the classifier and another $200$ models on which we used the classifier to estimate the DP lower bound. Lastly, we used the "Opacus" framework \citet{opacus} to run the experiment.

Figure \ref{fig::SGLD-empirical-epsilon-LeNet-clipped} depicts the model's accuracy as well as lower ($\epsilon_{lb}^{emp}$) and upper ($\epsilon_{ub}$) bounds over $\epsilon$, given $\delta = 10^{-5}$. The lower bound has a confidence value of $90\%$, i.e., $P(\epsilon \geq \epsilon_{lb}^{emp}) \geq 0.9025$, while the upper bound is computed using the "Opacus" framework \citet{opacus} in R\'enyi-DP terms (See definition \ref{Def-Renyi-DP}) and converted to $(\epsilon, \delta)$-DP terms using Lemma \ref{Lemma-RDP-to-ADP}.

From figure \ref{fig::SGLD-empirical-epsilon-LeNet-clipped}, we see that the attack did not succeed in showing a privacy breach. However, we also see that the maximum accuracy is $90.6\%$ (which is $8\%$ lower than the accuracy for models trained with SGLD, as shown in figure \ref{fig::SGLD-empirical-epsilon-LeNet}).

\begin{figure}[bh]
\begin{center}
\includegraphics[width=0.5\linewidth]{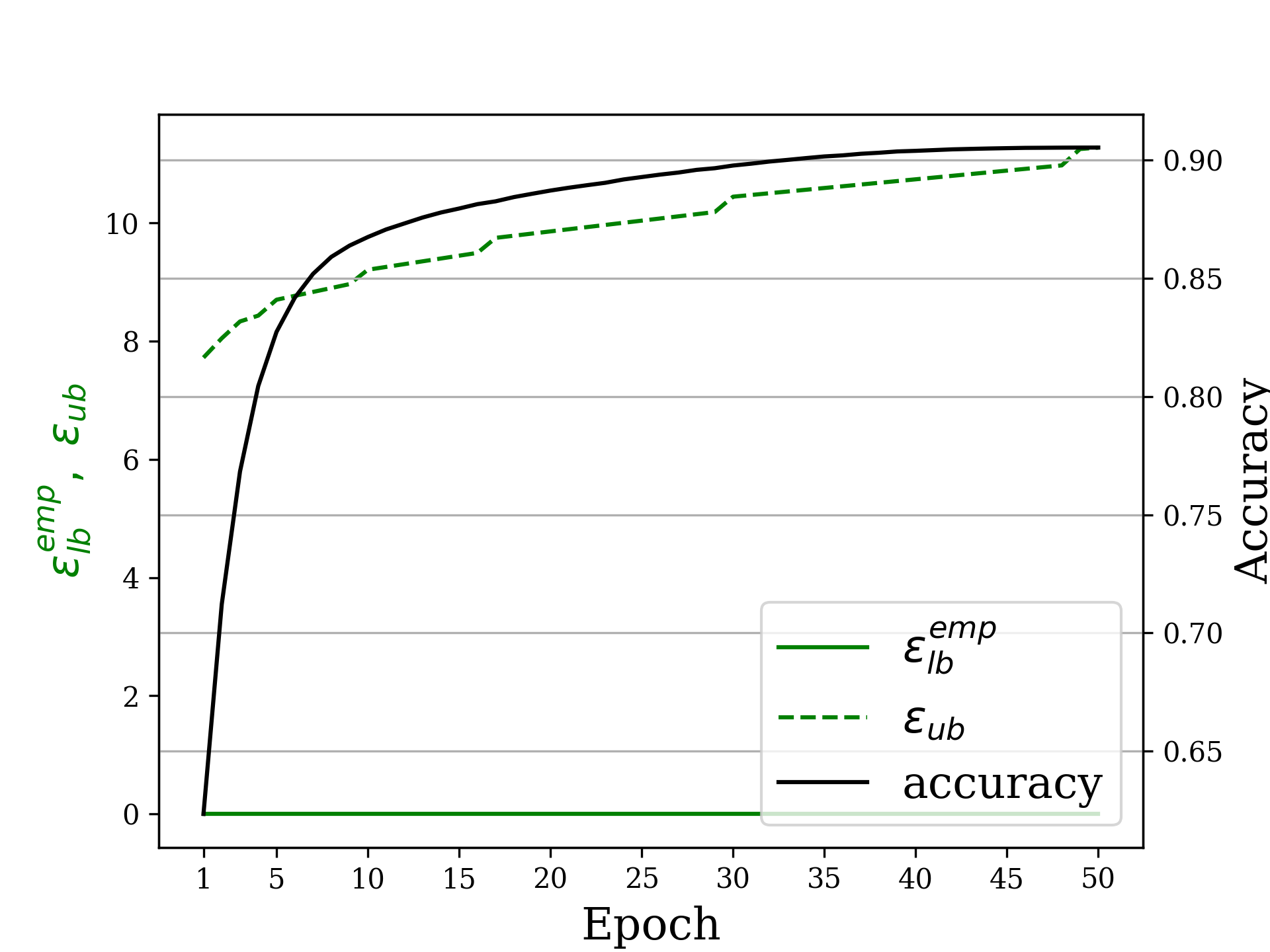}
\end{center}
\caption{Lower ($\epsilon_{lb}^{emp}$) and upper ($\epsilon_{ub}$) bounds over the differential privacy of the LeNet5, SGLD based, training process with clipped gradients over MNIST for a given $\delta$, for learning rate $0.001$, a batch size of $4$, and clipping value of $0.2$. Upper bound was calculated in R\'enyi-DP terms (See definition \ref{Def-Renyi-DP}) using \citet{opacus} and converted to $(\epsilon, \delta)$ -DP terms using \ref{Lemma-RDP-to-ADP}.}
\label{fig::SGLD-empirical-epsilon-LeNet-clipped}
\end{figure}
%

\end{document}